\documentclass[11pt,a4paper,reqno]{amsart}

\usepackage[english]{babel}
\usepackage[latin1]{inputenc}
\usepackage[top=3cm,bottom=2cm,left=3cm,right=3cm,marginparwidth=1.75cm]{geometry}

\usepackage{enumitem}
\usepackage{subcaption}
\usepackage{comment}
\usepackage{mathrsfs}

\usepackage{algorithm}
\usepackage{algpseudocode}
\usepackage{braket}
\usepackage{indentfirst}
\usepackage{amssymb}
\usepackage{amsmath}
\usepackage{latexsym}
\usepackage{amsthm}
\usepackage{tikz}
\usepackage{pictexwd, dcpic}
\usepackage[all]{xy}
\usetikzlibrary{patterns}
\usepackage{wrapfig}
\usepackage{graphicx}
\usepackage{float}
\usepackage{xcolor}
\usepackage{hyperref}
\hypersetup{
    colorlinks,
    linkcolor={red!50!black},
    citecolor={blue!50!black},
    urlcolor={blue!80!black}
    pdfhighlight=/I
}

\usepackage{cleveref}
\newtheorem{thm}{Theorem}[subsection]
\newtheorem{cor}{Corollary}[subsection]
\newtheorem{lem}{Lemma}[subsection]
\newtheorem{prop}{Proposition}[subsection]
\newtheorem{defn}{Definition}[subsection]
\newtheorem{alg}{Algorithm}[section]

\newcommand{\numberset}{\mathbb}
\newcommand{\R}{\numberset{R}}

\newcommand{\E}{\numberset{E}}

\newcommand{\rank}{\operatorname{rank}}

\newcommand{\n}[1]{\left\lVert#1\right\rVert}

\newcommand{\op}{\operatorname{op}}

\newcommand{\erm}{\operatorname{erm}}

\usepackage{wrapfig}

\usepackage{pgfplots,relsize,siunitx}
\usetikzlibrary{calc,trees,positioning,arrows,chains,shapes.geometric,%
    decorations.pathreplacing,decorations.pathmorphing,shapes,%
    matrix,shapes.symbols,automata,arrows.meta,quotes}
    
\usepackage{tikz-cd}
\newcommand{\eval}{\mathcal{E}}

\usepackage{verbatim}

\usepackage{algorithm}
\usepackage{algpseudocode}

\newcommand{\dist}{\operatorname{dist}}
\newcommand{\Jac}{\operatorname{Jac}}

\newcommand{\Mat}{\operatorname{Mat}}

\newcommand{\Sym}{\operatorname{Sym}}

\definecolor{mybluelight}{RGB}{208, 223, 241}
\definecolor{myblue}{RGB}{124,164,204} 
\definecolor{mydarkblue}{RGB}{40,70,100}

\begin{document}

\title{Testing Algebraic Complete Intersections}
\author{Alessandro Tamai}
\maketitle
\markboth{TESTING ALGEBRAIC COMPLETE INTERSECTIONS
- PART I}{ALESSANDRO TAMAI}

\begin{abstract}
Given i.i.d.\ samples from a probability distribution in a potentially high-dimensional real space, we study the problem of testing whether the distribution is concentrated near a real algebraic complete intersection of prescribed dimension, bounded degree, and bounded condition number. We design an explicit and effective learning procedure which either certifies the nonexistence of such a manifold, up to a controlled relaxation of the approximation threshold, or returns a candidate regression manifold with controlled geometric complexity. Equivalently, the procedure tests the manifold hypothesis within this hypothesis class. The proposed procedure relies on quantitative geometric estimates for regular polynomial systems, which lead to a tractable auxiliary optimization problem. We then develop a data-driven algorithm to solve this auxiliary  optimization problem, establishing explicit bounds on its sample and arithmetic complexity.

\end{abstract}

\tableofcontents

\section{Introduction and Motivation}

Modern machine learning methods have proved to be powerful tools for performing statistical inference, especially in the analysis of large datasets in high-dimensional spaces. 

Despite the exponential growth of volume in high dimensions and the relative degradation of classical statistical performance in these regimes, the efficiency of modern algorithms appears to rely on their ability to exploit intrinsic relations underlying the data. Empirical observations support this phenomenon, indicating that such relations often constrain real-world data to lie close to lower-dimensional geometric structures. 

Information about the underlying geometry can be used to optimize the learning process (\cite{Pope2021, Ansuini2019}), as well as to reduce the computational demands of learning algorithms (see \cite{McInnes2018, roweis2000lle, tenenbaum2000isomap}), playing a central role in modern data analysis.

In general, these geometric structures may possess varying degrees of regularity, ranging from smooth manifolds to merely topological sets for which not even a uniform notion of dimension can be defined. In terms of interpretability, highly regular structures enable more effective methods, while weaker regularity limits the available algorithmic approaches. Many efficient algorithms (such as, e.g., \cite{roweis2000lle, tenenbaum2000isomap}) explicitly rely on the assumption that the data concentrate near a lower-dimensional smooth manifold, a structural assumption commonly known as the \textit{manifold hypothesis}. 

From a mathematical viewpoint, the manifold hypothesis has been formalized in a sequence of works by Fefferman, Mitter, and Narayanan (\cite{narayanan2010sample, FeffermanTesting, fefferman2018fitting, fefferman2020reconstruction, fefferman2023fitting}).  In particular, in \cite{FeffermanTesting}, the authors provide a procedure for testing whether the manifold hypothesis holds for the class of manifolds with fixed dimension, volume bounded from above, and reach bounded from below. For a given approximation threshold, the proposed test is implemented via a learning algorithm that processes a sufficiently large dataset from the data-generating distribution and distinguishes among different possible alternatives.

From a more general viewpoint, \cite{FeffermanTesting} suggests that suitable candidate \textit{hypothesis classes} for developing testing procedures are given by classes of manifolds with \textit{controlled geometry}. Conceptually, this means that the local geometry of manifolds in the class is controlled in a quantitative and uniform way, and it classically corresponds to families of manifolds with uniformly bounded curvature and a positive lower bound on the injectivity radius. Different classes of manifolds with controlled geometry can therefore be considered, each corresponding to a different hypothesis class for the testing problem.

\begin{figure}[!h]
\center
\begin{tikzpicture}[scale=2]
  \begin{axis}[
    view={60}{30},
    axis equal image,
    hide axis,
  ]

    \addplot3[
      fill opacity=0,     
      draw=black,         
      line width=0.05pt,
      shader=flat,
      samples=40,
      samples y=25,
      domain=0:360,
      y domain=0:360,
    ]
    (
      {-( 1.2 * sin(y) )},
      {(3.5 + 1.2*cos(y)) * cos(x)},
      {-(3.5 + 1.2*cos(y)) * sin(x)}
    );

    \addplot3[
      only marks,
      mark=*,
      mark size=0.5pt,
      mark options={fill=mydarkblue,draw=mydarkblue},
      samples=15,
      domain=0:360,
      y domain=0:360,
    ]
    (
      {-( 1.2*sin(y) + 0.25*(rand-0.5))},
      {(3.5 + 1.2*cos(y) + 0.25*(rand-0.5)) * cos(x)},
      {-(3.5 + 1.2*cos(y) + 0.25*(rand-0.5)) * sin(x)}
    );
\end{axis}
\end{tikzpicture}
\vspace{-8mm}
\caption{A sample concentrated near a torus in $\R^3$}
\end{figure}
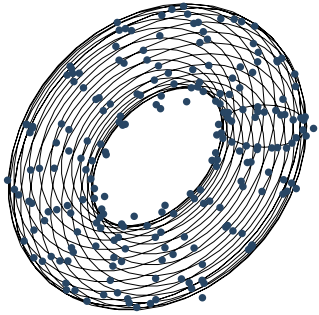

A more recent algebraic counterpart of the manifold hypothesis was introduced in \cite{lerario2025testing}, where the problem of testing whether a probability measure is concentrated near a real algebraic variety of prescribed dimension and bounded degree was studied under the name of the \emph{variety hypothesis}. Unlike in \cite{FeffermanTesting}, the real algebraic varieties considered in \cite{lerario2025testing} are not necessarily manifolds, and singularities are allowed. 

The purpose of the present work is to develop an effective testing procedure in a more structured algebraic setting, for which both statistical and computational complexity bounds can be obtained. More precisely, inspired by the testing framework of \cite{FeffermanTesting} and by the algebraic formulation of \cite{lerario2025testing}, we consider as hypothesis class a family of \emph{real algebraic manifolds with controlled geometry}. We introduce the class $\mathcal{H}_n(l,d,k)$ of real algebraic complete intersections of dimension $l$, degree at most $d$, and condition number at most $k$, and develop a procedure for testing the geometric hypothesis \eqref{def:GeomHyp_eq} with respect to this class. Roughly speaking, elements of $\mathcal{H}_n(l,d,k)$ are the zero sets of systems of polynomial equations of degree at most $d$ satisfying a quantitative regularity condition, encoded by a bound on the condition number of the defining system (see \Cref{subsec:RealAlgCompInt} for the precise definition).

The testing procedure takes as input a sufficiently large sample from the data-generating distribution and either identifies the case where no manifold exists in the prescribed class $\mathcal H_n(l,d,k)$, or certifies the existence of a regression manifold with controlled relaxation of the defining parameters. 

Unlike the setting of \cite{FeffermanTesting}, where the candidate manifolds can be locally parametrized and analyzed patch by patch, the algebraic constraints are global in nature and cannot be captured by purely local arguments. The requirement that there exist a single global polynomial system of prescribed degree cannot, in general, be enforced through independent local parametrizations alone. This makes the effective implementation of the testing procedure substantially different and naturally leads to an intrinsically global optimization strategy.

\subsection{The Main Result: Testing Algebraic Complete Intersections} A formal definition of the manifold hypothesis can be given by adapting the formulation introduced in \cite{narayanan2010sample} and the more general viewpoint adopted in \cite{lerario2025testing}. In the following, we refer to the resulting paradigm---encompassing both smooth manifolds and possibly singular geometric configurations---as the \emph{geometric hypothesis}. Without loss of generality, we restrict the problem to the disk $D_R^n\subseteq\mathbb{R}^n$ of arbitrary radius $R>0$. Denote by $\mathcal{K}(D_R^n)$ the class of non-empty compact subsets of $D_R^n$, and by $\mathcal{M}(D_R^n)$ the class of Borel probability measures compactly supported in $D_R^n$.

\begin{defn}[Geometric Hypothesis] \label{defn:GeneralGeomHyp} Let $X$ be a random variable with law $\mu \in \mathcal{M}(D^n_R)$. Given $\varepsilon>0$, we say that $X$ satisfies the $\varepsilon$-\textit{geometric hypothesis} associated with a class of compact sets $\mathcal{H}\subseteq \mathcal{K}(D^n_R)$ if there exists $Z \in \mathcal{H}$ such that 
\begin{equation}\label{def:GeomHyp_eq1}
\mathbb{E}[\dist(X,Z)^2] < \varepsilon ,
\end{equation} where $\dist(x,Z)$ denotes the closest point distance of $x$ from the set $Z$. 
\end{defn}

We refer to a hypothesis $Z$ satisfying condition \eqref{def:GeomHyp_eq1} as an $\varepsilon$-\textit{regression geometry} (or simply a regression geometry) for the random variable $X$. Given a probability measure $\mu$, a hypothesis class $\mathcal{H}$, and a tolerance $\varepsilon>0$, testing the hypothesis class $\mathcal{H}$ for the measure $\mu$ consists, conceptually, in determining whether \eqref{def:GeomHyp_eq1} holds for $\mu$ or not. 

The main result of this work is the following. In the context of computational decision problems, an algorithm is said to \emph{distinguish} between different scenarios if, given a set of input data, it correctly identifies the class to which the underlying instance belongs. In our setting the algorithm, processing a sufficiently large number of samples, either certifies the existence of an admissible regression manifold within a controlled relaxation of the prescribed parameters, or identifies the case where no such manifold exists in the prescribed class, up to a prescribed relaxation of the approximation threshold. 

\begin{thm}\label{Intro_thm:FinalMainTheorem}
Fix a dimension $l$, a degree $d$, a condition number bound $k$, a tolerance $\varepsilon>0$, and a confidence level $\alpha \in (0,1)$. Let $\mathcal{M}(D^n_R)$ be a class of Borel probability measures whose supports are uniformly separated from the boundary of $D^n_R$. Then there exist a universal constant $C'\ge1$, a constant $C\ge1$ depending only on $k$, and an effective learning algorithm with the following property. For any measure $\mu \in \mathcal{M}(D_R^n)$ and any random variable $X$ with law $\mu$, upon processing at least $m(n,l,d,k,\varepsilon,\alpha)<\infty$ i.i.d.\ samples from $X$, the algorithm distinguishes, with probability at least $1-\alpha$, among the three possible alternatives (A), (B), and ``(A) and (B)'', where:
\begin{enumerate}[label=(\Alph*)]
\item there exists a manifold $Z\in \mathcal{H}_n(l,d,kC')$ such that $$\E[\dist(X,Z)^2] \le C\,\varepsilon;$$
\item there exists no manifold $Z\in \mathcal{H}_n(l,d,k)$ such that $$\E[\dist(X,Z)^2] \le \frac{\varepsilon}{C}.$$
\end{enumerate}
\end{thm}
\noindent
In particular, whenever the algorithm returns alternative (A), either alone or together with (B), a candidate manifold satisfying the stated error bound is obtained. 

The sample size $m(n,l,d,k,\varepsilon,\alpha)$ appearing in the theorem yields an explicit upper bound on the sample complexity $m_{\operatorname{test}}^*(n,l,d,k,\varepsilon,\alpha)$ of the testing procedure. For both the arithmetic and statistical complexity analyses, we focus in particular on the regime of dimensional reduction in high-dimensional ambient spaces. More precisely, we consider the regime in which the complexity of the hypothesis class---namely, the intrinsic dimension $l$, the degree bound $d$, and the condition number bound $k$---is kept fixed, while the ambient dimension $n$ is allowed to grow.  As shown in \Cref{IPRthm:SampleCom} and \Cref{cor:TestingSampleComplexity}, in this regime we obtain
\begin{equation}\label{intro_thm:SamCom}
m_{\operatorname{test}}^*(n,l,d,k,\varepsilon,\alpha)\leq \mathcal{O}\left(\frac{C_{\operatorname{sc}}(k)}{\varepsilon^2}\bigg[\log\frac{2}{\alpha}+n^{d+1}\log\Big(\frac{C_{\operatorname{sc}}(k)}{\varepsilon}\Big)\bigg]\right),
\end{equation}
where $C_{\operatorname{sc}}(k)>0$ depends only on $k$. In particular, for fixed $l,d,k$, $\varepsilon$, and $\alpha$, the sample complexity grows at most as $\mathcal{O}(n^{d+1})$ as $n\to\infty$. A similar asymptotic analysis of the arithmetic complexity is carried out in Part II (see \Cref{thm:RelaxedAdmissibleGlobalOptimizationFinal}). 

The paper is organized as follows. The proof of \Cref{Intro_thm:FinalMainTheorem} is obtained by combining two main results, corresponding to the two parts of the paper. In Part I, \Cref{thm:MainTheoremPartI} establishes the testing framework and its statistical guarantees, reducing its effective implementation to an algorithmic optimization problem. In Part II, we instead provide an effective global optimization procedure together with the corresponding arithmetic complexity bounds, whose properties are summarized in \Cref{thm:RelaxedAdmissibleGlobalOptimizationFinal}. For an overview of the two parts, see next \Cref{subsec_OverI:Geometric_Problem} and \Cref{subsec_OverII:Algorithmic_Problem}.

\subsection{Part I: Geometric Problem and Statistical Bounds} \label{subsec_OverI:Geometric_Problem} The testing problem can be formulated in the language of statistical learning theory (see \Cref{subsec:prel_Statistical_Learning}). For a random variable $X$ with law $\mu$ on $D_R^n$ and the hypothesis class $\mathcal H_n(l,d,k)$, testing the geometric hypothesis at tolerance $\varepsilon$ amounts to comparing \begin{equation}\label{introI_eq:MinProblem} \inf_{Z\in\mathcal H_n(l,d,k)}\mathbb E[\dist(X,Z)^2] \end{equation} with the prescribed threshold $\varepsilon$. In this formulation, the functional $Z\mapsto\mathbb E[\dist(X,Z)^2]$ plays the role of the risk functional over the class $\mathcal H_n(l,d,k)$, whose elements are the zero sets of systems of polynomial equations of degree at most $d$ satisfying a quantitative regularity condition expressed through the condition number. More precisely, if $P$ denotes a defining polynomial system, then $P$ is required to satisfy $K_{D_R^n}(P)\leq k$, where $$K_{D_R^n}(P)=\frac{\n{P}_{C^2(D_R^n)}}{\dist(P,\Delta(D_R^n))},$$ and $\Delta(D_R^n)$ denotes the set of $C^2$--systems that fail to be regular on $D_R^n$.

Even following a standard empirical risk minimization scheme (see \Cref{subsec:prel_Statistical_Learning}), and hence replacing the risk functional with its empirical counterpart, the direct optimization of \eqref{introI_eq:MinProblem} remains computationally difficult because of the implicit description of the candidate manifolds and the evaluation of the cost functional. In particular, although the defining polynomial systems admit an explicit coefficient parametrization, the corresponding sets $Z(P)$ do not have a tractable explicit description, and computing $\dist(x,Z(P))$ would require solving a global optimization problem.

To overcome these computational difficulties, we introduce a computationally accessible auxiliary optimization problem and establish quantitative bounds relating it to the original one. 

Denote by $c:=n-l$ the codimension, and let $\mathcal P_n(c,d,k)$ be the family of $c$--tuples of polynomials of degree at most $d$, with condition number bounded by $k$ and unit $C^2$--norm on $D_R^n$, and let $\mathcal P_n^*(c,d,k)\subseteq\mathcal P_n(c,d,k)$ denote the subfamily of systems with non-empty zero set in $D_R^n$. The systems in $\mathcal P_n^*(c,d,k)$ correspond to the polynomial systems associated with the hypothesis class $\mathcal H_n(l,d,k)$, while systems in $\mathcal P_n(c,d,k)$ are not required a priori to have a non-empty zero set. On the other hand, the class $\mathcal P_n(c,d,k)$ admits an explicit description in terms of coefficient coordinates.

The first geometric result provides a quantitative criterion for the existence of zeros in $D_R^n$ for systems in $\mathcal P_n(c,d,k)$, and therefore a membership criterion for $\mathcal P_n^*(c,d,k)$. \Cref{thm:HdZeroes}, based on a quantitative Newton method, shows that, under the prescribed regularity assumptions, if $\n{P(y)}$ is sufficiently small at a point $y$ separated from the boundary of $D_R^n$, then $P$ admits a zero in a controlled neighborhood of $y$. As a consequence, the existence of zeros can be detected from evaluations of the polynomial system on a suitable finite net of $D_R^n$.

The second geometric result, collected in \Cref{thm:UnifDistBound} and \Cref{cor:UnifDistBound}, provides a uniform \L ojasiewicz-type inequality comparing the value $\n{P(y)}$ of a system with the distance of $y$ to its zero set. More precisely we prove that, for $P\in\mathcal P_n^*(c,d,k)$ and points $x$ uniformly separated from $\partial D_R^n$, one has \begin{equation}\label{Intro_eq_ineq}\dist(x,Z(P)\cap D_R^n)^2\leq C\n{P(x)}^2, \end{equation} where $C$ is uniform over the class and, in particular, depends only on the prescribed condition-number bound $k$.

The pointwise comparison \eqref{Intro_eq_ineq} extends directly to both the expected and empirical averages of the corresponding quantities (see \Cref{sec:Auxiliary_Optimization_Discretization}), thereby motivating the replacement of the geometric cost in \eqref{introI_eq:MinProblem} by the auxiliary risk $\mathbb E[\n{P(X)}^2]$. This leads to the auxiliary minimization problem \begin{equation}\label{eq:AuxiliaryOptimization} \inf_{P\in\mathcal P_n(c,d,k)}\mathbb E[\n{P(X)}^2]. \end{equation} The comparison bounds, together with the zero criterion for membership in $\mathcal P_n^*(c,d,k)$, allow the solution of the auxiliary problem to be translated into the alternatives of the original geometric testing problem. We refer to this step as the \emph{decision procedure}; its construction is carried out in \Cref{sec:Auxiliary_Optimization_Discretization}.

Unlike the original geometric risk, this auxiliary functional can be evaluated directly from the observed data and the defining polynomial system. Moreover, the corresponding minimization problem is defined over a class that admits an explicit parametrization, making it suitable for an effective algorithmic treatment. 

The proposed algorithm is based on a discretization of the admissibility conditions defining the class $\mathcal P_n(c,d,k)$, replacing the global quantities entering the condition number by their discrete counterparts evaluated on a finite net $\mathcal N\subset D_R^n$, and leading to the discretized condition number $$K_{\mathcal N}(P):=\frac{\n{P}_{\mathcal N}}{\delta_{\mathcal N}(P)},$$ where $\delta_{\mathcal N}$ denotes the discretized version of the discriminant functional characterizing the distance from the discriminant (see \Cref{defn:DeltaFunctional} and \Cref{prop:KCharact}). For general smooth functions, controlling the approximation of the $C^2$--norm by its discretized counterpart would typically require bounds on the $C^3$--data. In the polynomial setting, however, Markov-type inequalities allow higher-order derivatives to be controlled in terms of lower-order ones, yielding a discretization scale depending only on the degree and the geometric parameters; see \Cref{sec:Discretization_Polynomial_Norms}. These bounds, together with the discretization estimates proved in \Cref{subsec:EffectComp}, allow one to replace the original class $\mathcal P_n(c,d,k)$ by the finite-constraint feasible set \begin{equation}\mathcal I_{\mathcal N}(c,d,k):=\left\{P\in\mathbb R[x_1,\dots,x_n]_{\le d}^c\ \middle|\ K_{\mathcal N}(P)\le k,\quad \n{P}_{\mathcal N}^2=1\right\}.\end{equation} Given observed samples $x^{(1)},\ldots,x^{(m)}\in D_R^n$, replacing the auxiliary risk by its empirical counterpart yields then the optimization problem \begin{equation}\label{intro_eq:EMPAuxiliaryOptimization} \inf_{P\in\mathcal I_{\mathcal N}(c,d,k)}\frac{1}{m}\sum_{i=1}^m\n{P(x^{(i)})}^2. \end{equation} The algorithmic treatment of \eqref{intro_eq:EMPAuxiliaryOptimization} and the analysis of the corresponding arithmetic complexity are the subject of Part II.

\subsection{Part II: Algorithmic Problem and Computational Complexity Bounds}
\label{subsec_OverII:Algorithmic_Problem}
The main difficulty in treating \eqref{intro_eq:EMPAuxiliaryOptimization} algorithmically is the potentially large dimension of the coefficient space over which the admissible systems are parametrized. Distinguishing among the alternatives in \Cref{Intro_thm:FinalMainTheorem} would, in principle, require checking the admissible polynomial systems in $\mathcal I_{\mathcal N}(c,d,k)$. However, the testing problem only requires determining whether there exists a system with error below the prescribed threshold $\varepsilon$. The strategy developed in Part II is to exploit this fact by localizing the search near suitable sublevel sets of the empirical cost, rather than performing an exhaustive search over the full coefficient space. Precisely, we are interested in the sublevel sets \begin{equation}\label{intro_def_sublevel}\mathcal U(\varepsilon^*):=\left\{P\in\mathbb R[x_1,\ldots,x_n]_{\leq d}^c\ \middle|\ \mathcal E(P)\leq\varepsilon^*\right\}.\end{equation}

Although the sets \eqref{intro_def_sublevel} are not explicitly described, their geometry can be analyzed effectively by representing polynomial systems through their coefficients with respect to a fixed basis. If $\mathcal B=\{b_1,\ldots,b_N\}$ is a basis of $\mathbb R[x_1,\ldots,x_n]_{\leq d}$, every polynomial system can be written as $$P(x)=AP_{\mathcal B}(x),\qquad P_{\mathcal B}(x):=(b_1(x),\ldots,b_N(x))^\top,$$ for a coefficient matrix $A\in\Mat(c\times N,\mathbb R)$. As shown in \Cref{sec:MatrixRep}, in these coordinates the empirical cost is the quadratic functional $$\mathcal E(A)=\operatorname{tr}(AVA^\top),$$ where $V\in\Sym(N,\mathbb R)$ is positive semidefinite.  

This representation allows us to localize the low-cost region through the spectral structure of $V$. Let $0\leq\varepsilon_1\leq\cdots\leq\varepsilon_N$ be the eigenvalues of $V$, and let $\mathcal V_r\subseteq\Mat(c\times N,\mathbb R)$ denote the low spectral subspace generated by the eigenspaces corresponding to the first $r$ eigenvalues. In \Cref{subsec:Unconstrained_Optimization_Problem} we show that, after identifying polynomial systems with their coefficient matrices, the sublevel set is contained in the tubular neighborhood $$\mathcal U(\varepsilon^*)\subseteq\operatorname{Tube}_{\mathcal N}\left(\mathcal V_r,\rho_r(\varepsilon^*)\right),\qquad \rho_r(\varepsilon^*):=M_{\mathcal N}\sqrt{\frac{\varepsilon^*}{\varepsilon_{r+1}}},$$ for every spectral index $r$, where $M_{\mathcal N}$ is the equivalence constant relating the discretized norm to the Frobenius norm associated with the chosen basis. 
In \Cref{subsec:LocalizedNet}, we derive explicit bounds for the covering numbers of these tubular neighborhoods, allowing the full-space search to be replaced by a search over the tube corresponding to a suitable spectral index $r^*$. The choice of $r^*$ is data-dependent through the eigenvalues of the empirical cost matrix $V$ and is made so as to minimize the metric entropy of the corresponding localized region. Depending on the observed spectrum, this localization may provide a substantial reduction of the search space, although no such reduction is guaranteed uniformly over all datasets for a fixed input threshold. For a fixed dataset, however, decreasing the input threshold $\varepsilon^*$ decreases the localization radius $\rho_{r^*}(\varepsilon^*)$ and eventually yields a reduction whenever a positive spectral level $\varepsilon_{r^*+1}$ is available. The resulting reduction in metric entropy is quantified explicitly in \Cref{subsec:LocalizedNet}, while a more detailed analysis of the spectral configurations of the data that lead to the largest computational gains is left for future work.

Once the localized region has been identified, \Cref{sec:Explicit_Net_Construction} provides an explicit construction of a proper $h$--net $\mathcal G_h(r,\rho)$ of the corresponding tube. Since retaining only centers satisfying the exact condition $K_{\mathcal N}(G)\leq k$ would in general destroy the covering property near the boundary of the admissible region, in \Cref{subsec:Admissible_Localized_Net} we introduce the relaxed-admissible grid $\mathcal G_h^+(r,\rho,k)$, whose centers satisfy $$K_{\mathcal N}(G)\leq k\Lambda(h),$$ with $\Lambda(h)\to1$ as $h\to0$. The algorithm of \Cref{alg:relaxed-global} then evaluates the empirical cost over this localized relaxed-admissible grid and selects a candidate minimizer. Its convergence is established in \Cref{thm:RelaxedAdmissibleGlobalOptimization}, while the resulting decision procedure and arithmetic complexity estimates yield \Cref{thm:RelaxedAdmissibleGlobalOptimizationFinal}.

\section{Preliminaries}

\subsection{The Bounded Geometry of Regression Manifolds}
A \textit{real algebraic manifold} is a smooth submanifold $Z\subseteq\mathbb R^n$ that can be realized as the common zero set of a system of finitely many real polynomial equations
$$
P:\quad
\begin{cases}
p_1(x)=0,\\
\vdots\\
p_c(x)=0.
\end{cases}
$$

In general, there is no relation between the dimension $l$ of $Z$ and the number of defining equations $c$, which may exceed the codimension $n-l$ of the underlying manifold. This representation complexity depends on how the manifold is embedded in $\mathbb R^n$, and in particular on the topology of its normal bundle, as encoded by its Stiefel--Whitney classes. Within this broader class, the subfamily of \textit{complete intersections} is of natural interest. These correspond to algebraic manifolds defined by regular systems with the minimal number $c=n-l$ of equations, namely systems satisfying $\rank\Jac P(x)=c$ for every $x\in Z(P)$.

Real algebraic manifolds admit a natural notion of \textit{degree bound}, which quantitatively controls their geometry. We say that a real algebraic manifold has \textit{degree at most} $d$ if it can be defined by polynomials whose degrees are all bounded by $d$ (see \cite{Yomdin1987, Gromov1999}, and \cite{BasuPollackRoy2006} for modern references). This notion of degree is standard in real algebraic geometry and Nash manifold theory, and can be interpreted as a complexity parameter for the defining system of equations, as well as for the associated algebraic manifold. In particular, degree bounds control the geometry of real algebraic manifolds: by Thom--Milnor's bound the number of their connected components is $O(d^n)$ \cite{milnor1964betti}.

The regularity of real algebraic manifolds can be further quantified through the \textit{condition number} of the defining system (see \Cref{subsec:RealAlgCompInt}). An upper bound on the condition number provides quantitative control of the geometry of the underlying manifold and, in particular, implies a positive lower bound on the reach of the associated zero set. This admits a natural interpretation when the system defines a regression manifold: a bound on the condition number of the defining equations prevents overfitting phenomena in the corresponding solutions.

\subsection{The Geometric Hypothesis in Statistical Learning} \label{subsec:prel_Statistical_Learning} Statistical learning (see \cite{VladimirVapnik,TongZhang} for standard references) aims to address the following problems. Given a random variable $X$ taking values on a Borel measurable set $\mathcal{X}\subseteq\mathbb{R}^n$, with law $\mu\in\mathcal{M}(\mathcal{X})$, and a class $\mathcal{H}$ of hypotheses or models describing some meaningful property of $X$, one aims to minimize a \textit{risk functional} $\mathcal{R}_\mu:\mathcal{H}\to\mathbb{R}^+$, defined as the expectation $$\mathcal{R}_\mu(h):=\int_{\mathcal{X}}\ell(x,h)\,d\mu(x),$$ for a given measurable \textit{loss function} $\ell:\mathcal{X}\times\mathcal{H}\to\mathbb{R}^+$ that quantifies the cost incurred by a hypothesis $h$ on a sample $x$. The measure $\mu$ is in general unknown, and the goal is to minimize $\mathcal{R}_\mu$ using only a finite sequence of independent and identically distributed (i.i.d.) samples drawn from $\mu$. Approximate minimizers of the problem can then be produced by a \textit{learning algorithm}. The precise definition of learning algorithm adopted in this work is the following. Consider the \textit{sample space} of the problem $$\mathcal{S}(\mathcal{X}):=\bigsqcup_{m=1}^{\infty}\mathcal{S}_m(\mathcal{X}),\qquad \text{where}\qquad \mathcal{S}_m(\mathcal{X}):=\mathcal{X}^m,$$ and let $\mu^m:=\mu\otimes\cdots\otimes\mu$ be the product measure on $\mathcal{S}_m(\mathcal{X})$. A \emph{learning algorithm} is a map $\mathcal{A}\colon\mathcal{S}(\mathcal{X})\to\mathcal{H}$ such that, for every $m\in\mathbb{N}$, the restriction $\mathcal{A}_m:=\mathcal{A}|_{\mathcal{S}_m(\mathcal{X})}$ satisfies the following property: the map $$F_m\big(x,(x^{(1)},\ldots,x^{(m)})\big):=\ell\big(x,\mathcal{A}_m(x^{(1)},\ldots,x^{(m)})\big)$$ is $\mu\otimes\mu^m$--measurable for every $\mu\in\mathcal{M}(\mathcal{X})$ and computable in the extended Blum--Shub--Smale model specified in \Cref{sec:AlgorithmicOptimizationFramework}.

A standard approach to defining learning algorithms is provided by the \textit{empirical risk minimization} (ERM) scheme, which consists in replacing the unknown risk functional $\mathcal{R}_\mu$ with the empirical risk $$\mathcal{R}_m(h):=\frac{1}{m}\sum_{i=1}^m\ell(x^{(i)},h).$$ An ERM algorithm then selects a hypothesis in $\mathcal{H}$ minimizing $\mathcal{R}_m$, thereby providing an approximation of an optimal hypothesis for the true risk. Under suitable assumptions on the hypothesis class, for a sufficiently large sample size $m$, minimizers of the empirical risk provide approximations of minimizers of the true risk (see \Cref{sec:IPRSamCompBounds}).

\medskip
\
\\
\medskip

\noindent\textbf{Acknowledgments.} The author would like to thank Antonio Lerario, his Ph.D. advisor, for his guidance, constant support, and many valuable discussions throughout the development of this work. In particular, the author is grateful to him for introducing several of the ideas and references concerning condition numbers and the testing problem, and for pointing out the characterization in \Cref{prop:KCharact}. Special thanks are also due to Matteo Testa for many stimulating discussions that helped shape and refine the direction of the research. The author also thanks Martina Scolamiero and Peter Roos Hoefgeest for insightful discussions on related questions explored in earlier joint work.

\medskip
\noindent\textbf{Declarations.}
Generative AI tools were used for language editing and brainstorming. The author takes full responsibility for the content, accuracy, and conclusions of this work.

\clearpage

\bigskip
\begin{center}
{\Large\bf Part I}\\[0.3cm]
{\large Geometric and Statistical Foundations}
\end{center}
\addcontentsline{toc}{part}{Part I: Geometric and Statistical Foundations}
\bigskip

\section{Hypothesis Class and Testing Problem}\label{sec:implicit_poly_reg}

While our results are stated in terms of random variables for notational convenience, the model is formulated for the class of compactly supported Borel probability measures on $\mathbb{R}^n$, denoted by $\mathcal{M}(\mathbb{R}^n)$; additionally, to avoid boundary phenomena, we restrict our attention to measures whose support is quantitatively separated from the boundary of the domain. For a fixed radius $R>0$ and a separation parameter $0<\tau<R$, we define the class $$\mathcal{M}_{\tau}(D_R^n):=\left\{\mu\in\mathcal{M}(\mathbb{R}^n)\ \middle|\ \operatorname{supp}(\mu)\subseteq D_{R-\tau}^n\right\},$$ where $D_R^n\subseteq\mathbb{R}^n$ denotes the closed disk of radius $R$. In particular, $R$ may be taken arbitrarily large, so the support condition above imposes no genuine restriction on the generality of the setting. Moreover, although the constants appearing in our main results may depend on the geometric parameters $\tau$ and $R$, these quantities can be regarded as fixed, and their dependence will be absorbed into the constants as needed.

\subsection{Algebraic Complete Intersections and Condition Number}\label{subsec:RealAlgCompInt}

For any $F\in C^{\infty}(\mathbb{R}^n,\mathbb{R}^c)$, we define the $C^1$-- and $C^2$--seminorms on $D_R^n$ by $$\n{F}_{C^1(D_R^n)}:=\max_{x\in D_R^n}\sqrt{\n{F(x)}^2+\n{\operatorname{Jac}F(x)}_{\operatorname{op}}^2}$$ and $$\n{F}_{C^2(D_R^n)}:=\max_{x\in D_R^n}\sqrt{\n{F(x)}^2+\n{\operatorname{Jac}F(x)}_{\operatorname{op}}^2+\n{D^2F(x)}_{\operatorname{op}}^2},$$ where $\n{D^2F(x)}_{\operatorname{op}}$ denotes the operator norm of the second differential at $x$, viewed as a bilinear map from $\mathbb{R}^n\times\mathbb{R}^n$ to $\mathbb{R}^c$. Let $\mathbb{R}[x_1,\dots,x_n]_{\le d}^c$ denote the vector space of $c$--tuples of real polynomials in $n$ variables of degree at most $d$, naturally embedded in $C^{\infty}(\mathbb{R}^n,\mathbb{R}^c)$. Although the quantities above are only seminorms on $C^{\infty}(\mathbb{R}^n,\mathbb{R}^c)$, their restrictions to $\mathbb{R}[x_1,\dots,x_n]_{\le d}^c$ are norms, since a polynomial vanishing on $D_R^n$ must vanish identically. 

While all norms on this finite-dimensional space are equivalent, our specific choice is motivated by the dimensional dependence of the corresponding equivalence constants: replacing it by a different equivalent norm may introduce an additional dependence on the ambient dimension into the constants appearing in our main results.

We now introduce the notion of discriminant adopted in this paper,  and the associated condition number. Geometrically, the discriminant consists of those maps whose zero set fails to be a smooth complete intersection in $D_R^n$. While various definitions exist in the literature, we adopt the one from \cite{LerarioStecconi}.

\begin{defn} \label{defn:one}
For $c < n$, the discriminant in $D_R^n$ is the set $\Delta(D_R^n)$ defined by 
$$\Delta(D_R^n) = \Big\{F \in C^2(\mathbb{R}^n, \mathbb{R}^c) \mid \exists x \in D_R^n \text{ s.t. } F(x)=0, \ \operatorname{rank}(\operatorname{Jac} F(x)) < c \Big\}.$$
\end{defn}

In particular, if $F$ does not belong to $\Delta(D_R^n)$, it is said to be \textit{regular} on $D_R^n$. The regularity of a system of equations $F=0$ could be quantified through the distance \begin{equation}\label{eq:DistDiscr} \operatorname{dist}(F, \Delta(D_R^n)) = \inf_{G \in \Delta(D_R^n)} \|F - G\|_{C^1(D_R^n)}.\end{equation} However, the discriminant is a cone in $C^2(\mathbb{R}^n, \mathbb{R}^c)$, since $\lambda F \in \Delta(D_R^n)$ for any $F \in \Delta(D_R^n)$ and $\lambda \in \mathbb{R}$. Moreover, the set $Z(F)$ of the solutions of $F=0$ is invariant under non zero scaling of $F$. It is then more natural to consider a relative version of this quantity, called the \textit{condition number}, defined as follows.

\begin{defn}\label{defn:cond_num}
The $C^2$--condition number of a regular function $F \in C^2(\R^n, \R^c)$ on $D_R^n$ is defined by
$$K_{D_R^n}(F) := \frac{\|F\|_{C^2(D_R^n)}}{\dist(F, \Delta(D_R^n))}.$$
\end{defn}

We now introduce the families of polynomials that will be central to our subsequent constructions. Consider the set  $$\mathcal{P}_n(c, d, k) := \left\{ P \in \mathbb{R}[x_1, \dots, x_n]_{\le d}^c \ \big| \ \|P\|_{C^2(D_R^n)} = 1, \,\, K_{D_R^n}(P) \le k \right\},$$  and the subfamily $\mathcal{P}_n^*(c, d, k) \subseteq \mathcal{P}_n(c, d, k)$ consisting of polynomials that additionally satisfy $Z(P) \cap \operatorname{int}(D_R^n) \neq \emptyset$, where $\operatorname{int}(D_R^n)$ denotes the interior of the disk. We are now ready to provide the formal definition of the hypothesis class considered in this work.

\begin{defn}
We define the hypothesis class $\mathcal{H}_n(l, d, k)$ of \textit{real algebraic complete intersections} of dimension $1 \le l \le n-1$, degree at most $d \ge 1$, and condition number at most $k$, as the family $$\mathcal{H}_n(l, d, k) := \big\{ Z(P) \cap D_R^n \mid P \in \mathcal{P}^*_n(c, d, k), \ c = n - l \big\}.$$
\end{defn}
Note that elements of $\mathcal{H}_n(l,d,k)$ are not necessarily embedded submanifolds of $D_R^n$ viewed as a manifold with boundary, since we do not require them to be transverse to $\partial D_R^n$. However, the condition $Z(P)\cap\operatorname{int}(D_R^n)\neq\emptyset$ for $P\in\mathcal{P}_n^*(c,d,k)$, together with the condition number bound, ensures that the intersection $Z(P)\cap\operatorname{int}(D_R^n)$ is an embedded submanifold of $\operatorname{int}(D_R^n)$.

We can now describe the testing problem formally. Let $\mathcal{K}(D_R^n)$ denote the family of non-empty compact subsets of $D_R^n$. The distance from a point $x$ to a closed set $Z \subseteq D_R^n$ is defined by  $$\operatorname{dist}(x,Z) := \min_{z \in Z} \|x-z\|.$$  Recalling the definition of a geometric hypothesis from the introduction, we adapt it here to the family of measures $\mathcal{M}_{\tau}(D_R^n)$ and the hypothesis class $\mathcal{H}_n(l, d, k)$.

\begin{defn}\label{defn:GeomHyp}
Let $X$ be a random variable with law $\mu \in \mathcal{M}_{\tau}(D^n_R)$. Given $\varepsilon>0$, we say that $X$ satisfies the $\varepsilon$-\textit{geometric hypothesis} with respect to the class $\mathcal{H}_n(l, d, k)\subseteq \mathcal{K}(D^n_R)$ if there exists $Z \in \mathcal{H}_n(l, d, k)$ such that 
\begin{equation}
\mathbb{E}[\dist(X,Z)^2] < \varepsilon,
\label{def:GeomHyp_eq}
\end{equation}

\end{defn}

Observe that this definition is well-posed: indeed, since $Z$ is compact, the map $x \mapsto \operatorname{dist}(x,Z)$ is continuous. Consequently, the integral  $$\mathbb{E}[\operatorname{dist}(X,Z)^2] = \int_{D_R^n} \operatorname{dist}(x,Z)^2 \, d\mu(x)$$  is well-defined. This allows us to introduce the risk functional, which serves as the core metric for evaluating the fit of a candidate geometry to the data-generating measure.

\begin{defn}
Given $\mu \in \mathcal{M}_{\tau}(D_R^n)$, we define the risk functional $\mathcal{R}_{\mu} : \mathcal{H}_n(l, d, k) \to \mathbb{R}$ as 
$$\mathcal{R}_{\mu}(Z) := \int_{D_R^n} \operatorname{dist}(x, Z)^2 d\mu(x).$$
\end{defn}
We observe that determining whether a measure $\mu$ satisfies the manifold hypothesis for the class $\mathcal{H}_n(l,d,k)$ is equivalent to evaluating the quantity
\begin{equation}\label{CostInf}
\inf_{Z \in \mathcal{H}_n(l,d,k)} \mathcal{R}_{\mu}(Z) = \inf_{P \in \mathcal{P}_n^*(c,d,k)} \mathcal{R}_{\mu}(Z(P)\cap D^n_R),
\end{equation} 
and subsequently comparing the threshold $\varepsilon$ with the resulting value. As argued in the introduction, while $\mathcal{P}_n(c, d, k)$ can be described---after fixing a basis for the space of polynomials---by imposing norm and condition number constraints on the coefficients, no equally explicit description exists for $\mathcal{P}_n^*(c, d, k)$ in terms of coefficients.  The goal of next \Cref{sec:RegComInt} is to study the geometric properties of the elements in $\mathcal{P}_n(c,d,k)$ in order to bridge this gap.

\section{Geometric Properties of Complete Intersections}
\label{sec:RegComInt}

In this section, we establish the two main geometric results for regular systems of equations, namely \Cref{thm:HdZeroes} and \Cref{thm:UnifDistBound}. These results are of independent interest and apply not only to polynomial systems, but more generally to maps of class $C^2(\mathbb{R}^n,\mathbb{R}^c)$ with $c\le n$.

\subsection{A Zero Criterion for Regular Maps}
\label{subsec:AZeroCriterion}

The distance \eqref{eq:DistDiscr} from the discriminant admits a characterization in terms of the following discriminant functional. We first recall the notation for singular values. Given a matrix $A\in\operatorname{Mat}(c\times n,\mathbb R)$ with $c\le n$, the $c\times c$ matrix $AA^\top$ is symmetric and positive semidefinite. By the spectral theorem, its eigenvalues are non-negative; the singular values $\sigma_1(A)\ge\dots\ge\sigma_c(A)\ge0$ of $A$ are defined as the square roots of these eigenvalues. In particular, we denote by $\sigma_{\min}(A):=\sigma_c(A)$ the smallest singular value of $A$.

\begin{defn}\label{defn:DeltaFunctional}
Following \cite{LerarioStecconi}, for any compact set $E \subseteq \mathbb{R}^n$, we define the \textit{discriminant functional} $\delta_E : C^2(\mathbb{R}^n, \mathbb{R}^c) \to \mathbb{R}^+$ by
\begin{equation}\label{def:DeltaFunctional}
\delta_E(F) := \min_{x \in E} \sqrt{\|F(x)\|^2 + \sigma_{\min}(\operatorname{Jac} F(x))^2}.
\end{equation}
\end{defn}
The following \Cref{prop:KCharact} shows that the $C^1$--distance of a function $F$ from the discriminant $\Delta(D_R^n)$ can be characterized purely in terms of its $C^1$--data, namely the values of $F$ and its Jacobian on $D_R^n$. This admits a clear geometric interpretation: at a regular point $x\in Z(F)$, the normal space to the zero set is spanned by the gradients of the components of $F$, which correspond to the rows of $\operatorname{Jac}F(x)$. The smallest singular value $\sigma_{\min}(\operatorname{Jac}F(x))$ therefore measures the degree of linear independence of these gradients, and hence quantifies the transversality of the intersection of the zero sets of the components of $F$. For the remainder of this section, we omit the domain $D_R^n$ from the $C^1$-- and $C^2$--norm subscripts \emph{within the proofs}, writing $\|\cdot\|_{C^1}$ and $\|\cdot\|_{C^2}$ in place of $\|\cdot\|_{C^1(D_R^n)}$ and $\|\cdot\|_{C^2(D_R^n)}$, respectively. The dependence on the domain is always made explicit in the statements.

\begin{prop}[Characterization]\label{prop:KCharact}
For any $F \in C^2(\R^n,\R^c)$ where $c\le n$, we have $$ \delta_{D^n_R}(F) =  \dist(F, \Delta(D^n_R)).$$
\end{prop}
\begin{proof}
\medskip
\noindent
\textbf{I)} Let $G \in \Delta(D_R^n)$ be arbitrary. By Definition \ref{defn:one}, there exists a point $y \in D_R^n$ such that:
$$ G(y) = 0 \quad \text{and} \quad \operatorname{rank}(\operatorname{Jac}G(y)) < c. $$ Therefore, by the definition of the $C^1$ norm: $$ \|F - G\|_{C^1} = \max_{x \in D_R^n} \sqrt{\|F(x) - G(x)\|^2 + \|\operatorname{Jac}F(x) - \operatorname{Jac}G(x)\|_{\mathrm{op}}^2}. $$ Since the maximum is greater than or equal to the value at $y$, and using the fact that $G(y) = 0$:
$$
\begin{aligned}
 \|F - G\|_{C^1}^2 &\ge \|F(y) - G(y)\|^2 + \|\operatorname{Jac}F(y) - \operatorname{Jac}G(y)\|_{\mathrm{op}}^2. \\
 &=  \|F(y)\|^2 + \|\operatorname{Jac}F(y) - \operatorname{Jac}G(y)\|_{\mathrm{op}}^2.
\end{aligned}
$$
Since $\operatorname{rank}(\operatorname{Jac}G(y)) < c$, the matrix $\operatorname{Jac}G(y)$ is singular. By the Eckart--Young--Mirsky theorem for unitarily invariant norms (see \cite[Theorem~7.4.51]{MatrixAnalysisJohnson}), the distance from a matrix $A$ to the set of matrices with rank strictly less than $c$ in the operator norm is equal to its smallest singular value, $\sigma_{\min}(A)$. Therefore:
$$ \|\operatorname{Jac}F(y) - \operatorname{Jac}G(y)\|_{\mathrm{op}} \ge \sigma_{\min}(\operatorname{Jac}F(y)). $$
Substituting these into the inequality:
$$
\begin{aligned}
 \|F - G\|_{C^1}^2& \ge \|F(y)\|^2 + \sigma_{\min}(\operatorname{Jac}F(y))^2.\\
 &\ge \min_{x \in D_R^n} \left( \|F(x)\|^2 + \sigma_{\min}(\operatorname{Jac}F(x))^2 \right) = \delta_{D_R^n}(F)^2.
\end{aligned}
$$
That gives the first inequality.
\\
\\
\textbf{II)} We prove the upper bound on $\operatorname{dist}(F, \Delta(D_R^n))$. We observe that it is enough to provide, for any $F\in C^2(\R^n,\R^c)$, a perturbation $h$ such that: $$F - h \in \Delta(D_R^n) \quad \text{and} \quad \|h\|_{C^1} \le \delta_{D_R^n}(F).$$ Since $D_R^n$ is compact and $F$ is $C^2$, the function $$x \mapsto \|F(x)\|^2 + \sigma_{\min}(\Jac F(x))^2$$ attains its minimum at some point $x^* \in D_R^n$. Thus: $$ \delta_{D_R^n}(F) = \sqrt{\|F(x^*)\|^2 + \sigma_{\min}(\operatorname{Jac}F(x^*))^2}. $$ Define then $$u = F(x^*) \in \mathbb{R}^c \quad \text{and} \quad J = \operatorname{Jac}F(x^*) \in \mathbb{R}^{c \times n}.$$ Let $w \in \mathbb{R}^n$ (with $\|w\|=1$) be a right singular vector of $J$ corresponding to the smallest singular value $\sigma_{\min}(J)$, such that $\|Jw\| = \sigma_{\min}(J)$, and let $z = Jw \in \mathbb{R}^c$. We define the perturbation $$h(x) = \cos(w^T(x-x^*)) u + \sin(w^T(x-x^*)) z,$$ and $G = F - h$. We now prove that $G \in \Delta(D_R^n)$ and we compute $\|h\|_{C^1}$.
\\
\\
a) A straightforward computation shows that
$$h(x^*) = u = F(x^*) \quad \text{and} \quad \Jac h(x^*)=z w^T.$$
Therefore $G(x^*) = 0$ and $$\operatorname{Jac}G(x^*) = J - \operatorname{Jac}h(x^*) = J - z w^T.$$ Recall that $z = Jw$ where $w$ is the right singular vector of $J = \operatorname{Jac}F(x^*)$ corresponding to the smallest singular value $\sigma_{\min}$. Using the Singular Value Decomposition (SVD), we can write $$J = \sum_{i=1}^c \sigma_i u_i v_i^T \quad \text{and} \quad z w^T = (\sigma_c u_c) v_c^T = \sigma_c u_c v_c^T,$$ where $\sigma_c = \sigma_{\min}$. Therefore $$\Jac G(x^*) = \sum_{i=1}^{c-1} \sigma_i u_i v_i^T.$$ The resulting matrix has rank at most $c-1$, thus $G \in \Delta(D_R^n)$.
\\
\\
b) We now compute the norm of $h$. Let $\theta = w^T(x-x^*)$. A straightforward computation (note that $\operatorname{Jac}h(x)$ has rank one) shows that $$ \|h(x)\|^2 + \|\operatorname{Jac}h(x)\|_{\mathrm{op}}^2 = (\cos^2\theta+\sin^2\theta)(\|u\|^2 + \|z\|^2) = \|u\|^2 + \|z\|^2. $$ This value is constant for all $x$, so: $$ \|h\|_{C^1} = \sqrt{\|u\|^2 + \|z\|^2} = \sqrt{\|F(x^*)\|^2 + \sigma_{\min}(\Jac F(x^*))^2} = \delta_{D_R^n}(F). $$ Thus, there exists $G \in \Delta(D_R^n)$ with $\|F-G\|_{C^1} = \delta_{D_R^n}(F)$, proving the upper bound.
\end{proof}

The characterization provided by \Cref{prop:KCharact} plays a central role in both the subsequent proofs. Since the elements of $\mathcal{P}_n(c,d,k)$ have unit $C^2$--norm by definition, \Cref{prop:KCharact} allows one to reformulate the class as:  $$\mathcal{P}_n(c,d,k)=\left\{ P \in \mathbb{R}[x_1,\ldots,x_n]_{\le d}^c \ \middle|\ \delta_{D^n_R}(P) \ge \frac{1}{k}, \ \|P\|_{C^2}=1 \right\}.$$ 

Given this characterization, it is natural to investigate how a lower bound on the discriminant functional $\delta_{D_R^n}(P)$ controls the local geometry of the zero set. We establish here a quantitative criterion ensuring the existence of a zero under such a uniform non-degeneracy assumption. The argument is based on a localized Newton iteration combined with a uniform lower bound on the smallest singular value of the Jacobian. The Newton method is a local iterative scheme for solving $F(x)=0$ by successively solving the affine approximation $$F(x_j) + \operatorname{Jac} F(x_j)(x - x_j) = 0.$$ 

While the Jacobian in our setting is not square, a uniform lower bound on its smallest singular value serves as a quantitative substitute for invertibility. This allows for an analogous scheme that converges quadratically to the zero set $Z(F)$, provided the initial guess is sufficiently close, for which the radius of convergence can be established quantitatively. 

To this end, we first state the following \Cref{lem:RightInverseBound}:

\begin{lem}\label{lem:RightInverseBound}
Let $A \in \operatorname{Mat}(c\times n, \mathbb{R})$ with $c \le n$, and assume $\sigma_{\min}(A) \ge m > 0$. Then there exists a matrix $L \in \operatorname{Mat}(n\times c, \mathbb{R})$ such that $$AL = I_{\mathbb{R}^c} \qquad \text{and} \qquad  \|L\|_{\mathrm{op}} \le \frac{1}{m}.$$
\end{lem}

\begin{proof}
Since $\sigma_{\min}(A) \ge m > 0$, the matrix $A$ has full rank $c$, and its restriction to $(\ker A)^\perp$ is an isomorphism onto $\mathbb{R}^c$. We define the right inverse $L: \mathbb{R}^c \to (\ker A)^\perp \subseteq \mathbb{R}^n$ as the inverse of this restriction. For any $b \in \mathbb{R}^c$, by the definition of the smallest singular value, we have: $$\|b\| = \|A (Lb)\| \ge \sigma_{\min}(A) \|Lb\| \ge m \|Lb\|.$$ It follows that $\|Lb\| \le \frac{1}{m} \|b\|$, which implies $\|L\|_{\mathrm{op}} \le \frac{1}{m}$.
\end{proof}

We also need the following local stability estimate for the smallest singular value of the Jacobian.

\begin{prop}\label{prop:SVContinuity} Let $B_r(y)\subseteq D^n_R$ be an open ball centered at $y \in D^n_R$. Given $F \in C^2(\mathbb{R}^n, \mathbb{R}^c)$, assume $\sigma_{\min}(\Jac F(y)) \ge \delta$. Then $$\inf_{x \in B_r(y)} \sigma_{\min}(\Jac F(x)) \ge \frac{\delta}{2} \quad \text{provided} \quad r \le \frac{\delta}{2 \|F\|_{C^2(D^n_R)}}.$$ \end{prop}

\begin{proof} Using the perturbation estimate $|\sigma_{\min}(A) - \sigma_{\min}(B)| \le \|A - B\|_{\op}$, for any $x \in B_r(y)$ we have $$\sigma_{\min}(\Jac F(x)) \ge \sigma_{\min}(\Jac F(y)) - \|\Jac F(x) - \Jac F(y)\|_{\op} \ge \delta - \|F\|_{C^2(D^n_R)}\, r.$$ Thus, $$\sigma_{\min}(\Jac F(x)) \ge \frac{\delta}{2} \quad \text{provided} \quad r \le \frac{\delta}{2 \|F\|_{C^2(D^n_R)}}.$$ \end{proof}

Combining the bounded right inverse with \Cref{prop:SVContinuity}, we obtain the following quantitative Newton criterion for the existence of a nearby zero (see \cite{DeuflhardNewton} for an overview).

\begin{thm}\label{thm:HdZeroes}
Let $F\in C^2(\mathbb{R}^n,\mathbb{R}^c)$ and let $k$ be its $C^2$--condition number on $D_R^n$. For any $\tau>0$ and any $y\in D_R^n$ such that $\operatorname{dist}(y,\partial D_R^n)\ge\tau$, there exists a radius $r(k,\tau)>0$ such that, for every $r\le r(k,\tau)$ and every $\delta>0$, there exists a constant $C(\delta,k,r)>0$ with the following property: if $$\n{F(y)}^2\le C(\delta,k,r)\qquad\text{and}\qquad\delta_{D_R^n}(F)\ge\delta,$$ then $Z(F)\cap B_r(y)\neq\emptyset$.
\end{thm}
\begin{proof}
Set $\delta_F:=\delta_{D_R^n}(F)$ and consider the radius $$r(k,\tau):=\min\left\{\frac{\tau}{2},\frac{1}{4\sqrt{2}k}\right\}$$ and fix any $r\le r(k,\tau)$. We claim that a suitable constant is given by $$C(\delta,k,r):=\min\left\{\frac{\delta^2}{64k^2},\frac{\delta^2r^2}{128}\right\}.$$

\noindent
\textbf{I)} We first note that, with this choice of $r(k,\tau)$ and $C(\delta,k,r)$, it holds that $$\sigma_{\min}(\operatorname{Jac}F(x))\ge\frac{\delta_F}{2\sqrt{2}}\ge\frac{\delta}{2\sqrt{2}}\qquad\text{for all }x\in B_r(y).$$ Indeed, since $\delta_F\ge\delta$, if $F$ satisfies $\n{F(y)}^2\le C(\delta,k,r)$, then in particular $\n{F(y)}^2\le\delta_F^2/2$. By the definition of the discriminant functional, this implies $$\n{F(y)}^2+\sigma_{\min}(\operatorname{Jac}F(y))^2\ge\delta_F^2\quad\Rightarrow\quad\sigma_{\min}(\operatorname{Jac}F(y))\ge\frac{\delta_F}{\sqrt{2}}.$$ Applying \Cref{prop:SVContinuity}, we deduce that for any $x\in B_r(y)$, we have $$\sigma_{\min}(\operatorname{Jac}F(x))\ge\frac{\delta_F}{2\sqrt{2}}\qquad\text{provided}\qquad r\le\frac{\delta_F}{2\sqrt{2}\n{F}_{C^2}}=\frac{1}{2\sqrt{2}k},$$ which follows from our choice of $r(k,\tau)$. In particular, $\operatorname{Jac}F(x)$ is surjective for all $x\in B_r(y)$.

\noindent
\textbf{II)} We now show that the Newton method yields the existence of a zero. Let $m:=\frac{\delta_F}{2\sqrt{2}}$. By \Cref{lem:RightInverseBound}, for each $x\in B_r(y)$ there exists a linear map $L_x:\mathbb{R}^c\to\mathbb{R}^n$ such that $$\operatorname{Jac}F(x)L_x=I_{\mathbb{R}^c}\qquad\text{and}\qquad\n{L_x}_{\mathrm{op}}\le\frac{1}{m}.$$ Define the Newton map $N(x):=x-L_xF(x)$, and consider the sequence $\{x_j\}$ defined by $x_0:=y$ and $x_{j+1}:=N(x_j)$. We can bound the step size by $$\n{x_{j+1}-x_j}=\n{L_{x_j}F(x_j)}\le\frac{1}{m}\n{F(x_j)}.$$ By Taylor expansion, there exists a remainder term $R(x_j,x_{j+1})\in\mathbb{R}^c$ such that $$F(x_{j+1})=F(x_j)+\operatorname{Jac}F(x_j)(x_{j+1}-x_j)+R(x_j,x_{j+1}),$$ satisfying $$\n{R(x_j,x_{j+1})}\le\frac{\n{F}_{C^2}}{2}\n{x_{j+1}-x_j}^2.$$ Substituting $x_{j+1}-x_j=-L_{x_j}F(x_j)$ into the expansion, we obtain $$F(x_{j+1})=F(x_j)-\operatorname{Jac}F(x_j)L_{x_j}F(x_j)+R(x_j,x_{j+1}).$$ Since $\operatorname{Jac}F(x_j)L_{x_j}=I_{\mathbb{R}^c}$, the first two terms cancel, yielding $F(x_{j+1})=R(x_j,x_{j+1})$. Hence:
$$
\begin{aligned}
\n{F(x_{j+1})}=\n{R(x_j,x_{j+1})}
&\le\frac{\n{F}_{C^2}}{2}\n{x_{j+1}-x_j}^2\\
&\le\frac{\n{F}_{C^2}}{2}\n{L_{x_j}F(x_j)}^2\\
&\le\frac{\n{F}_{C^2}}{2m^2}\n{F(x_j)}^2.
\end{aligned}
$$
Note that, since $m=\delta_F/(2\sqrt{2})$ and $\n{F}_{C^2}=k\delta_F$, one has $$\frac{\n{F}_{C^2}}{2m^2}\n{F(y)}=\frac{4k}{\delta_F}\n{F(y)}\le\frac{1}{2}$$ provided $$\n{F(y)}^2\le\frac{\delta_F^2}{64k^2},$$ which follows from $\n{F(y)}^2\le C(\delta,k,r)$ and $\delta_F\ge\delta$. We now prove simultaneously by induction that \begin{equation}\label{eq:ConvToZero} x_j\in B_r(y)\qquad\text{and}\qquad \n{F(x_j)}\le\frac{\n{F(y)}}{2^j} \end{equation} for every $j\ge0$. The claim is immediate for $j=0$. Assume that it holds for some $j\ge0$. Since $x_j\in B_r(y)$, the preceding Newton estimate applies and gives $$\n{F(x_{j+1})}\le\frac12\n{F(x_j)}\le\frac{\n{F(y)}}{2^{j+1}}.$$ Moreover, using the induction hypothesis for the preceding iterates, $$\n{x_{j+1}-y}\le\sum_{i=0}^j\n{x_{i+1}-x_i}\le\frac{1}{m}\sum_{i=0}^j\n{F(x_i)}\le\frac{1}{m}\sum_{i=0}^j\frac{\n{F(y)}}{2^i}\le\frac{2\n{F(y)}}{m}=\frac{4\sqrt{2}\n{F(y)}}{\delta_F}.$$ Since $\n{F(y)}^2\le C(\delta,k,r)\le\delta^2r^2/128$ and $\delta_F\ge\delta$, we obtain $$\n{x_{j+1}-y}\le\frac r2<r.$$ Hence $x_{j+1}\in B_r(y)$, completing the induction. Moreover, the step estimates show that $\{x_j\}$ is a Cauchy sequence, and hence it converges to some $x^*\in\overline{B}_{r/2}(y)\subset B_r(y)$. By the continuity of $F$ and \eqref{eq:ConvToZero}, we conclude $F(x^*)=0$, thus $x^*\in Z(F)\cap B_r(y)$.
\end{proof}

\subsection{A Uniform \L ojasiewicz-Type Inequality}
\label{sec:LowIneq}

We derive from \Cref{thm:HdZeroes} a uniform \L{}ojasiewicz-type inequality for $C^2$--maps with bounded condition number. Recall that, in its classical analytic form, a \L{}ojasiewicz inequality states that if $F:\mathbb{R}^n\to\mathbb{R}$ is real analytic and $x\in Z(F)$, then there exist a constant $C>0$ and an exponent $\alpha>0$, depending on $x$ and on a sufficiently small neighborhood, such that
$$
C\operatorname{dist}(x',Z(F))\le |F(x')|^\alpha
$$
for all $x'$ in that neighborhood. For algebraic $F$, the constant and exponent can be chosen uniformly with respect to $x$ on a fixed compact set, although they may still depend on $F$.

Here, a uniform lower bound on the discriminant functional yields instead a constant that is uniform over classes of $C^2$--maps with bounded condition number and, in particular, over $\mathcal P_n(c,d,k)$. More precisely, the constant depends only on the geometry of the domain and on the quantitative bounds on $F$, rather than on the particular map or on the location of a zero.

This non-degeneracy also makes it possible to construct a normalized Newton flow in a neighborhood of the zero set and to show that it reaches $Z(F)$ in finite time. Along its trajectories, the residual $\|F\|$ decreases at unit speed, while the trajectory remains quantitatively controlled inside the ambient domain. These properties are summarized in the following lemma.

\begin{lem}\label{lem:LocalFlowProps} Let $D_R^n\subseteq\mathbb R^n$ be the disk of radius $R$ and let $0<\tau<R$. Suppose that $F\in C^2(\mathbb R^n,\mathbb R^c)$ satisfies \begin{equation}\label{GradFlowHyp}\delta_{D_R^n}(F)\ge\delta>0\quad\text{and}\quad Z(F)\cap\operatorname{int}(D_R^n)\neq\emptyset,\end{equation} and let $k$ be the condition number of $F$. Then there exists a radius $r(\tau,k)>0$ such that, for any $0<r\le r(\tau,k)$ and any $$x\in\bigl(U_r(Z(F))\setminus Z(F)\bigr)\cap D_R^n \qquad \text{satisfying} \qquad \dist(x,\partial D_R^n)\ge\tau,$$ there exists a normalized flow $\Phi_t$ starting at $x$ that satisfies: \begin{enumerate} \item The flow extends continuously to a time $T>0$ at which it reaches the zero set, i.e., $\Phi_T(x)\in Z(F)$, and in fact $T=\n{F(x)}$. \item For every $t\in[0,T]$, the trajectory remains contained in the disk and in a controlled neighborhood of the zero set, namely $$\Phi_t(x)\in U_{\frac{2kr}{\sqrt3}}(Z(F))\cap D_R^n.$$ \end{enumerate} \end{lem}

\begin{proof} Choose $$r(\tau,k):=\min\left\{\frac{1}{2k},\frac{\sqrt3\,\tau}{4k}\right\}$$ and set $\delta_F:=\delta_{D_R^n}(F)$.

\medskip \noindent\textbf{I)} We begin by introducing the normalized Newton flow associated with $F$. Its definition requires the Jacobian $\Jac F(y)$ to be surjective, so that a canonical right inverse is well defined. We therefore first show that the initial point lies in a region where this surjectivity holds. Fix $0<r\le r(\tau,k)$ and let $$x\in\bigl(U_r(Z(F))\setminus Z(F)\bigr)\cap D_R^n$$ satisfy $\dist(x,\partial D_R^n)\ge\tau$. Since $x\in U_r(Z(F))$, there exists $z\in Z(F)$ such that $\n{x-z}<r$. Since $k\ge1$, we have $$r\le\frac{\sqrt3\,\tau}{4k}<\tau\le\dist(x,\partial D_R^n).$$ Hence $z\in B_\tau(x)\subseteq\operatorname{int}(D_R^n)$. Since $D_R^n$ is convex and $F(z)=0$, we obtain $$\n{F(x)}=\n{F(x)-F(z)}\le\n{F}_{C^1(D_R^n)}\n{x-z}<\n{F}_{C^2(D_R^n)}r=k\delta_Fr.$$ Since $r\le1/(2k)$, it follows that \begin{equation}\label{eq:NormalizedNewtonInitialBound}\n{F(x)}<\frac{\delta_F}{2}.\end{equation} For every $y\in D_R^n$ such that $\n{F(y)}<\delta_F$, the definition of the discriminant functional gives $$\n{F(y)}^2+\sigma_{\min}(\Jac F(y))^2\ge\delta_F^2,$$ and therefore \begin{equation}\label{eq:NormalizedNewtonSingularValue}\sigma_{\min}(\Jac F(y))\ge\sqrt{\delta_F^2-\n{F(y)}^2}>0.\end{equation} Thus $\Jac F(y)$ is surjective on the open set $$\Omega:=\left\{y\in\operatorname{int}(D_R^n)\ \middle|\ 0<\n{F(y)}<\delta_F\right\},$$ and \eqref{eq:NormalizedNewtonInitialBound} shows that $x\in\Omega$. For every $y\in\Omega$, we can therefore define the canonical right inverse of $\Jac F(y)$ by $$L_F(y):=\Jac F(y)^\top\bigl(\Jac F(y)\Jac F(y)^\top\bigr)^{-1},$$ which satisfies \begin{equation}\label{eq:NormalizedNewtonRightInverse}\Jac F(y)L_F(y)=I_{\mathbb R^c}\qquad\text{and}\qquad\n{L_F(y)}_{\operatorname{op}}=\frac{1}{\sigma_{\min}(\Jac F(y))}.\end{equation} We now define the normalized Newton vector field $V:\Omega\to\mathbb R^n$ by $$V(y):=-L_F(y)\frac{F(y)}{\n{F(y)}}.$$ Since $F\in C^2$, the vector field $V$ is of class $C^1$ on $\Omega$, and hence locally Lipschitz. Therefore, by the standard existence and uniqueness theorem for ordinary differential equations, there exists a unique maximal integral curve $\phi_x:[0,T_{\max})\to\Omega$ satisfying \begin{equation}\label{eq:Newton_Flow}\dot\phi_x(t)=V(\phi_x(t))=-L_F(\phi_x(t))\frac{F(\phi_x(t))}{\n{F(\phi_x(t))}},\qquad\phi_x(0)=x.\end{equation}

\medskip \noindent\textbf{II)} We now derive quantitative estimates along the maximal solution $\phi_x$. As long as the solution is defined, by \eqref{eq:NormalizedNewtonRightInverse} we have $$\frac{d}{dt}F(\phi_x(t))=\Jac F(\phi_x(t))\dot\phi_x(t)=-\Jac F(\phi_x(t))L_F(\phi_x(t))\frac{F(\phi_x(t))}{\n{F(\phi_x(t))}}=-\frac{F(\phi_x(t))}{\n{F(\phi_x(t))}}.$$ Setting $Y(t):=F(\phi_x(t))$, this reads $$Y'(t)=-\frac{Y(t)}{\n{Y(t)}}.$$ Since $\phi_x(t)\in\Omega$ for every $0\le t<T_{\max}$, we have $Y(t)\neq0$, and therefore $$\frac{d}{dt}\n{F(\phi_x(t))}=\frac{d}{dt}\n{Y(t)}=\frac{\langle Y(t),Y'(t)\rangle}{\n{Y(t)}}=-1.$$ Setting $T:=\n{F(x)}>0$, and using $\phi_x(0)=x$, we obtain \begin{equation}\label{eq:NormalizedNewtonExactDecay}\n{F(\phi_x(t))}=T-t\end{equation} for every $0\le t<T_{\max}$. In particular, $T_{\max}\le T$. Moreover, the direction of $Y(t)$ is constant along the flow. Indeed, for every $0\le t<T_{\max}$, $$\begin{aligned}\frac{d}{dt}\left(\frac{Y(t)}{\n{Y(t)}}\right)&=\frac{Y'(t)}{\n{Y(t)}}-\frac{Y(t)}{\n{Y(t)}^2}\frac{d}{dt}\n{Y(t)}\\&=-\frac{Y(t)}{\n{Y(t)}^2}+\frac{Y(t)}{\n{Y(t)}^2}=0.\end{aligned}$$ Hence $$\frac{F(\phi_x(t))}{\n{F(\phi_x(t))}}=\frac{F(x)}{\n{F(x)}}.$$ Combining this identity with \eqref{eq:NormalizedNewtonExactDecay}, we obtain the exact relation \begin{equation}\label{eq:NormalizedNewtonVectorDecay}F(\phi_x(t))=\left(1-\frac{t}{T}\right)F(x)\end{equation} for every $0\le t<T_{\max}$. In particular, by \eqref{eq:NormalizedNewtonInitialBound}, $$\n{F(\phi_x(t))}\le\n{F(x)}<\frac{\delta_F}{2},$$ and hence \eqref{eq:NormalizedNewtonSingularValue} gives $$\sigma_{\min}(\Jac F(\phi_x(t)))\ge\frac{\sqrt3}{2}\delta_F.$$ Combining the lower bound on the smallest singular value with \eqref{eq:NormalizedNewtonRightInverse}, we obtain $$\n{L_F(\phi_x(t))}_{\operatorname{op}}=\frac{1}{\sigma_{\min}(\Jac F(\phi_x(t)))}\le\frac{2}{\sqrt3\,\delta_F}.$$ The defining equation \eqref{eq:Newton_Flow} then yields \begin{equation}\label{eq:NormalizedNewtonSpeed}\n{\dot\phi_x(t)}\le\n{L_F(\phi_x(t))}_{\operatorname{op}}\le\frac{2}{\sqrt3\,\delta_F}.\end{equation} Integrating this estimate gives $$\n{\phi_x(t)-x}\le\frac{2t}{\sqrt3\,\delta_F}<\frac{2T}{\sqrt3\,\delta_F}.$$ Moreover, by the estimate obtained in Part I, $T=\n{F(x)}<k\delta_Fr$, and therefore \begin{equation}\label{eq:NormalizedNewtonDisplacement}\n{\phi_x(t)-x}<\frac{2kr}{\sqrt3}.\end{equation} Since $r\le\sqrt3\,\tau/(4k)$, we obtain $$\n{\phi_x(t)-x}<\frac{\tau}{2}.$$ Thus $$\dist(\phi_x(t),\partial D_R^n)\ge\dist(x,\partial D_R^n)-\n{\phi_x(t)-x}>\frac{\tau}{2},$$ so the trajectory remains in $\operatorname{int}(D_R^n)$ for every $0\le t<T_{\max}$.

\medskip \noindent\textbf{III)} We finally show that the maximal solution is defined up to the time $T$ and reaches the zero set at $t=T$. Recall that \eqref{eq:NormalizedNewtonExactDecay} implies $T_{\max}\le T$. Assume by contradiction that $T_{\max}<T$. By \eqref{eq:NormalizedNewtonSpeed}, the curve $\phi_x$ is uniformly Lipschitz on $[0,T_{\max})$, and therefore the limit $$y:=\lim_{t\to T_{\max}^-}\phi_x(t)$$ exists. By the estimates obtained in Part II, $$\dist(y,\partial D_R^n)\ge\frac{\tau}{2}>0,$$ so $y\in\operatorname{int}(D_R^n)$. Moreover, by \eqref{eq:NormalizedNewtonExactDecay}, $$\n{F(y)}=\lim_{t\to T_{\max}^-}\n{F(\phi_x(t))}=T-T_{\max}>0,$$ while $$\n{F(y)}\le\n{F(x)}<\frac{\delta_F}{2}<\delta_F.$$ Hence $y\in\Omega$. Since the normalized Newton vector field is locally Lipschitz on $\Omega$, the solution can be continued from $y$ beyond $T_{\max}$, contradicting the maximality of $T_{\max}$. Therefore, $T_{\max}=T$. By \eqref{eq:NormalizedNewtonSpeed}, the curve $\phi_x$ is uniformly Lipschitz on $[0,T)$, and hence it extends continuously to $t=T$ by setting $$\phi_x(T):=\lim_{t\to T^-}\phi_x(t).$$ Since the trajectory remains at distance at least $\tau/2$ from $\partial D_R^n$, we have $\phi_x(T)\in D_R^n$. Using \eqref{eq:NormalizedNewtonVectorDecay} and the continuity of $F$, we obtain $$F(\phi_x(T))=\lim_{t\to T^-}F(\phi_x(t))=\lim_{t\to T^-}\left(1-\frac{t}{T}\right)F(x)=0.$$ Therefore, $$\phi_x(T)\in Z(F)\cap D_R^n.$$ Finally, for every $0\le t<T$, since $\phi_x(T)\in Z(F)$, $$\dist(\phi_x(t),Z(F))\le\n{\phi_x(t)-\phi_x(T)}\le\int_t^T\n{\dot\phi_x(s)}\,ds\le\frac{2(T-t)}{\sqrt3\,\delta_F}\le\frac{2T}{\sqrt3\,\delta_F}<\frac{2kr}{\sqrt3}.$$ The same conclusion is immediate at $t=T$. Hence $$\phi_x(t)\in U_{\frac{2kr}{\sqrt3}}(Z(F))\cap D_R^n$$ for every $t\in[0,T]$. Taking $\Phi_t(x):=\phi_x(t)$ completes the proof.
\end{proof}

We obtain the following uniform \L{}ojasiewicz-type inequality. Its specialization to the hypothesis class $\mathcal H_n(l,d,k)$ is given in \Cref{cor:UnifDistBound}.

\begin{thm}\label{thm:UnifDistBound}
Let $D^n_R \subseteq \mathbb{R}^n$ be the disk of radius $R$ and let $0 < \tau < R$. Suppose that $F \in C^2(\mathbb{R}^n, \mathbb{R}^c)$ satisfies
\begin{equation}\label{DistBoudnHyp}\delta_{D^n_R}(F) \ge \delta \quad \text{and} \quad Z(F) \cap \mathrm{int}(D^n_R) \neq \emptyset\end{equation} and let $k$ be the condition number of $F$. Then there exists a radius $ r(\tau,k) > 0$ with the following property: for any radius $r \le r(\tau,k) $, there exists a positive constant $C(\delta,\tau, r, k)$ such that for any $x \in D^n_R$ with $\mathrm{dist}(x, \partial D^n_R) \ge \tau$, the following inequality holds: $$ \mathrm{dist}(x, Z(F) \cap D_R^n)^2 \le C(\delta,\tau, r, k ) \|F(x)\|^2. $$
\end{thm}

\begin{proof} Set $Z:=Z(F)\cap D_R^n$ and $$E:=\left\{x\in D_R^n\ \middle|\ \dist(x,\partial D_R^n)\ge\tau\right\}.$$ Let $r_H(\tau,k)>0$ be the radius provided by \Cref{thm:HdZeroes} and let $r_N(\tau,k)>0$ be the radius provided by \Cref{lem:LocalFlowProps}. We choose $$r(\tau,k):=\min\{r_H(\tau,k),r_N(\tau,k)\}.$$ In particular, using the explicit choices in the two results, we may take $$r(\tau,k):=\min\left\{\frac{1}{4\sqrt2\,k},\frac{\sqrt3\,\tau}{4k}\right\}.$$ Fix $0<r\le r(\tau,k)$. We establish the inequality separately on $E\setminus U_r(Z)$ and on $E\cap U_r(Z)$. \medskip \noindent\textbf{I)} Let $x\in E\setminus U_r(Z)$ and let $C'(\delta,k,r)>0$ be the constant provided by \Cref{thm:HdZeroes}. We claim that $$\n{F(x)}^2>C'(\delta,k,r).$$ Indeed, if $\n{F(x)}^2\le C'(\delta,k,r)$, then \Cref{thm:HdZeroes} implies that $$Z(F)\cap B_r(x)\neq\emptyset.$$ Since $x\in E$ and $r<\tau$, the ball $B_r(x)$ is contained in $\operatorname{int}(D_R^n)$, and hence the zero provided by \Cref{thm:HdZeroes} belongs to $Z$. It follows that $\dist(x,Z)<r$, contradicting $x\notin U_r(Z)$. Since both $x$ and every point of $Z$ belong to $D_R^n$, we have $\dist(x,Z)\le2R$, and therefore \begin{equation}\label{eq:DistBoundFinal1}\dist(x,Z)^2\le4R^2<\frac{4R^2}{C'(\delta,k,r)}\n{F(x)}^2.\end{equation} \medskip \noindent\textbf{II)} Let now $x\in E\cap U_r(Z)$. If $x\in Z$, the desired inequality is immediate. Assume therefore that $x\notin Z$. Since $Z\subseteq Z(F)$, we have $x\in U_r(Z(F))\setminus Z(F)$, and \Cref{lem:LocalFlowProps} provides a normalized Newton flow $\Phi_t$ starting at $x$, which extends continuously to the time $T=\n{F(x)}$ and satisfies $$\Phi_T(x)\in Z(F)\cap D_R^n=Z.$$ Setting $\delta_F:=\delta_{D_R^n}(F)$, the proof of \Cref{lem:LocalFlowProps} gives $$\n{\frac{d}{dt}\Phi_t(x)}\le\frac{2}{\sqrt3\,\delta_F}$$ for every $0\le t<T$. Hence $$\begin{aligned}
\dist(x,Z)\le\n{x-\Phi_T(x)}_{\R^n}\le\int_0^T\n{\frac{d}{dt}\Phi_t(x)} dt  \le\frac{2T}{\sqrt3\,\delta_F} &=\frac{2}{\sqrt3\,\delta_F}\n{F(x)}\\
\end{aligned}$$ Therefore \begin{equation}\label{eq:DistBoundFinal2}\dist(x,Z)^2\le\frac{4}{3\delta^2}\n{F(x)}^2.\end{equation} \medskip \noindent\textbf{III)} Combining \eqref{eq:DistBoundFinal1} and \eqref{eq:DistBoundFinal2}, for every $x\in D_R^n$ satisfying $\dist(x,\partial D_R^n)\ge\tau$ we obtain $$\dist(x,Z(F)\cap D_R^n)^2\le\max\left\{\frac{4R^2}{C'(\delta,k,r)},\frac{4}{3\delta^2}\right\}\n{F(x)}^2.$$ Thus the claim follows by taking $$C(\delta,\tau,r,k):=\max\left\{\frac{4R^2}{C'(\delta,k,r)},\frac{4}{3\delta^2}\right\},$$ where the fixed geometric parameter $R$ is absorbed into the constant. \end{proof}

For polynomial systems in the class $\mathcal{P}_n^*(c,d,k)$, the previous result takes a simpler form. In this case, the $C^2$--norm is normalized to one and the discriminant functional is bounded below by $k^{-1}$. The following corollary gives the corresponding \L{}ojasiewicz-type inequality used in the decision process.

\begin{cor}\label{cor:UnifDistBound}
Let $D_R^n\subseteq\mathbb{R}^n$ be the disk of radius $R$ and let $0<\tau<R$. For any $k\ge1$, there exists a positive constant $C(\tau,k)$ such that, for every $P\in\mathcal{P}_n^*(c,d,k)$ and any $x\in D_R^n$ with $\dist(x,\partial D_R^n)\ge \tau$, the following inequality holds: $$\dist(x,Z(P)\cap D_R^n)^2\le C(\tau,k)\n{P(x)}^2.$$
\end{cor}

\begin{proof}
Let $P\in\mathcal{P}_n^*(c,d,k)$. By the definition of the class, we have $\n{P}_{C^2(D_R^n)}=1$ and $K_{D_R^n}(P)\le k$. Recalling the characterization in \Cref{prop:KCharact}, the condition number bound implies $$\delta_{D_R^n}(P)=\operatorname{dist}(P,\Delta(D_R^n))=\frac{\n{P}_{C^2(D_R^n)}}{K_{D_R^n}(P)}\ge\frac{1}{k}.$$ Thus, $P$ satisfies the hypotheses of \Cref{thm:UnifDistBound} with $\delta=1/k$. Let $r(\tau,k)$ be any fixed admissible radius provided by \Cref{thm:UnifDistBound}. Since $\n{P}_{C^2(D_R^n)}=1$ and $\delta_{D_R^n}(P)\ge1/k$, the constant in \Cref{thm:UnifDistBound} depends only on $\tau$ and $k$, with the fixed geometric parameter $R$ absorbed into the constant. Taking $$C(\tau,k):=C\left(\frac1k,\tau,r(\tau,k),k\right)$$ gives the desired inequality.
\end{proof}

\section{Auxiliary Optimization Problem and Discretization}
\label{sec:Auxiliary_Optimization_Discretization}
In this section, we combine the geometric results developed in the previous sections, specifically \Cref{thm:HdZeroes} and \Cref{thm:UnifDistBound}, to establish a concrete testing procedure.

\subsection{Auxiliary Optimization Problem} Let $X$ be a random variable with law $\mu \in \mathcal{M}_{\tau}(D_R^n)$, and let $\{x^{(1)},\dots,x^{(m)}\}$ be an i.i.d.\ sample drawn from $X$. We define the \textit{auxiliary functional} $\eval_{\mu}$ and its empirical counterpart $\eval_m$ on $\mathcal{P}_n(c,d,k)$ by $$\eval_{\mu}(P):=\int_{D_R^n}|P(x)|^2\,d\mu(x) \qquad\text{and}\qquad \eval_m(P):=\frac{1}{m}\sum_{i=1}^m |P(x^{(i)})|_{\mathbb{R}^c}^2.$$
As argued at the end of \Cref{sec:implicit_poly_reg}, given $\mu\in\mathcal{M}_{\tau}(D_R^n)$, testing the hypothesis requires bounding the value of the optimization problem
\begin{equation}\label{TrueOpt}
\inf_{P\in\mathcal{P}_n^*(c,d,k)}\mathcal{R}_{\mu}(Z(P)\cap D_R^n).
\end{equation}
The following theorem is a direct consequence of \Cref{thm:UnifDistBound}.

\begin{thm}
\label{thm:ProblemReduction}
Let $D^n_R$ be the disk of radius $R$ and consider $\tau \in (0, R)$. For any $k \ge 1$, there exists a radius $r(k, \tau)$ such that for any $r\le r(k, \tau)$ there exists a constant $C'(k, \tau ,r ) > 0$ with the following property. For any measure $\mu \in \mathcal{M}_{\tau}(D^n_R)$ and any $P \in \mathcal{P}^*_n(c, d, k)$, the following inequalities hold: $$\eval_{\mu}(P) \le \mathcal{R}_{\mu}(Z(P)\cap D^n_R) \le C'(k, \tau ,r ) \eval_{\mu}(P).$$ Furthermore, for the empirical functionals $\eval_m$ and $\mathcal{R}_m$, we have:$$\eval_m(P) \le \mathcal{R}_m(Z(P)\cap D^n_R) \le C'(k, \tau ,r ) \eval_m(P).$$
\end{thm}
\begin{proof} \textbf{I)} We first establish the lower bound for $\eval_{\mu}$. Consider $P \in \mathcal{P}^*_n(c, d, k)$ and fix $x \in D^n_R$. For any point $z \in Z(P) \cap D^n_R$, the convexity of the disk $D^n_R$ ensures that the segment $[z, x]$ is contained in the domain. By the Fundamental Theorem of Calculus, we have $$P(x) = P(x) - P(z) = \int_0^1 \Jac P(z + s(x - z)) \cdot (x - z) \, ds.$$ Taking the norm and applying the integral triangle inequality, we obtain $$\n{P(x)}\le \int_0^1 \n{\operatorname{Jac} P(z + s(x - z))}_{\mathrm{op}} \n{x - z} \, ds \le \n{P}_{C^1} \n{x - z}.$$ Since this holds for any $z \in Z(P) \cap D^n_R$, taking the infimum over the zero set yields $$\n{P(x)}\le \n{P}_{C^1} \operatorname{dist}(x, Z(P) \cap D^n_R).$$ Using the normalization $\n{P}_{C^2} = 1$, we have $\n{P}_{C^1} \le 1$. Squaring both sides and integrating with respect to $\mu$ yields $$\eval_{\mu}(P) = \int_{D_R^n} \n{P(x)}^2 \, d\mu(x) \le \int_{D_R^n} \operatorname{dist}(x, Z(P) \cap D^n_R)^2 \, d\mu(x) = \mathcal{R}_{\mu}(Z(P)\cap D^n_R).$$ The same pointwise inequality, applied to each sample point $x^{(i)}$ and averaged over $i=1,\dots,m$, yields $$\eval_m(P)\leq\mathcal{R}_m(Z(P)\cap D^n_R).$$ \textbf{II)} Let $\mu \in \mathcal{M}_{\tau}(D^n_R)$ and fix $P \in \mathcal{P}^*_n(c, d, k)$. By definition, $P$ satisfies $$\frac{1}{k} \le \delta_{D^n_R}(P) \le 1.$$ According to \Cref{thm:UnifDistBound}, there exists a constant $C(1/k, \tau,k, r )$ such that for every $x \in D^n_R$ with $\operatorname{dist}(x, \partial D^n_R) \geq \tau$, we have \begin{equation}\label{eq:DistBoundReduc} \operatorname{dist}(x, Z(P) \cap D^n_R)^2 \le C(1/k, \tau,k, r ) \n{P(x)}^2. \end{equation} Setting $C'(k, \tau, r) := C(1/k, \tau,k, r)$, we observe that for any $\mu \in \mathcal{M}_{\tau}(D^n_R)$, the support of the measure satisfies $$\operatorname{supp}(\mu) \subseteq \{x \in D^n_R \mid \operatorname{dist}(x, \partial D^n_R) \ge \tau\}.$$ Since all terms are non-negative, integrating with respect to $\mu$ yields $$\mathcal{R}_{\mu}(Z(P)\cap D^n_R) = \int_{D_R^n} \operatorname{dist}(x, Z(P) \cap D^n_R)^2 \, d\mu(x) \le C'(k, \tau, r) \int_{D_R^n} \n{P(x)}^2 \, d\mu(x).$$ The same argument applies to the empirical measure $\mu_m := \frac{1}{m} \sum_{i=1}^m \delta_{x^{(i)}}$ since $x^{(i)} \in \operatorname{supp}(\mu)$ satisfies the same boundary distance condition, yielding $$\mathcal{R}_m(Z(P) \cap D^n_R) \le C'(k, \tau, r)\eval_m(P).$$ \end{proof}

\Cref{thm:ProblemReduction} provides a bound for the original optimization problem, provided one can solve the auxiliary problem $$\inf_{P\in\mathcal P_n^*(c,d,k)}\mathcal E(P).$$ Since the class $\mathcal P_n^*(c,d,k)$ is defined by the existence of zeros, it is more convenient to work with the parametrized class $\mathcal P_n(c,d,k)$, for which in general one only has $$\inf_{P\in\mathcal P_n(c,d,k)}\mathcal E(P)\leq \inf_{P\in\mathcal P_n^*(c,d,k)}\mathcal E(P).$$ However, \Cref{thm:HdZeroes} provides sufficient conditions ensuring that an element of $\mathcal P_n(c,d,k)$ belongs to $\mathcal P_n^*(c,d,k)$, and these conditions can be checked on a finite discretization of the domain. This leads to the following definition.

\begin{defn}\label{def:ZeroFunctional} Let $\mathcal N$ be an $r$--net of a compact set $E\subseteq\mathbb R^n$. We define the \textit{zero functional} associated with $\mathcal N$ as the map $\zeta_{\mathcal N}:\mathbb R[x_1,\dots,x_n]_{\leq d}^c\to\R^+$ given by $$\zeta_{\mathcal N}(P):=\min_{y\in\mathcal N}\|P(y)\|^2.$$ \end{defn}
Although the zero functional defined in \Cref{def:ZeroFunctional} depends on the specific choice of the $r$--net $\mathcal N$, the estimates used below depend only on the resolution $r$ of the discretization and not on the specific arrangement of its points. This dependence is made precise by the following property.
\begin{prop}\label{prop:ZeroFunctional} Let $D_R^n$ be the disk of radius $R$ and let $0<\tau<R$, and set $R_\tau:=R-\tau$. There exists a radius $r(\tau,k)>0$ such that, for any $r\le r(\tau,k)$, there exists a constant $C''(k,r)>0$ with the following property: if $\mathcal N$ is any $r$--net of the disk $D_{R_\tau}^n$ and $\zeta_{\mathcal N}:\mathcal P_n(c,d,k)\to\R^+$ is the associated zero functional, then for any $P\in\mathcal P_n(c,d,k)$, $$\zeta_{\mathcal N}(P)\le C''(k,r)\quad\Rightarrow\quad P\in\mathcal P_n^*(c,d,k).$$ \end{prop}

\begin{proof} Recall that any $P\in\mathcal P_n(c,d,k)$ satisfies $$\frac{1}{k}\le \delta_{D_R^n}(P)\le 1.$$ Define $$r(\tau,k):=\min\left\{\frac{\tau}{2},\frac{1}{4\sqrt{2}k}\right\},$$ and let $\mathcal N$ be an $r$--net of $D_{R_\tau}^n$ with $r\le r(\tau,k)$. For any $y\in\mathcal N$, one has $\dist(y,\partial D_R^n)\ge\tau$, and therefore $$D_{R_\tau}^n\subseteq\bigcup_{y\in\mathcal N}B_r(y)\subseteq D_R^n.$$ Let $\zeta_{\mathcal N}$ be the zero functional associated with $\mathcal N$ and define $C''(k,r):=C(1/k,k,r)$, where $C$ is the constant provided by \Cref{thm:HdZeroes}. Let $P\in\mathcal P_n(c,d,k)$ satisfy $$\zeta_{\mathcal N}(P)=\min_{y\in\mathcal N}|P(y)|^2\le C''(k,r).$$ Then, by \Cref{thm:HdZeroes}, there exist $y\in\mathcal N$ and $z\in B_r(y)$ such that $P(z)=0$. Since $\n{z-y}<r\le\tau/2$ and $\dist(y,\partial D_R^n)\ge\tau$, it follows that $$\dist(z,\partial D_R^n)\ge\dist(y,\partial D_R^n)-\n{z-y}>0.$$ Hence $z\in\operatorname{int}(D_R^n)$, and therefore $P\in\mathcal P_n^*(c,d,k)$. \end{proof}

\subsection{Discretized Optimization Problem}
\label{subsec:EffectComp}
We now consider the class $\mathcal{P}_n(c,d,k)$ from a computational point of view. Membership in the class $\mathcal{P}_n(c,d,k)$ is determined by two global conditions, namely the normalization of the $C^2$--norm and a uniform lower bound on the discriminant functional. To obtain a tractable testing procedure, we discretize both quantities on a finite net. Let $\mathcal{N}\subseteq D_R^n$ be a finite $r$--net of the disk. We define the discretized constraint set
\begin{equation}
\mathcal{I}_{\mathcal{N}}(c,d,k):=\left\{P\in\mathbb{R}[x_1,\dots,x_n]_{\le d}^c\ \Big|\ \delta_{\mathcal N}(P)\ge\frac{1}{k},\ \n{P}_{\mathcal N}=1\right\},
\end{equation}
where the discretized discriminant functional and the discretized $C^2$--norm are respectively given by $$\delta_{\mathcal N}(P):=\min_{y\in\mathcal N}\sqrt{\n{P(y)}^2+\sigma_{\min}(\operatorname{Jac}P(y))^2},$$ and $$\n{P}_{\mathcal N}^2:=\max_{y\in\mathcal N}\left(\n{P(y)}^2+\n{\operatorname{Jac}P(y)}_{\operatorname{op}}^2+\n{D^2P(y)}_{\operatorname{op}}^2\right).$$

To compare the discretized optimization problem with the one over $\mathcal{P}_n(c,d,k)$, we first control the effect of the discretization on both the $C^2$--norm and the discriminant functional. For the $C^2$--norm, such a control would typically depend on the $C^3$--data of the functions under consideration.

For polynomial systems, this dependence can instead be avoided by using Markov-type inequalities, which allow us to control higher-order derivatives in terms of lower-order ones, up to constants depending on the degree $d$. The required estimate is summarized in the following proposition, whose proof is deferred to \Cref{sec:Discretization_Polynomial_Norms}.

\begin{prop}\label{prop:DiscretizedC2Norm}
Let $R>0$ and let $0<\eta<1$. There exists a radius $r_0=r_0(\eta,d,R)>0$, independent of the ambient dimensions $n$ and $c$, such that for every $r\le r_0$, every $r$--net $\mathcal N\subseteq D_R^n$, and every polynomial map $P\in\mathbb R[x_1,\dots,x_n]_{\le d}^c$, one has
$$\n{P}_{\mathcal N}\le \n{P}_{C^2(D_R^n)}\le \frac{1}{1-\eta}\n{P}_{\mathcal N}.$$ Moreover, the radius $r_0$ may be chosen so that $$r_0\ge C\frac{\eta R}{d^4},$$ where $C>0$ is a universal constant.
\end{prop}

We next compare the discretized discriminant functional with its continuous counterpart.

\begin{prop}\label{prop:QInvert}
Let $y\in D_R^n$ and let $r>0$. For any $F\in C^2(\mathbb{R}^n,\mathbb{R}^c)$ with $\delta_y(F)>0$, we have $$\delta_{\overline{B}_r(y)\cap D_R^n}(F)\ge\frac{\delta_y(F)}{2\sqrt{2}},\qquad\text{provided}\qquad r\le\frac{\delta_y(F)}{2\sqrt{2}\n{F}_{C^2(D_R^n)}}.$$
\end{prop}
\begin{proof} Consider $F\in C^2(\mathbb{R}^n,\mathbb{R}^c)$ and a point $y\in D_R^n$. Suppose that $$\n{F(y)}^2+\sigma_{\min}(\operatorname{Jac}F(y))^2\ge\delta^2,$$ where $\delta:=\delta_y(F)$. We partition $D_R^n$ into the sets $$E^-=\left\{x\in D_R^n\ \middle|\ \n{F(x)}^2<\frac{\delta^2}{2}\right\},\qquad E^+=\left\{x\in D_R^n\ \middle|\ \n{F(x)}^2\ge\frac{\delta^2}{2}\right\}.$$ \medskip \noindent\textbf{I)} Suppose $y\in E^-$. Then it must hold that $\sigma_{\min}(\operatorname{Jac}F(y))^2\ge\delta^2/2$. By the quantitative continuity of singular values (cf. \Cref{prop:SVContinuity}), for any radius $$r\le\frac{\delta}{2\sqrt{2}\n{F}_{C^2(D_R^n)}},$$ we have $$\inf_{x\in B_r(y)\cap D_R^n}\sigma_{\min}(\operatorname{Jac}F(x))^2\ge\frac{\delta^2}{8}.$$ Consequently, it follows that $$\min_{x\in\overline B_r(y)\cap D_R^n}\left\{\n{F(x)}^2+\sigma_{\min}(\operatorname{Jac}F(x))^2\right\}\ge\frac{\delta^2}{8}.$$ \medskip \noindent\textbf{II)} Suppose $y\in E^+$, so that $\n{F(y)}\ge\delta/\sqrt{2}$. For any $x\in B_r(y)\cap D_R^n$, since $D_R^n$ is convex, the segment joining $x$ and $y$ is contained in $D_R^n$. We therefore use the integral form of the Taylor expansion $$F(x)-F(y)=\int_0^1\operatorname{Jac}F(y+s(x-y))\cdot(x-y)\,ds.$$ By the reverse triangle inequality and the definition of the $C^1$--norm, we have $$\begin{aligned}\left|\n{F(x)}-\n{F(y)}\right|\le\n{F(x)-F(y)}&\le\int_0^1\n{\operatorname{Jac}F(y+s(x-y))\cdot(x-y)}\,ds\\&\le\n{F}_{C^1(D_R^n)}\n{x-y}.\end{aligned}$$ This implies $$\n{F(x)}\ge\n{F(y)}-\n{F}_{C^1(D_R^n)}\n{x-y}\ge\frac{\delta}{\sqrt{2}}-\n{F}_{C^1(D_R^n)}r.$$ Setting $$r\le\frac{\delta}{2\sqrt{2}\n{F}_{C^1(D_R^n)}},$$ we obtain $\n{F(x)}\ge\delta/(2\sqrt{2})$. Since $\n{F}_{C^1(D_R^n)}\le\n{F}_{C^2(D_R^n)}$, it follows that for any $$r\le\frac{\delta}{2\sqrt{2}\n{F}_{C^2(D_R^n)}},$$ we have $$\min_{x\in\overline B_r(y)\cap D_R^n}\left\{\n{F(x)}^2+\sigma_{\min}(\operatorname{Jac}F(x))^2\right\}\ge\frac{\delta^2}{8}.$$ Both case \textbf{I)} and case \textbf{II)} yield the desired bound, provided $$r\le\frac{\delta}{2\sqrt{2}\n{F}_{C^2(D_R^n)}}.$$ \end{proof}

While the condition-number constraint is invariant under rescaling, the cost functional is homogeneous of degree two and hence depends on the normalization. We therefore introduce the intermediate feasible set $$ \mathcal J_{\mathcal N}(c,d,k):=\left\{P\in\mathbb R[x_1,\dots,x_n]_{\le d}^c\ \Big|\ \delta_{\mathcal N}(P)\ge\frac{1}{k},\ \n{P}_{C^2(D_R^n)}=1\right\}. $$ By extending the local stability estimate of \Cref{prop:QInvert} to the entire discretization, we obtain that, for a sufficiently fine net, $\mathcal{J}_{\mathcal{N}}(c,d,k)$ lies between $\mathcal{P}_n(c,d,k)$ and a relaxed version of it.

\begin{prop}\label{prop:FeasibleNet} For any $k\ge1$, there exists a radius $r(k)>0$ such that, for any $r$--net $\mathcal N$ of $D_R^n$ with $r\le r(k)$, the following inclusions hold: $$\mathcal P_n(c,d,k)\subseteq \mathcal{J}_{\mathcal N}(c,d,k)\subseteq \mathcal P_n\bigl(c,d,kC'''\bigr),$$ where $C'''>1$ is a universal constant.
\end{prop}
\begin{proof} Let $\mathcal N$ be an $r$--net of $D_R^n$. First, observe that, independently of $r$, since $\mathcal N\subseteq D_R^n$, for any $P\in\R[x_1,\dots,x_n]_{\le d}^c$ one has $\delta_{\mathcal N}(P)\ge \delta_{D_R^n}(P)$. This immediately implies that $$\mathcal P_n(c,d,k)\subseteq \mathcal{J}_{\mathcal N}(c,d,k).$$ Conversely, let $P\in\mathcal J_{\mathcal N}(c,d,k)$. Then $\delta_{\mathcal N}(P)\ge1/k$ and $\n{P}_{C^2(D_R^n)}=1$. By \Cref{prop:QInvert}, if $$r\le \frac{1}{2\sqrt{2}\,k},$$ then $$\delta_{D_R^n}(P)\ge \frac{1}{2\sqrt{2}\,k}.$$ We therefore obtain $$\mathcal{J}_{\mathcal N}(c,d,k)\subseteq \mathcal P_n\bigl(c,d,2\sqrt2\,k\bigr).$$ Hence the conclusion follows with $r(k):=1/(2\sqrt2\,k)$ and the universal constant $C'''=2\sqrt2$. \end{proof}

We next establish the compactness of $\mathcal I_{\mathcal N}(c,d,k)$ and $\mathcal J_{\mathcal N}(c,d,k)$, and hence the existence of minimizers for the corresponding optimization problems.

\begin{prop}\label{prop:FeasPolyComp} Let $0<\eta<1$, and let $\mathcal N$ be an $r$--net of $D_R^n$ with $r\le r_0(\eta,d,R)$, where $r_0$ is the radius provided by \Cref{prop:DiscretizedC2Norm}. Then the sets $\mathcal J_{\mathcal N}(c,d,k)$ and $\mathcal I_{\mathcal N}(c,d,k)$ are compact. \end{prop} \begin{proof} Since $\mathbb R[x_1,\dots,x_n]_{\le d}^c$ is finite-dimensional, it is enough to prove that the two sets are closed and bounded.
\
\\
 \noindent\textbf{I)} We first consider $\mathcal J_{\mathcal N}(c,d,k)$. For every fixed $y\in\mathcal N$, the maps $P\mapsto P(y)$ and $P\mapsto\operatorname{Jac}P(y)$ are continuous in the coefficients of $P$. Since the smallest singular value is a continuous function of the matrix entries, the map $P\mapsto\delta_{\mathcal N}(P)$ is continuous. Hence the condition $\delta_{\mathcal N}(P)\ge 1/k$ defines a closed set. Moreover, the constraint $\n{P}_{C^2(D_R^n)}=1$ defines a compact subset of the finite-dimensional space $\mathbb R[x_1,\dots,x_n]_{\le d}^c$. Therefore $\mathcal J_{\mathcal N}(c,d,k)$ is compact. 
 \
 \\
\noindent\textbf{II)} We now consider $\mathcal I_{\mathcal N}(c,d,k)$. Since both $P\mapsto\n{P}_{\mathcal N}$ and $P\mapsto\delta_{\mathcal N}(P)$ are continuous, the constraints $$\delta_{\mathcal N}(P)\ge\frac1k,\qquad \n{P}_{\mathcal N}=1$$ define a closed set. On the other hand, since $r\le r_0(\eta,d,R)$, \Cref{prop:DiscretizedC2Norm} gives $$\n{P}_{C^2(D_R^n)}\le\frac{1}{1-\eta}\n{P}_{\mathcal N}=\frac{1}{1-\eta}$$ for every $P\in\mathcal I_{\mathcal N}(c,d,k)$. Thus $\mathcal I_{\mathcal N}(c,d,k)$ is bounded in the $C^2$--norm. Being closed and bounded in a finite-dimensional vector space, it is compact. \end{proof}

Since $\mathcal J_{\mathcal N}(c,d,k)$ and $\mathcal I_{\mathcal N}(c,d,k)$ differ only through the normalization constraint, the homogeneity of the cost functional yields the following comparison between their optimal values.

\begin{prop}\label{prop:discrete_shell_bounds} Let $0<\eta<1$ and let $\mathcal N$ be an $r$--net of $D_R^n$ with $r\le r_0(\eta,d,R)$, where $r_0$ is the radius provided by \Cref{prop:DiscretizedC2Norm}. Let $\mathcal E(P)$ be a continuous cost functional homogeneous of degree two and set $\lambda(\eta)=1/(1-\eta)^2$. Then $$\min_{P\in\mathcal I_{\mathcal N}(c,d,k)}\mathcal E(P)\le\lambda(\eta)\min_{P\in\mathcal J_{\mathcal N}(c,d,k)}\mathcal E(P),\quad\text{and}\quad\min_{P\in\mathcal J_{\mathcal N}(c,d,\sqrt{\lambda(\eta)}\,k)}\mathcal E(P)\le\min_{P\in\mathcal I_{\mathcal N}(c,d,k)}\mathcal E(P).$$ \end{prop}

\begin{proof} For the first inequality, let $P\in\mathcal J_{\mathcal N}(c,d,k)$ and define $$U:=\frac{P}{\n{P}_{\mathcal N}}.$$ Since $\n{P}_{C^2(D_R^n)}=1$, \Cref{prop:DiscretizedC2Norm} gives $$1-\eta\le\n{P}_{\mathcal N}\le1.$$ By construction, $\n{U}_{\mathcal N}=1$, while $$\delta_{\mathcal N}(U)=\frac{\delta_{\mathcal N}(P)}{\n{P}_{\mathcal N}}\ge\delta_{\mathcal N}(P)\ge\frac1k.$$ Hence $U\in\mathcal I_{\mathcal N}(c,d,k)$. Since $\mathcal E$ is homogeneous of degree two, $$\mathcal E(U)=\frac{1}{\n{P}_{\mathcal N}^2}\mathcal E(P)\le\frac{1}{(1-\eta)^2}\mathcal E(P)=\lambda(\eta)\mathcal E(P).$$ Taking the minimum over $P\in\mathcal J_{\mathcal N}(c,d,k)$ yields $$\min_{P\in\mathcal I_{\mathcal N}(c,d,k)}\mathcal E(P)\le\lambda(\eta)\min_{P\in\mathcal J_{\mathcal N}(c,d,k)}\mathcal E(P).$$ For the second inequality, let $P\in\mathcal I_{\mathcal N}(c,d,k)$ and define $$P_c:=\frac{P}{\n{P}_{C^2(D_R^n)}}.$$ By construction, $\n{P_c}_{C^2(D_R^n)}=1$. Since $\n{P}_{\mathcal N}=1$, \Cref{prop:DiscretizedC2Norm} gives $$1\le\n{P}_{C^2(D_R^n)}\le\frac{1}{1-\eta}.$$ Therefore $$\delta_{\mathcal N}(P_c)=\frac{\delta_{\mathcal N}(P)}{\n{P}_{C^2(D_R^n)}}\ge\frac{1-\eta}{k}=\frac{1}{\sqrt{\lambda(\eta)}\,k},$$ and hence $P_c\in\mathcal J_{\mathcal N}(c,d,\sqrt{\lambda(\eta)}\,k)$. Moreover, $$\mathcal E(P_c)=\frac{1}{\n{P}_{C^2(D_R^n)}^2}\mathcal E(P)\le\mathcal E(P).$$ Taking the minimum over $P\in\mathcal I_{\mathcal N}(c,d,k)$ yields $$\min_{P\in\mathcal J_{\mathcal N}(c,d,\sqrt{\lambda(\eta)}\,k)}\mathcal E(P)\le\min_{P\in\mathcal I_{\mathcal N}(c,d,k)}\mathcal E(P),$$ which completes the proof. \end{proof}

\section{Decision Procedure and Sample Complexity}
\label{subsec:DecisionSample}

\subsection{Algorithmic Decision Procedure}
\label{subsec:Algorithmic_Decision} Let $X$ be a random variable with law $\mu\in\mathcal M_{\tau}(D_R^n)$, and let $\{x^{(1)},\dots,x^{(m)}\}$ be an i.i.d.\ sample drawn from $\mu$. The testing problem for the class $\mathcal{H}_n(l,d,k)$ is reduced to the discretized empirical minimization problem
\begin{equation}\label{scellOpt}
\min_{P\in\mathcal I_{\mathcal N}(c,d,k)}\mathcal E_m(P).
\end{equation}

The comparison results of \Cref{subsec:EffectComp} relate this problem to the corresponding minimization over $\mathcal J_{\mathcal N}(c,d,k)$ and, through the latter, to the continuous polynomial class. Since $\mathcal E_m$ is continuous and both $\mathcal I_{\mathcal N}(c,d,k)$ and $\mathcal J_{\mathcal N}(c,d,k)$ are compact by \Cref{prop:FeasPolyComp}, the associated minimization problems admit minimizers. The quantitative loss in passing between the two discretized formulations is governed by the parameter $\lambda(\eta)$, with $\lambda(\eta)\to1$ as $\eta\to0$. In particular, $\eta$ may be fixed so that $\lambda(\eta)$ is a universal constant; for example, $\eta=1-1/\sqrt2$ gives $\lambda(\eta)=2$. The dependence on the degree $d$ enters through the discretization scale required by the norm estimates.

The following result, proved in Part II, provides the algorithmic procedure for the discretized optimization problem.

\begin{thm}\label{thm:RelaxedAdmissibleGlobalOptimizationPartI}
Fix a degree $d$, a condition-number bound $k\geq1$, and $0<\varepsilon\leq\bar\varepsilon$. Then there exist a constant $\Lambda_0>1$ and $h>0$ such that \Cref{alg:relaxed-global} outputs a polynomial system $\hat P\in\mathcal I_{\mathcal N}(c,d,\Lambda_0k)$ satisfying one of the following alternatives:
\begin{enumerate}[label=(\roman*)]
\item The output satisfies $\mathcal E(\hat P)\leq\varepsilon$.
\item The output satisfies $\mathcal E(\hat P)>\varepsilon$, and $$\min_{P\in\mathcal I_{\mathcal N}(c,d,k)}\mathcal E(P) \ge \min\left\{\frac{\mathcal E(\hat P)}{2},\bar\varepsilon\right\}. $$
\end{enumerate}
\end{thm}

The result above provides the algorithmic component of the testing procedure. To transfer its conclusions from the empirical optimization problem to the population risk, we also need to compare the empirical functional $\mathcal E_m$ with its population counterpart $\mathcal E_\mu$ uniformly over the relevant feasible set. The required deviation estimates, together with the resulting sample complexity bounds, are established in \Cref{sec:IPRSamCompBounds}. Combining these statistical bounds with the algorithmic result above yields the main theorem of Part I.

\begin{thm}
\label{thm:MainTheoremPartI}
Fix a dimension $l$, a degree $d$, a condition number bound $k$, a tolerance $0<\varepsilon\leq4R^2$, and a confidence level $\alpha\in(0,1)$. Let $\mu\in\mathcal M_{\tau}(D_R^n)$ for some $0<\tau<R$, and let $X$ be a random variable with law $\mu$. Assume that \Cref{thm:RelaxedAdmissibleGlobalOptimizationPartI} holds. Then there exist a universal constant $C'\geq1$, a constant $C\geq1$ depending only on $k$, an integer $m(n,l,d,k,\varepsilon,\alpha)<\infty$, and an effective learning algorithm with the following property. Upon processing at least $m(n,l,d,k,\varepsilon,\alpha)$ i.i.d.\ samples from $X$, the algorithm distinguishes, with probability at least $1-\alpha$, among the three possible alternatives (A), (B), and ``(A) and (B)'', where:
\begin{enumerate}[label=(\Alph*)]
\item there exists a manifold $Z\in\mathcal H_n(l,d,kC')$ such that $$\mathbb E[\operatorname{dist}(X,Z)^2]\leq C\,\varepsilon;$$
\item there exists no manifold $Z\in\mathcal H_n(l,d,k)$ such that $$\mathbb E[\operatorname{dist}(X,Z)^2]\leq\frac{\varepsilon}{C}.$$
\end{enumerate}
\end{thm}

\begin{proof} Throughout the proof, the dependence of the auxiliary radii and constants on the discretization parameter $\eta$ and on the interior parameter $\tau$ is left implicit whenever no ambiguity can arise. The parameter $\eta$ is chosen sufficiently small to guarantee the discretization and relaxation bounds used below, while $\tau$ is fixed by the assumption $\mu\in\mathcal M_{\tau}(D_R^n)$. Let $\mu\in\mathcal M_{\tau}(D_R^n)$ and fix the hypothesis parameters $l,d,k$, an input threshold $0<\varepsilon\leq4R^2$, and a confidence level $1-\alpha$ with $\alpha\in(0,1)$. Denote $R_{\tau}:=R-\tau$ and set $c:=n-l$. Let $\Lambda_0>1$ be the constant provided by \Cref{thm:RelaxedAdmissibleGlobalOptimizationPartI}. To establish the decision rule, we must appropriately define the nets for $D_R^n$ and $D_{R_{\tau}}^n$, identify the corresponding feasible-set parameters, and track the resulting geometric constants. First recall that, by \Cref{prop:discrete_shell_bounds}, we have \begin{equation}\label{eq:RelaxBound} \min_{P\in\mathcal I_{\mathcal N}(c,d,k)}\mathcal E(P)\leq\lambda\min_{P\in\mathcal J_{\mathcal N}(c,d,k)}\mathcal E(P), \qquad \min_{P\in\mathcal J_{\mathcal N}(c,d,\sqrt{\lambda}\Lambda_0k)}\mathcal E(P)\leq\min_{P\in\mathcal I_{\mathcal N}(c,d,\Lambda_0k)}\mathcal E(P). \end{equation} We proceed as follows. \begin{enumerate} \item We first characterize the feasible set of polynomials for the optimization. Let $r(k)$ be the radius provided by \Cref{prop:FeasibleNet}, and fix a net $\mathcal N$ of $D_R^n$ with balls of radius $r$ sufficiently small so that \Cref{prop:FeasibleNet} and \Cref{prop:discrete_shell_bounds} apply at all the condition-number levels considered below. By \Cref{prop:FeasibleNet}, we have the inclusions \begin{equation}\label{FeasableInclusion} \begin{aligned} \mathcal P_n(c,d,k)&\subseteq\mathcal J_{\mathcal N}(c,d,k)\subseteq\mathcal P_n(c,d,k'),\\ \mathcal P_n(c,d,\sqrt{\lambda}\Lambda_0k)&\subseteq\mathcal J_{\mathcal N}(c,d,\sqrt{\lambda}\Lambda_0k)\subseteq\mathcal P_n(c,d,\sqrt{\lambda}\Lambda_0k'), \end{aligned} \end{equation} where $k':=C'''k$, with $C'''>1$ the universal constant from \Cref{prop:FeasibleNet}. \item We choose $$r':=\min\left\{\frac{1}{8\sqrt{2}\,\sqrt{\lambda}\Lambda_0k'},\frac{\tau}{8(\sqrt{\lambda}\Lambda_0k')^2}\right\}.$$ Let $\mathcal N^-(r')$ be a net of $D_{R_{\tau}}^n$ chosen so as to include the sample points, that is, $$\{x^{(1)},\dots,x^{(m)}\}\subseteq\mathcal N^-(r').$$ For any polynomial system $Q$, write $$\zeta_{r'}(Q):=\min_{y\in\mathcal N^-(r')}\n{Q(y)}^2.$$ In particular, for every $P\in\mathcal P_n(c,d,\sqrt{\lambda}\Lambda_0k')$, the zero functional associated with $\mathcal N^-(r')$ satisfies \begin{equation}\label{eq:ZeroEval} \zeta_{r'}(P)\leq\min_{i=1,\dots,m}\n{P(x^{(i)})}^2\leq\eval_m(P). \end{equation} \item The inclusion $\mathcal J_{\mathcal N}(c,d,\sqrt{\lambda}\Lambda_0k)\subseteq\mathcal P_n(c,d,\sqrt{\lambda}\Lambda_0k')$ implies, by \Cref{thm:ProblemReduction}, that there exists a constant $C'(\sqrt{\lambda}\Lambda_0k',\tau,r')>0$ such that, for every $P\in\mathcal P_n^*(c,d,\sqrt{\lambda}\Lambda_0k')$, the following bounds hold: \begin{equation}\label{eq1:LossControl} \eval_{\mu}(P)\leq\mathcal R_{\mu}(Z(P)\cap D_R^n)\leq C'(\sqrt{\lambda}\Lambda_0k',\tau,r')\eval_{\mu}(P). \end{equation} More precisely, if $C''(\sqrt{\lambda}\Lambda_0k',r')$ denotes the constant associated with the zero functional in \Cref{prop:ZeroFunctional}, the constant above may be chosen so that $$C'(\sqrt{\lambda}\Lambda_0k',\tau,r')\geq\max\left\{\frac{4R^2}{C''(\sqrt{\lambda}\Lambda_0k',r')},4(\sqrt{\lambda}\Lambda_0k')^2\right\}.$$ \item We now establish the uniform approximation threshold for the functionals $\eval_{\mu}$ and $\eval_m$. Define the precision $$\varepsilon':=\frac{1}{4}\cdot\frac{\varepsilon}{\lambda C'(\sqrt{\lambda}\Lambda_0k',\tau,r')}.$$ Let $m_{\operatorname{ERM}}^*(n,c,d,\varepsilon',\alpha)$ denote the corresponding sample size provided by \Cref{IPRthm:SampleCom}. By taking $$m\geq m_{\operatorname{ERM}}^*(n,n-l,d,\varepsilon',\alpha),$$ the following uniform deviation bound holds with probability at least $1-\alpha$ for every $U\in\mathcal I_{\mathcal N}(c,d,\sqrt{\lambda}\Lambda_0k)$: \begin{equation}\label{eq1:ApproxControl} \big|\eval_{\mu}(U)-\eval_m(U)\big|\leq\frac{\varepsilon}{4\lambda C'(\sqrt{\lambda}\Lambda_0k',\tau,r')}. \end{equation} Since $\mathcal I_{\mathcal N}(c,d,k)\subseteq\mathcal I_{\mathcal N}(c,d,\sqrt{\lambda}\Lambda_0k)$, the same bound holds in particular on $\mathcal I_{\mathcal N}(c,d,k)$. It also holds for every $P\in\mathcal J_{\mathcal N}(c,d,\sqrt{\lambda}\Lambda_0k)$. Indeed, setting $U:=P/\n{P}_{\mathcal N}$, we have $U\in\mathcal I_{\mathcal N}(c,d,\sqrt{\lambda}\Lambda_0k)$, while $\n{P}_{\mathcal N}\leq\n{P}_{C^2(D_R^n)}=1$. By homogeneity, $$\big|\eval_{\mu}(P)-\eval_m(P)\big|=\n{P}_{\mathcal N}^2\big|\eval_{\mu}(U)-\eval_m(U)\big|\leq\frac{\varepsilon}{4\lambda C'(\sqrt{\lambda}\Lambda_0k',\tau,r')}.$$ \end{enumerate} The decision process is based on the following strategy. Set $$\varepsilon_1:=\frac{3\varepsilon}{4C'(\sqrt{\lambda}\Lambda_0k',\tau,r')},\qquad \varepsilon_2:=\frac{3\varepsilon}{4}.$$ Since $C'(\sqrt{\lambda}\Lambda_0k',\tau,r')>1$, we have $\varepsilon_1<\varepsilon_2$. We apply \Cref{thm:RelaxedAdmissibleGlobalOptimizationPartI} to the empirical cost $\mathcal E=\eval_m$, with threshold $\varepsilon_1$ and with the localized grid constructed at level $\varepsilon_2$. If the algorithm returns \textnormal{EMPTY}, then $$\min_{P\in\mathcal I_{\mathcal N}(c,d,k)}\eval_m(P)>\varepsilon_2.$$ On the uniform deviation event \eqref{eq1:ApproxControl}, it follows that $$\min_{P\in\mathcal I_{\mathcal N}(c,d,k)}\eval_\mu(P)\geq\varepsilon_2-\varepsilon'.$$ By \eqref{eq:RelaxBound}, \eqref{FeasableInclusion}, and \Cref{thm:ProblemReduction}, we obtain $$\inf_{Z\in\mathcal H_n(l,d,k)}\mathcal R_\mu(Z)\geq\frac{\varepsilon_2-\varepsilon'}{\lambda}\geq\frac{\varepsilon}{2\lambda},$$ where we used $\varepsilon_2=3\varepsilon/4$ and $\varepsilon'\leq\varepsilon/4$. Thus, choosing the final constant $C>2\lambda$, the algorithm certifies alternative (B) and terminates without returning a candidate manifold. Otherwise, let $\hat P$ be the output of the algorithm. By \Cref{thm:RelaxedAdmissibleGlobalOptimizationPartI}, $$\hat P\in\mathcal I_{\mathcal N}(c,d,\Lambda_0k),$$ and one of the following alternatives holds: \begin{enumerate}[label=(\roman*)] \item The output satisfies $$\eval_m(\hat P)\leq\varepsilon_1.$$ \item The output satisfies $\eval_m(\hat P)>\varepsilon_1$, and $$\min_{P\in\mathcal I_{\mathcal N}(c,d,k)}\eval_m(P)\geq\min\left\{\frac{\eval_m(\hat P)}{2},\varepsilon_2\right\}.$$ \end{enumerate} Since $\hat P\in\mathcal I_{\mathcal N}(c,d,\Lambda_0k)$, define its $C^2$--normalized representative by $$\hat{P}':=\frac{\hat P}{\n{\hat P}_{C^2(D_R^n)}}.$$ By \Cref{prop:DiscretizedC2Norm}, $$1=\n{\hat P}_{\mathcal N}\leq\n{\hat P}_{C^2(D_R^n)}\leq\sqrt{\lambda},$$ and therefore $\hat{P}'\in\mathcal J_{\mathcal N}(c,d,\sqrt{\lambda}\Lambda_0k)$. In particular, by \eqref{FeasableInclusion}, $$\hat{P}'\in\mathcal P_n(c,d,\sqrt{\lambda}\Lambda_0k').$$ Moreover, by homogeneity, $$\eval_m(\hat{P}')=\frac{\eval_m(\hat P)}{\n{\hat P}_{C^2(D_R^n)}^2}\leq\eval_m(\hat P),\qquad \zeta_{r'}(\hat{P}')=\frac{\zeta_{r'}(\hat P)}{\n{\hat P}_{C^2(D_R^n)}^2}\leq\zeta_{r'}(\hat P),$$ and $Z(\hat{P}')=Z(\hat P)$. We now classify the output according to the value of $\eval_m(\hat P)$. We consider the partition $$[0,+\infty)=\left[0,\varepsilon_1\right]\cup\left(\varepsilon_1,\varepsilon_2\right)\cup\left[\varepsilon_2,+\infty\right).$$ We show that the algorithm classifies the hypothesis state according to the interval in which the output $\eval_m(\hat P)$ falls. \medskip \noindent\textbf{I)} First, suppose that $\eval_m(\hat P)\in[0,\varepsilon_1]$. Then $$\eval_m(\hat{P}')\leq\eval_m(\hat P)\leq\varepsilon_1=\frac{3\varepsilon}{4C'(\sqrt{\lambda}\Lambda_0k',\tau,r')}.$$ Hence \begin{equation}\label{eq:EvalVal1} \eval_m(\hat{P}')\leq\frac{3\varepsilon}{4C'(\sqrt{\lambda}\Lambda_0k',\tau,r')}. \end{equation} Let $C''(\sqrt{\lambda}\Lambda_0k',r')$ be the constant associated with the zero functional from \Cref{prop:ZeroFunctional}. By the choice of $C'(\sqrt{\lambda}\Lambda_0k',\tau,r')$, \begin{equation}\label{eq:ZeroOperBound} \eval_m(\hat{P}')\leq\frac{3\varepsilon}{4C'(\sqrt{\lambda}\Lambda_0k',\tau,r')}\leq\frac{3\varepsilon}{4}\frac{C''(\sqrt{\lambda}\Lambda_0k',r')}{4R^2}\leq\frac34C''(\sqrt{\lambda}\Lambda_0k',r')<C''(\sqrt{\lambda}\Lambda_0k',r'), \end{equation} where we used $\varepsilon\leq4R^2$. From \eqref{eq:ZeroEval}, we obtain $$\zeta_{r'}(\hat{P}')\leq\eval_m(\hat{P}')<C''(\sqrt{\lambda}\Lambda_0k',r').$$ Thus, by \Cref{prop:ZeroFunctional}, $\hat{P}'\in\mathcal P_n^*(c,d,\sqrt{\lambda}\Lambda_0k')$. Combining \eqref{eq:EvalVal1} with the deviation bound \eqref{eq1:ApproxControl} yields $$\eval_{\mu}(\hat{P}')\leq\eval_m(\hat{P}')+\frac{\varepsilon}{4\lambda C'(\sqrt{\lambda}\Lambda_0k',\tau,r')}\leq\left(\frac34+\frac{1}{4\lambda}\right)\frac{\varepsilon}{C'(\sqrt{\lambda}\Lambda_0k',\tau,r')}\leq\frac{\varepsilon}{C'(\sqrt{\lambda}\Lambda_0k',\tau,r')}.$$ Finally, applying the loss control bounds from \eqref{eq1:LossControl}, we conclude $$\mathcal R_{\mu}(Z(\hat{P}')\cap D_R^n)\leq C'(\sqrt{\lambda}\Lambda_0k',\tau,r')\eval_{\mu}(\hat{P}')\leq\varepsilon.$$ Therefore, with probability at least $1-\alpha$, there exists $Z\in\mathcal H_n(l,d,\sqrt{\lambda}\Lambda_0k')$ such that $$\mathcal R_{\mu}(Z)\leq\varepsilon.$$ \medskip \noindent\textbf{II)} Suppose now that $\eval_m(\hat P)\in(\varepsilon_1,\varepsilon_2)$. By alternative (ii) of \Cref{thm:RelaxedAdmissibleGlobalOptimizationPartI}, $$\min_{P\in\mathcal I_{\mathcal N}(c,d,k)}\eval_m(P)\geq\min\left\{\frac{\eval_m(\hat P)}{2},\varepsilon_2\right\}=\frac{\eval_m(\hat P)}{2}>\frac{\varepsilon_1}{2}=\frac{3\varepsilon}{8C'(\sqrt{\lambda}\Lambda_0k',\tau,r')}.$$ Combining this bound with the uniform deviation estimate \eqref{eq1:ApproxControl}, we obtain \begin{equation}\label{eq:SecondBound} \min_{P\in\mathcal I_{\mathcal N}(c,d,k)}\eval_{\mu}(P)\geq\frac{3\varepsilon}{8C'(\sqrt{\lambda}\Lambda_0k',\tau,r')}-\frac{\varepsilon}{4\lambda C'(\sqrt{\lambda}\Lambda_0k',\tau,r')}\geq\frac{\varepsilon}{8C'(\sqrt{\lambda}\Lambda_0k',\tau,r')}. \end{equation} By the first inequality in \eqref{eq:RelaxBound}, applied to $\eval_{\mu}$, $$\min_{P\in\mathcal J_{\mathcal N}(c,d,k)}\eval_{\mu}(P)\geq\frac1{\lambda}\min_{P\in\mathcal I_{\mathcal N}(c,d,k)}\eval_{\mu}(P)\geq\frac{\varepsilon}{8\lambda C'(\sqrt{\lambda}\Lambda_0k',\tau,r')}.$$ Since $\mathcal P_n(c,d,k)\subseteq\mathcal J_{\mathcal N}(c,d,k)$ by \eqref{FeasableInclusion}, it follows that $$\frac{\varepsilon}{8\lambda C'(\sqrt{\lambda}\Lambda_0k',\tau,r')}\leq\inf_{P\in\mathcal P_n^*(c,d,k)}\eval_{\mu}(P).$$ Consequently, there is no $P\in\mathcal P_n^*(c,d,k)$ such that $$\eval_{\mu}(P)<\frac{\varepsilon}{8\lambda C'(\sqrt{\lambda}\Lambda_0k',\tau,r')}.$$ By \Cref{thm:ProblemReduction}, we conclude that there is no $Z\in\mathcal H_n(l,d,k)$ such that \begin{equation}\label{eq:NegConcl1} \mathcal R_{\mu}(Z)<\frac{\varepsilon}{8\lambda C'(\sqrt{\lambda}\Lambda_0k',\tau,r')}. \end{equation} We can refine this conclusion by evaluating the zero functional directly on the output $\hat P$. If $$\zeta_{r'}(\hat P)\geq C''(\sqrt{\lambda}\Lambda_0k',r'),$$ the existence of a suitable relaxed manifold is not certified by this test. If instead $$\zeta_{r'}(\hat P)<C''(\sqrt{\lambda}\Lambda_0k',r'),$$ then $$\zeta_{r'}(\hat{P}')\leq\zeta_{r'}(\hat P)<C''(\sqrt{\lambda}\Lambda_0k',r'),$$ and hence, by \Cref{prop:ZeroFunctional}, $\hat{P}'\in\mathcal P_n^*(c,d,\sqrt{\lambda}\Lambda_0k')$. Moreover, since $\eval_m(\hat{P}')\leq\eval_m(\hat P)<\varepsilon_2=3\varepsilon/4$, the deviation bound \eqref{eq1:ApproxControl} gives $$\eval_{\mu}(\hat{P}')\leq\eval_m(\hat{P}')+\frac{\varepsilon}{4\lambda C'(\sqrt{\lambda}\Lambda_0k',\tau,r')}\leq\left(\frac34+\frac{1}{4\lambda C'(\sqrt{\lambda}\Lambda_0k',\tau,r')}\right)\varepsilon.$$ Since $C'(\sqrt{\lambda}\Lambda_0k',\tau,r')\geq1$, \Cref{thm:ProblemReduction} implies $$\begin{aligned} \mathcal R_{\mu}(Z(\hat{P}')\cap D_R^n)&\leq C'(\sqrt{\lambda}\Lambda_0k',\tau,r')\eval_{\mu}(\hat{P}')\\ &\leq\left(\frac34C'(\sqrt{\lambda}\Lambda_0k',\tau,r')+\frac{1}{4\lambda}\right)\varepsilon\\ &\leq C'(\sqrt{\lambda}\Lambda_0k',\tau,r')\varepsilon. \end{aligned}$$ In this case, we additionally conclude that there exists $Z\in\mathcal H_n(l,d,\sqrt{\lambda}\Lambda_0k')$ such that $$\mathcal R_{\mu}(Z)\leq C'(\sqrt{\lambda}\Lambda_0k',\tau,r')\varepsilon.$$ \medskip \noindent\textbf{III)} Finally, suppose that $$\eval_m(\hat P)\in[\varepsilon_2,+\infty).$$ Since $\varepsilon_2>\varepsilon_1$, alternative (ii) of \Cref{thm:RelaxedAdmissibleGlobalOptimizationPartI} applies, and therefore $$\min_{P\in\mathcal I_{\mathcal N}(c,d,k)}\eval_m(P)\geq\min\left\{\frac{\eval_m(\hat P)}{2},\varepsilon_2\right\}\geq\frac{\varepsilon_2}{2}=\frac{3\varepsilon}{8}.$$ Combining this bound with the uniform deviation estimate \eqref{eq1:ApproxControl}, we obtain \begin{equation}\label{eq:ThirdsBound} \min_{P\in\mathcal I_{\mathcal N}(c,d,k)}\eval_{\mu}(P)\geq\frac{3\varepsilon}{8}-\frac{\varepsilon}{4\lambda C'(\sqrt{\lambda}\Lambda_0k',\tau,r')}. \end{equation} Applying the first inequality in \eqref{eq:RelaxBound} to $\eval_{\mu}$ gives $$\min_{P\in\mathcal J_{\mathcal N}(c,d,k)}\eval_{\mu}(P)\geq\left(\frac38-\frac{1}{4\lambda C'(\sqrt{\lambda}\Lambda_0k',\tau,r')}\right)\frac{\varepsilon}{\lambda}.$$ From \eqref{FeasableInclusion}, it follows that $$\left(\frac38-\frac{1}{4\lambda C'(\sqrt{\lambda}\Lambda_0k',\tau,r')}\right)\frac{\varepsilon}{\lambda}\leq\inf_{P\in\mathcal P_n^*(c,d,k)}\eval_{\mu}(P).$$ Applying the lower bound from \Cref{thm:ProblemReduction}, we conclude that there is no $Z\in\mathcal H_n(l,d,k)$ such that \begin{equation}\label{eq:NegConcl2} \mathcal R_{\mu}(Z)<\left(\frac38-\frac{1}{4\lambda C'(\sqrt{\lambda}\Lambda_0k',\tau,r')}\right)\frac{\varepsilon}{\lambda}. \end{equation} Since $\lambda\geq1$ and $C'(\sqrt{\lambda}\Lambda_0k',\tau,r')>1$, the lower bound in \eqref{eq:NegConcl2} is positive and is stronger than the threshold established in Case \textbf{II)}. \medskip \noindent\textbf{IV)} We can therefore describe the decision rule directly in terms of the alternatives appearing in the statement of the theorem. In Case \textbf{I)}, the algorithm returns alternative (A), with the stronger bound $\mathcal R_{\mu}(Z)\leq\varepsilon$. In Case \textbf{II)}, the negative conclusion \eqref{eq:NegConcl1} always yields alternative (B); if in addition $\zeta_{r'}(\hat P)<C''(\sqrt{\lambda}\Lambda_0k',r')$, the preceding argument also yields alternative (A), and the algorithm returns ``(A) and (B)''. In Case \textbf{III)}, the negative conclusion \eqref{eq:NegConcl2} implies alternative (B). We have thus shown that, on the event where the uniform deviation bound \eqref{eq1:ApproxControl} holds, the algorithm returns one of the alternatives (A), (B), and ``(A) and (B)''. This event has probability at least $1-\alpha$. It remains only to rewrite the quantitative bounds in the form stated in the theorem. Since $k'=C'''k$, set $$C':=\sqrt{\lambda}\Lambda_0C''',$$ so that $$\mathcal H_n(l,d,\sqrt{\lambda}\Lambda_0k')=\mathcal H_n(l,d,C'k).$$ The constant $C'$ is universal once the discretization parameter $\eta$ has been fixed universally. We then absorb the factors appearing in the risk bounds into a constant $C>1$ depending only on $k$; since $R$ and $\tau$ are fixed geometric parameters, their dependence is absorbed into this constant. Enlarging $C$ by a fixed numerical factor if necessary allows the strict negative bounds above to be replaced by the non-strict threshold appearing in the statement. This yields the two alternatives of the theorem and completes the proof. \end{proof}

The sample complexity of the testing algorithm follows from the ERM estimate established in \Cref{sec:IPRSamCompBounds}. While the latter is independent of the condition-number bound $k$, the testing procedure requires the ERM accuracy parameter to be chosen in terms of the geometric tolerance $\varepsilon$ and the constants entering the comparison between the auxiliary and geometric risks. Consequently, the sample complexity of the testing procedure depends on $k$.

\begin{cor}\label{cor:TestingSampleComplexity} Under the assumptions of \Cref{thm:MainTheoremPartI}, let $m_{\operatorname{test}}^*(n,l,d,k,\varepsilon,\alpha)$ denote the sample complexity of the testing algorithm. Then there exists a constant $C_{\operatorname{sc}}(k)>0$ such that $$m_{\operatorname{test}}^*(n,l,d,k,\varepsilon,\alpha)\leq\frac{C_{\operatorname{sc}}(k)}{\varepsilon^2}\bigg[\log\frac{2}{\alpha}+(n-l)N(n,d)\log\left(\frac{C_{\operatorname{sc}}(k)}{\varepsilon}\right)\bigg].$$
\end{cor}

 \begin{proof} Let $$\varepsilon':=\frac{1}{4}\cdot\frac{\varepsilon}{\lambda C'(\sqrt{\lambda}\Lambda_0k',\tau,r')},$$ where $k'=C'''k$ and the constants are those introduced in the proof of \Cref{thm:MainTheoremPartI}. By \Cref{IPRthm:SampleCom}, the sample complexity of the ERM procedure with accuracy parameter $\varepsilon'$ satisfies $$m_{\operatorname{ERM}}^*(n,n-l,d,\varepsilon',\alpha)\leq\frac{128}{(\varepsilon')^2}\left[\log\frac{2}{\alpha}+N'\log\left(1+\frac{128}{\varepsilon'}\right)\right],$$ where $N'=(n-l)N(n,d)$. Hence it is sufficient for the testing algorithm to take a sample of size $$m\geq\frac{128}{(\varepsilon')^2}\left[\log\frac{2}{\alpha}+N'\log\left(1+\frac{128}{\varepsilon'}\right)\right].$$ Substituting the definition of $\varepsilon'$ gives $$m\geq\frac{2048\,\lambda^2 C'(\sqrt{\lambda}\Lambda_0k',\tau,r')^2}{\varepsilon^2}\left[\log\frac{2}{\alpha}+(n-l)N(n,d)\log\left(1+\frac{512\,\lambda C'(\sqrt{\lambda}\Lambda_0k',\tau,r')}{\varepsilon}\right)\right].$$ Since $\lambda$ and $\Lambda_0$ are universal once the discretization parameters are fixed, while $k'=C'''k$, the quantities appearing in the prefactor and inside the logarithm depend only on $k$, with the dependence on the fixed geometric parameters $R$ and $\tau$ absorbed into the constants. Moreover, since $0<\varepsilon\leq4R^2$, the term inside the logarithm can be bounded by a quantity of the form $C_{\operatorname{sc}}'(k)/\varepsilon$. Therefore, enlarging the constant if necessary, there exists $C_{\operatorname{sc}}(k)>0$ such that $$m_{\operatorname{test}}^*(n,l,d,k,\varepsilon,\alpha)\leq\frac{C_{\operatorname{sc}}(k)}{\varepsilon^2}\left[\log\frac{2}{\alpha}+(n-l)N(n,d)\log\left(\frac{C_{\operatorname{sc}}(k)}{\varepsilon}\right)\right].$$
 \end{proof}

\subsection{Sample Complexity and Statistical Consistency}
\label{sec:IPRSamCompBounds}
We now derive sample complexity bounds for the empirical risk minimization procedure, starting from the corresponding notion of statistical consistency. Consider the random variable $(X^{(1)},\dots,X^{(m)})$ taking values in $\mathcal{S}_m(\mathcal{X})$ with law $\mu^m$. A learning algorithm then induces an $\mathcal{H}$--valued random variable $\mathcal{A}[X^{(1)},\dots,X^{(m)}]$.

\begin{defn}
A learning algorithm $\mathcal{A}\colon\mathcal{S}(\mathcal{X})\to\mathcal{H}$ is said to be \emph{consistent} with respect to $\mathcal{M}(\mathcal{X})$ and the risk functional $\mathcal{R}$ if, for every $\varepsilon>0$ and $\alpha\in(0,1)$, there exists $m_0\in\mathbb{N}$ such that, for every $\mu\in\mathcal{M}(\mathcal{X})$ and every $m\ge m_0$,
\begin{equation}\label{prel:ConsAlg}
\mathbb{P}\left(\mathcal{R}_\mu(\mathcal{A}[X^{(1)},\dots,X^{(m)}])-\inf_{h\in\mathcal{H}}\mathcal{R}_\mu(h)>\varepsilon\right)<\alpha,
\end{equation}
where the probability is taken with respect to the product measure $\mu^m$.
\end{defn}

The \emph{sample complexity} of a learning algorithm $\mathcal{A}$, denoted by $m^*(\varepsilon,\alpha,\mathcal{H},\mathcal{A})$, is defined as the smallest integer $m_0\in\mathbb{N}$ such that condition \eqref{prel:ConsAlg} holds for every $\mu\in\mathcal{M}(\mathcal{X})$ and every $m\ge m_0$.

We emphasize that the sample complexity considered in this section refers to the statistical accuracy of the ERM procedure. Although the feasible set $\mathcal{I}_{\mathcal N}(c,d,k)$ depends on the condition-number bound $k$, the sample complexity estimate obtained below is independent of $k$. This should be distinguished from the sample complexity of the final testing algorithm, for which the statistical accuracy required from the ERM procedure is chosen in terms of the geometric tolerance and the condition-number bound, and therefore depends on $k$ through the geometric comparison estimates.

The sample complexity bound follows from Hoeffding's inequality and a uniform deviation estimate over bounded-degree polynomial systems, obtained through a covering argument. The required covering number is controlled by \Cref{lem:CoveringPolyn}.In the following, we denote by $N=N(n,d)$ the dimension of $\mathbb{R}[x_1,\dots,x_n]_{\le d}$, and by $N':=cN$ the dimension of $\mathbb{R}[x_1,\dots,x_n]^c_{\le d}$.

\begin{lem}\label{lem:CoveringPolyn}
For every $n,k,d\in\mathbb N$ and every $\varepsilon>0$, there exists $r_0=r_0(d,R)>0$ such that for every $r\le r_0$ and every finite net $\mathcal N$ of $D_R^n$ of radius $r$, there exist polynomials $P_1,\dots,P_{\nu(\varepsilon)}\in \R[x_1,\dots,x_n]^c_{\le d}$ satisfying
$$\mathcal I_{\mathcal N}(c,d,k)\subseteq \bigcup_{i=1}^{\nu(\varepsilon)} B_\varepsilon(P_i) \qquad \text{and} \qquad  \nu(\varepsilon)\le \left(1+\frac{4}{\varepsilon}\right)^{N'},$$
where $B_\varepsilon(P_i)$ denotes the ball of radius $\varepsilon$ centered at $P_i$ with respect to the norm $\|\cdot\|_{C^2(D_R^n)}$, and each $P_i$ satisfies $\|P_i\|_{C^2(D_R^n)}\le 2$.
\end{lem}
\begin{proof}
Fix $\varepsilon>0$. By \Cref{thm:disc_C2_ball_one_net_delta}, choosing the parameter in that result equal to $1/2$, there exists $r_0=r_0(d,R)>0$ such that for every $r\le r_0$ and every polynomial $P\in\R[x_1,\dots,x_n]^c_{\le d}$ one has $$\|P\|_{C^2}\le 2\|P\|_{\mathcal N}.$$ Therefore, if $P\in \mathcal I_{\mathcal N}(c,d,k)$, by definition of $\mathcal I_{\mathcal N}(c,d,k)$ we have $\|P\|_{\mathcal N}= 1$, and hence $\|P\|_{C^2}\le 2$. It follows that $$\mathcal I_{\mathcal N}(c,d,k)\subseteq \left\{Q\in \R[x_1,\dots,x_n]^c_{\le d}\ \middle|\ \|Q\|_{C^2}\le 2\right\}.$$ Since $\R[x_1,\dots,x_n]^c_{\le d}$ has dimension $N'$, a standard covering argument implies that the set
$$\left\{Q\in \R[x_1,\dots,x_n]^c_{\le d}\ \middle|\ \|Q\|_{C^2}\le 2\right\}$$
can be covered by at most
$$\left(1+\frac{4}{\varepsilon}\right)^{N'}$$
balls of radius $\varepsilon$ with respect to the norm $\|\cdot\|_{C^2}$. Hence, there exist polynomials $P_1,\dots,P_{\nu(\varepsilon)}$ with $\|P_i\|_{C^2}\le 2$ such that
$$\left\{Q\in \R[x_1,\dots,x_n]^c_{\le d}\ \middle|\ \|Q\|_{C^2}\le 2\right\}\subseteq \bigcup_{i=1}^{\nu(\varepsilon)} B_\varepsilon(P_i).$$
Since $\mathcal I_{\mathcal N}(c,d,k)$ is contained in this set, the conclusion follows.
\end{proof}

We recall Hoeffding's inequality in the form used below. We denote by $\mathcal{S}_m$ the space of $m$-samples $\mathcal{S}_m(D_R^n)$, by $\underline{x}_m$ a generic element of $\mathcal{S}_m$, and by $\mathbb{P}$ the probability induced by the product measure $\mu^m$ on $\mathcal{S}_m$.
\begin{lem}[Hoeffding's inequality]\label{lem:Hoeffdings inequality}
Let $X$ be a real-valued random variable such that $a\le X\le b$ almost surely, and let $X_1,\dots,X_m$ be i.i.d.\ copies of $X$. Then, for every $\gamma>0$, $$\mathbb{P}\!\left(\left|\frac{1}{m}\sum_{j=1}^m X_j-\mathbb{E}[X]\right|>\gamma\right)\le 2\exp\!\left(-\frac{2m\gamma^2}{(b-a)^2}\right).$$
\end{lem}

We first derive a local uniform deviation bound in a $C^2$--neighborhood of a fixed polynomial.

\begin{lem}\label{lem:local-dev}
There exists $r_0=r_0(d,R)>0$ such that, for every $r$--net $\mathcal N$ of $D_R^n$ with $r\leq r_0$, the following holds. Let $\mu\in\mathcal M_{\tau}(D_R^n)$ and let $X$ be a random variable with law $\mu$. Let $X_1,\dots,X_m$ be i.i.d.\ copies of $X$. Let $P_0\in\R[x_1,\dots,x_n]^c_{\le d}$ satisfy $|P_0|_{C^2(D_R^n)}\le 2$. Let $\varepsilon>0$ and let $\varepsilon'>0$ be such that $\varepsilon'\le \varepsilon/16$. Then, for every $m\in\mathbb N$, $$\mathbb{P}\Bigg(\sup_{\substack{P:\,|P-P_0|_{C^2(D_R^n)}\le \varepsilon'\\ P\in\mathcal I_{\mathcal N}(c,d,k)}}\bigg|\frac{1}{m}\sum_{i=1}^m \n{P(X_i)}^2-\mathbb E\big[\n{P(X)}^2\big]\bigg|>\varepsilon\Bigg)\le 2e^{-\frac{m\varepsilon^2}{32}}.$$
\end{lem}

\begin{proof}
For any $x\in D_R^n$ and any $P$ with $\|P-P_0\|_{C^2}\le \varepsilon'$, we have $$\|P(x)-P_0(x)\|\le \|P-P_0\|_{C^2}\le \varepsilon'.$$ Moreover, since $P\in\mathcal I_{\mathcal N}(c,d,k)$, we have $\n{P}_{\mathcal N}=1$, and therefore, by \Cref{prop:DiscretizedC2Norm} with $\eta=1/2$, $$\n{P}_{C^2}\le2.$$ Since $\n{P_0}_{C^2}\le2$ by assumption, it follows that for every $x\in D_R^n$, $\n{P(x)}\le2$ and $\n{P_0(x)}\le2$. Therefore, for every $x\in D_R^n$, $\|P(x)\|\le 2$ and $\|P_0(x)\|\le 2$. It follows that $$\begin{aligned}
\big|\|P(x)\|^2-\|P_0(x)\|^2\big|
&=\big|\|P(x)\|-\|P_0(x)\|\big|\big(\|P(x)\|+\|P_0(x)\|\big)\\
&\le \|P(x)-P_0(x)\|\big(\|P(x)\|+\|P_0(x)\|\big)\\
&\le 4\varepsilon'.
\end{aligned}$$ Hence, $$\left|\frac{1}{m}\sum_{i=1}^m\|P(X_i)\|^2-\frac{1}{m}\sum_{i=1}^m\|P_0(X_i)\|^2\right|\le 4\varepsilon'$$ and $$\left|\mathbb E\|P(X)\|^2-\mathbb E\|P_0(X)\|^2\right|\le 4\varepsilon'.$$ Therefore, by the triangle inequality,
\begin{equation}\label{lem:local-dev_eq1}
\left|\frac{1}{m}\sum_{i=1}^m\|P(X_i)\|^2-\mathbb E\|P(X)\|^2\right|
\le \left|\frac{1}{m}\sum_{i=1}^m\|P_0(X_i)\|^2-\mathbb E\|P_0(X)\|^2\right|+8\varepsilon'.
\end{equation} Assume $\varepsilon'\le \varepsilon/16$, so that $8\varepsilon'\le \varepsilon/2$. Then \eqref{lem:local-dev_eq1} implies that, for every such $P$, $$\left|\frac{1}{m}\sum_{i=1}^m\|P(X_i)\|^2-\mathbb E\|P(X)\|^2\right|>\varepsilon$$ forces $$\left|\frac{1}{m}\sum_{i=1}^m\|P_0(X_i)\|^2-\mathbb E\|P_0(X)\|^2\right|>\frac{\varepsilon}{2}.$$ Consider now the event $$A(P,\varepsilon):=\Bigg\{\underline{x}_m\in\mathcal S_m\ \bigg|\ \bigg|\frac{1}{m}\sum_{i=1}^m\|P(x^{(i)})\|^2-\mathbb E\big[\|P(X)\|^2\big]\bigg|>\varepsilon\Bigg\}.$$ Since $\|P_0\|_{C^2}\le 2$ by assumption, we have $$0\le \|P_0(X)\|^2\le 4$$ almost surely. By \Cref{lem:Hoeffdings inequality}, for every $\varepsilon>0$ and $m\in\mathbb N$ we have $$\mu^m\left(A\left(P_0,\frac{\varepsilon}{2}\right)\right)\le 2\exp\!\left(-\frac{2m(\varepsilon/2)^2}{4^2}\right)=2\exp\!\left(-\frac{m\varepsilon^2}{32}\right).$$ Together with \eqref{lem:local-dev_eq1}, this implies $$\mu^m\Bigg(\underline{x}_m\in\mathcal S_m\ \bigg|\ \sup_{\substack{P\in\mathcal I_{\mathcal N}(c,d,k):\\ \|P-P_0\|_{C^2}\le \varepsilon'}} \bigg|\frac{1}{m}\sum_{i=1}^m\|P(x^{(i)})\|^2-\mathbb E\big[\|P(X)\|^2\big]\bigg|>\varepsilon\Bigg)\le 2e^{-\frac{m\varepsilon^2}{32}}.$$
This concludes the proof.
\end{proof}

Combining \Cref{lem:local-dev} with the covering estimate of \Cref{lem:CoveringPolyn}, we obtain the following uniform deviation bound over the entire feasible set.
\begin{lem}
\label{lem:UniformDeviationBoundAuxiliary}
Fix $c,d,k$ and let $\mu\in\mathcal M_{\tau}(D_R^n)$, and let $X$ be a random variable with law $\mu$. There exists $r_0=r_0(d,R)>0$ such that, for every $r$--net $\mathcal N$ of $D_R^n$ with $r\leq r_0$, for any $\alpha\in(0,1)$, $\varepsilon>0$, and $$m\geq\frac{32}{\varepsilon^2}\bigg[\log\frac{2}{\alpha}+N'\log\left(1+\frac{64}{\varepsilon}\right)\bigg],$$
we have $$\mu^m\bigg(\underline{x}_m\in\mathcal S_m\ \bigg|\ \sup_{P\in\mathcal I_{\mathcal N}(c,d,k)}\left|\frac{1}{m}\sum_{i=1}^m\n{P(x^{(i)})}^2-\mathbb E\big[\n{P(X)}^2\big]\right|<\varepsilon\bigg)>1-\alpha.$$
\end{lem}
\begin{proof}
Fix $\varepsilon>0$ and set $\varepsilon':=\varepsilon/16$. Let $r_0=r_0(d,R)>0$ be chosen so that the conclusions of \Cref{lem:CoveringPolyn} and \Cref{prop:DiscretizedC2Norm} hold for every $r$--net $\mathcal N$ of $D_R^n$ with $r\leq r_0$. Since the net in the statement satisfies $r\leq r_0$, \Cref{lem:CoveringPolyn} provides polynomials $P_1,\dots,P_{\nu(\varepsilon')}\in\R[x_1,\dots,x_n]^c_{\le d}$ such that $$\mathcal I_{\mathcal N}(c,d,k)\subseteq\bigcup_{j=1}^{\nu(\varepsilon')}B_{\varepsilon'}(P_j),\qquad \nu(\varepsilon')\leq\left(1+\frac{4}{\varepsilon'}\right)^{N'}.$$ For each $j=1,\dots,\nu(\varepsilon')$, define the event $$E_j(m,\varepsilon):=\Bigg\{\underline{x}_m\in\mathcal S_m\ \bigg|\ \sup_{\substack{P\in\mathcal I_{\mathcal N}(c,d,k)\\ \n{P-P_j}_{C^2(D_R^n)}\leq\varepsilon'}}\left|\frac{1}{m}\sum_{i=1}^m\n{P(x^{(i)})}^2-\E\big[\n{P(X)}^2\big]\right|>\varepsilon\Bigg\}.$$ By \Cref{lem:local-dev}, for every $j$ and every $m\in\mathbb N$, $$\mu^m\big(E_j(m,\varepsilon)\big)\leq2e^{-\frac{m\varepsilon^2}{32}}.$$ Now suppose that $\underline{x}_m\in\mathcal S_m$ satisfies $$\sup_{P\in\mathcal I_{\mathcal N}(c,d,k)}\left|\frac{1}{m}\sum_{i=1}^m\n{P(x^{(i)})}^2-\E\big[\n{P(X)}^2\big]\right|>\varepsilon.$$ Then there exists $P\in\mathcal I_{\mathcal N}(c,d,k)$ for which the preceding deviation is larger than $\varepsilon$. By the covering property, there exists $j\in\{1,\dots,\nu(\varepsilon')\}$ such that $\n{P-P_j}_{C^2(D_R^n)}\leq\varepsilon'$, and hence $\underline{x}_m\in E_j(m,\varepsilon)$. Therefore, $$\Bigg\{\underline{x}_m\in\mathcal S_m\ \bigg|\ \sup_{P\in\mathcal I_{\mathcal N}(c,d,k)}\left|\frac{1}{m}\sum_{i=1}^m\n{P(x^{(i)})}^2-\E\big[\n{P(X)}^2\big]\right|>\varepsilon\Bigg\}\subseteq\bigcup_{j=1}^{\nu(\varepsilon')}E_j(m,\varepsilon).$$ By the union bound, $$\mu^m\Bigg(\underline{x}_m\in\mathcal S_m\ \bigg|\ \sup_{P\in\mathcal I_{\mathcal N}(c,d,k)}\left|\frac{1}{m}\sum_{i=1}^m\n{P(x^{(i)})}^2-\E\big[\n{P(X)}^2\big]\right|>\varepsilon\Bigg)\leq\nu(\varepsilon')\,2e^{-\frac{m\varepsilon^2}{32}}.$$ Since $\varepsilon'=\varepsilon/16$, $$\nu(\varepsilon')\leq\left(1+\frac{4}{\varepsilon'}\right)^{N'}=\left(1+\frac{64}{\varepsilon}\right)^{N'}.$$ Consequently, $$\nu(\varepsilon')\,2e^{-\frac{m\varepsilon^2}{32}}\leq2\exp\left(N'\log\left(1+\frac{64}{\varepsilon}\right)-\frac{m\varepsilon^2}{32}\right).$$ The right-hand side is at most $\alpha$ whenever $$m\geq\frac{32}{\varepsilon^2}\left[\log\frac{2}{\alpha}+N'\log\left(1+\frac{64}{\varepsilon}\right)\right].$$ Therefore the desired uniform deviation bound holds with probability at least $1-\alpha$. This concludes the proof.
\end{proof}

We now derive the sample complexity bound for empirical risk minimization. Let $$\mathcal{A}_{\operatorname{erm}}:\mathcal{S}(D_R^n)\to\mathcal{I}_{\mathcal N}(c,d,k)$$ be an ERM algorithm that, given a sample $\{x^{(1)},\dots,x^{(m)}\}$, returns a minimizer of \eqref{scellOpt}.

\begin{thm}\label{IPRthm:SampleCom} There exists $r_0=r_0(d,R)>0$ such that, for every $r$--net $\mathcal N$ of $D_R^n$ with $r\leq r_0$, the sample complexity $m_{\operatorname{ERM}}^*(n,c,d,\varepsilon,\alpha)$ of any ERM algorithm over the discretized feasible set $\mathcal{I}_{\mathcal{N}}(c,d,k)$ is upper bounded by $$m_{\operatorname{ERM}}^*(n,c,d,\varepsilon,\alpha)\le \frac{128}{\varepsilon^2}\bigg[ \log \frac{2}{\alpha}+N'\log\Big(1+\frac{128}{\varepsilon}\Big)\bigg].$$ \end{thm} \begin{proof} By definition, consider $$\mu^m \bigg( \underline{x}_m \in \mathcal{S}_m \ \bigg|\ \eval_{\mu}\left(\mathcal{A}_{\erm}\big[\underline{x}_m\big] \right)-\min_{P \in \mathcal{I}_{\mathcal{N}}(c,d,k)} \eval_{\mu}(P)<\varepsilon\bigg).$$ \noindent \textbf{I)} For any $\underline{x}_m \in \mathcal{S}_m$, denote by $P_{\underline{x}_m}:=\mathcal{A}_{\erm}\big[\underline{x}_m\big]$. We have $$ \begin{aligned} \underbrace{\eval_{\mu}\left(P_{\underline{x}_m}\right)-\min_{P \in\mathcal{I}_{\mathcal{N}}} \eval_{\mu}(P)}_{(*)} & \le \underbrace{\bigg| \eval_{\mu}\left(P_{\underline{x}_m} \right) -\frac{1}{m}\sum_{i=1}^m |P_{\underline{x}_m}\big(x^{(i)}\big)|^2 \bigg|}_{(a)} \\ &\quad + \underbrace{\bigg|\frac{1}{m}\sum_{i=1}^m |P_{\underline{x}_m}\big(x^{(i)}\big)|^2 -\min_{P \in\mathcal{I}_{\mathcal{N}}} \frac{1}{m}\sum_{i=1}^m |P\big(x^{(i)}\big)|^2 \bigg|}_{(b)} \\ &\quad + \underbrace{\bigg| \min_{P \in \mathcal{I}_{\mathcal{N}}} \frac{1}{m}\sum_{i=1}^m |P\big(x^{(i)}\big)|^2 -\min_{P \in \mathcal{I}_{\mathcal{N}}} \eval_{\mu}(P)\bigg|}_{(c)} . \end{aligned} $$ By definition of ERM and compactness of $\mathcal I_{\mathcal N}(c,d,k)$, the term $(b)$ vanishes. Define $$ S_{\varepsilon}^*=\{\underline{x}_m\in \mathcal{S}_m \mid (*)\le \varepsilon\},\qquad S^a_{\varepsilon}=\{\underline{x}_m\in \mathcal{S}_m \mid (a)\le \varepsilon\},\qquad S^c_{\varepsilon}=\{\underline{x}_m\in \mathcal{S}_m \mid (c)\le \varepsilon\}. $$ Clearly, $$S^*_{\varepsilon}\supseteq S^a_{\varepsilon/2}\cap S^c_{\varepsilon/2},$$ and therefore $$\mu^m(S^*_{\varepsilon})\ge \mu^m(S^a_{\varepsilon/2}\cap S^c_{\varepsilon/2}).$$ \noindent \textbf{II)} Consider the set \begin{equation} S^u_{\varepsilon}:=\bigg\{ \underline{x}_m \in \mathcal{S}_m \ \bigg| \ \sup_{P \in \mathcal{I}_{\mathcal{N}}(c,d,k)} \left| \frac{1}{m} \sum_{i=1}^m |P(x^{(i)})|^2 -\E\big[ |P(X)|^2\big]\right|<\varepsilon \bigg\}. \label{set} \end{equation} For any $\underline{x}_m\in S^u_{\varepsilon/2}$ and any $P \in \mathcal{I}_{\mathcal{N}}(c,d,k)$ we have \begin{equation} \label{eq:IneqMeanAux} \left| \frac{1}{m} \sum_{i=1}^m |P\big(x^{(i)}\big)|^2-\E\big[ |P(X)|^2\big]\right|<\frac{\varepsilon}{2}. \end{equation} In particular, $$\bigg|\eval_{\mu} \left(P_{\underline{x}_m}\right)-\frac{1}{m}\sum_{i=1}^m |P_{\underline{x}_m}\big(x^{(i)}\big)|^2 \bigg| \le \frac{\varepsilon}{2},$$ so that $S^u_{\varepsilon /2}\subseteq S_{\varepsilon /2}^{a}$. On the other hand, taking the infimum in \eqref{eq:IneqMeanAux} yields $$\min_{P \in \mathcal{I}_{\mathcal{N}}} \E\big[ |P(X)|^2\big]- \frac{\varepsilon}{2}\le \min_{P \in \mathcal{I}_{\mathcal{N}}} \frac{1}{m} \sum_{i=1}^m |P(x^{(i)})|^2\le \min_{P \in \mathcal{I}_{\mathcal{N}}}\E\big[ |P(X)|^2\big]+\frac{\varepsilon}{2}, $$ hence $S^u_{\varepsilon /2}\subseteq S_{\varepsilon /2}^{c}$. Therefore, $$\mu^m(S^*_{\varepsilon})\ge \mu^m(S^u_{\varepsilon /2}).$$ Finally, by \Cref{lem:UniformDeviationBoundAuxiliary}, for $$m \ge \frac{128}{\varepsilon^2} \bigg[ \log \frac{2 }{\alpha}+N' \log \Big(1+\frac{128 }{\varepsilon}\Big)\bigg],$$ we have $\mu^m(S^u_{\varepsilon /2})\ge 1-\alpha$, and therefore $\mu^m(S^*_{\varepsilon})\ge 1-\alpha$. \end{proof} 

For the hypothesis class $\mathcal{H}_n(l,d,k)$, with $c=n-l$, consider the regime in which $l,d,k$ are fixed while the ambient dimension $n$ tends to infinity. In this setting, the dimension $N'$ of the coefficient space $\R[x_1,\dots,x_n]^c_{\le d}$ satisfies $$N' = c \binom{n+d}{d} = O(n^{d+1}),$$ for fixed degree $d$. Substituting this growth of $N'$ into the sample complexity bound, and observing that $$ \log\left(\frac{2}{\alpha}\right)+N'\log\left(1+\frac{128}{\varepsilon}\right) = O\big(N'\log(128/\varepsilon)\big), $$ we obtain that, for fixed $l,d,k$ and $\varepsilon,\alpha\in(0,1)$, $$m_{\operatorname{ERM}}^*(n,l,d,\varepsilon,\alpha) = O\left(\frac{1}{\varepsilon^2}n^{d+1}\log\left(\frac{128}{\varepsilon}\right)\right).$$ In particular, for fixed degree $d$ and fixed accuracy parameters, the sample complexity grows at most polynomially in the ambient dimension $n$.\clearpage

\bigskip
\begin{center}
{\Large\bf Part II}\\[0.3cm]
{\large Effective Algorithms and Complexity Bounds}
\end{center}
\addcontentsline{toc}{part}{Part II: Effective Algorithms and Complexity Bounds}
\bigskip

\section{Algorithmic Optimization Framework}\label{sec:AlgorithmicOptimizationFramework}
In \Cref{subsec:DecisionSample}, and in particular in the main theorem of Part I, \Cref{thm:MainTheoremPartI}, we showed that the testing problem reduces to a constrained optimization problem over polynomial systems for the empirical risk functional $\mathcal E_m$, whose algorithmic solution is provided by \Cref{thm:RelaxedAdmissibleGlobalOptimizationPartI}. Part II is devoted to the proof of this latter result. Since the observed sample is regarded as fixed throughout Part II, we simplify the notation by writing $\mathcal E:=\mathcal E_m$ for the empirical auxiliary cost introduced in Part I. We therefore consider the problem of constructing an algorithm that, given an observed sample $x^{(1)},\ldots,x^{(m)}$, solves the optimization problem $$\min_{P\in\mathcal I_{\mathcal N}(c,d,k)}\mathcal E(P),\qquad \mathcal E(P):=\frac{1}{m}\sum_{i=1}^m\n{P(x^{(i)})}^2.$$ Here $\mathcal I_{\mathcal N}(c,d,k)\subseteq S_{\mathcal N}^{cN-1}$ is obtained by discretizing the discriminant functional $\delta_{D_R^n}$ and the $C^2$--norm over a net $\mathcal N$ of $D_R^n$, where $$S_{\mathcal N}^{cN-1}:=\left\{P\in\mathbb R[x_1,\ldots,x_n]_{\le d}^c\ \middle|\ \n{P}_{\mathcal N}=1\right\}.$$ We first reformulate this optimization problem in terms of a matrix representation of polynomial systems, which provides the finite-dimensional setting used in the localization and discretization arguments below.

Throughout Part II, we work in an extended Blum--Shub--Smale model allowing exact spectral decomposition of real symmetric matrices and exact evaluation of the tensor spectral norms appearing below. Arithmetic complexity is measured by the number of arithmetic operations and comparisons performed. In this idealized model, we assign a cost $\mathcal O(c^3)$ to the spectral decomposition of an $c\times c$ real symmetric matrix, following the conventional dense numerical linear algebra operation count. The cost of tensor spectral norm evaluation is denoted by $T_{\mathrm{tensor}}(n,c)$ and is left unspecified.

\subsection{Matrix Representation of Polynomial Systems}\label{sec:MatrixRep}
Let $N=N(n,d)$ denote the dimension of the vector space $\mathbb R[x_1,\dots,x_n]_{\leq d}$, and fix a basis $\mathcal B=\{b_1(x),\dots,b_N(x)\}$. We identify each polynomial system with its coefficient matrix with respect to $\mathcal B$. Under this identification, the empirical cost, the discretized norm, and the discriminant functional admit finite-dimensional matrix representations. This viewpoint is inspired by Darboux kernel methods; see \cite{Pepeejin} for applications in TDA and \cite{lasserre2015introduction} for a general overview.  Define $$P_{\mathcal B}(x):=(b_1(x),\dots,b_N(x))^\top\in\mathbb R[x_1,\dots,x_n]_{\leq d}^N.$$ Then every polynomial map $P\in\mathbb R[x_1,\dots,x_n]_{\leq d}^c$ can be uniquely represented in the form \begin{equation}\label{eq:PolyMatrix} P(x)=A P_{\mathcal B}(x), \end{equation} where $A\in\Mat(c\times N,\mathbb R)$ is the coefficient matrix of $P$. Differentiating this representation gives $$\Jac P(x)=A J_{\mathcal B}(x),\qquad J_{\mathcal B}(x):=\Jac P_{\mathcal B}(x),$$ where $J_{\mathcal B}(x)$ is the matrix whose rows are the gradients of the basis elements $b_1,\dots,b_N$. The following \Cref{prop:JetSec} collects some basic properties of this matrix-valued map.

\begin{prop}\label{prop:JetSec}
Let $d\geq 1$ and let $\mathcal B$ be a basis of $\R[x_1,\dots,x_n]_{\leq d}$. Then:
\\
i) A $c\times n$ matrix $V(x)$ with polynomial entries is the Jacobian of some $P \in \R[x_1,\dots,x_n]_{\leq d}^c$ if and only if $V(x)=A J_{\mathcal B}(x)$ for a matrix $A \in \Mat(c\times N,\R)$.
\\
ii) For every $x\in\R^n$, the matrix $J_{\mathcal B}(x)$ has rank $n$.
\end{prop}

\begin{proof} i) Suppose that $V(x)=A J_{\mathcal B}(x)$ for some $A \in \Mat(c\times N,\R)$. Define $P_A(x)\in \R[x_1,\dots,x_n]^c_{\le d}$ by $P_A(x):=A P_{\mathcal B}(x)$. Then $$\Jac P_A(x)=A J_{\mathcal B}(x)=V(x).$$ Conversely, if $V(x)=\Jac P(x)$ for some $P\in \R[x_1,\dots,x_n]^c_{\le d}$, then there exists a unique matrix $A$ such that $P(x)=A P_{\mathcal B}(x)$, and therefore $$V(x)=\Jac P(x)=A J_{\mathcal B}(x).$$ ii) Fix $x\in\R^n$ and consider the linear map $\mathcal D_x:\R[x_1,\dots,x_n]_{\le d}\to\R^n$ defined by $$\mathcal D_x(p)=\nabla p(x).$$ For any basis $\mathcal B=\{b_1,\dots,b_N\}$, the image of $\mathcal D_x$ coincides with the span of $\{\nabla b_i(x)\}$. Since $d\ge1$, the space of affine functions
$$A:=\left\{a_0+\sum_{j=1}^n a_j x_j \mid a_j\in\R\right\}$$
is contained in $\R[x_1,\dots,x_n]_{\le d}$, and on this space $\mathcal D_x$ acts as $$\mathcal D_x\!\left(a_0+\sum_{j=1}^n a_j x_j\right)=(a_1,\dots,a_n).$$ Hence $\mathcal D_x$ is surjective, which implies $\rank \mathcal D_x=n$, and therefore $\rank J_{\mathcal B}(x)=n$ for every $x\in\R^n$.
\end{proof}

We now express the auxiliary empirical risk, the discriminant functional, and the discretized norm as functionals on $\Mat(c\times N,\mathbb R)$ under the above identification. Consider the linear isomorphism $\varphi_{\mathcal B}:\Mat(c\times N,\mathbb R)\to\mathbb R[x_1,\dots,x_n]_{\leq d}^c$, defined by $$\varphi_{\mathcal B}(A):=A P_{\mathcal B}(x).$$ We denote the corresponding pull-backs of the auxiliary empirical risk, the discriminant functional, and the discretized norm by $$\mathcal E_{\mathcal B}(A):=\mathcal E(\varphi_{\mathcal B}(A)),\qquad \delta_{\mathcal B,\mathcal N}(A):=\delta_{\mathcal N}(\varphi_{\mathcal B}(A)),\qquad \n{A}_{\mathcal B,\mathcal N}:=\n{\varphi_{\mathcal B}(A)}_{\mathcal N}.$$ To simplify the notation, in the remainder of the paper we omit the subscript $\mathcal B$ and denote $\mathcal E_{\mathcal B}$, $\delta_{\mathcal B,\mathcal N}$, and $\n{A}_{\mathcal B,\mathcal N}$ by $\mathcal E$, $\delta_{\mathcal N}$, and $\n{A}_{\mathcal N}$, respectively, specifying the dependence on the chosen basis $\mathcal B$ whenever relevant.

\begin{prop}
\label{prop:MatrixRepresentationCost}
Let $\mathcal B=\{b_1,\ldots,b_N\}$ be a basis of $\mathbb R[x_1,\ldots,x_n]_{\leq d}$. Then there exists a positive semidefinite matrix $V\in\Sym(N,\mathbb R)$ such that $$\mathcal{E}(A)=\operatorname{tr}(AVA^\top). $$
\end{prop}

\begin{proof}
Define $$V:=\frac1m\sum_{j=1}^m P_{\mathcal B}(x^{(j)})P_{\mathcal B}(x^{(j)})^\top.$$ Then $V$ is symmetric and positive semidefinite. Moreover, using $\|Aw\|^2=\operatorname{tr}(Aww^\top A^\top)$, we obtain $$\mathcal E(A) =\frac1m\sum_{j=1}^m\|AP_{\mathcal B}(x^{(j)})\|^2=\operatorname{tr}(AVA^\top).$$
\end{proof}

The same Gram-matrix representation applies to the pointwise quantities entering the discriminant functional.

\begin{prop}
\label{prop:MatrixRepresentationDiscriminant}
For every $y_i\in\mathcal N$, there exist matrices $G_i^{(0)},G_i^{(1)}\in\Sym(N,\mathbb R)$ with $\operatorname{rank}(G_i^{(1)})=n$, such that $$\delta_{\mathcal N}(A)=\min_{1\leq i\leq|\mathcal N|}\sqrt{\operatorname{tr}(AG_i^{(0)}A^\top)+\lambda_{\min}(AG_i^{(1)}A^\top)}.$$ 
\end{prop}

\begin{proof} Fix $y_i\in\mathcal N$ and define $$G_i^{(0)}:=P_{\mathcal B}(y_i)P_{\mathcal B}(y_i)^\top,\qquad G_i^{(1)}:=\operatorname{Jac}P_{\mathcal B}(y_i)\operatorname{Jac}P_{\mathcal B}(y_i)^\top.$$ Both matrices are positive semidefinite. Moreover, $$\|P(y_i)\|^2=\operatorname{tr}(AG_i^{(0)}A^\top) \qquad \text{and} \qquad \operatorname{Jac}P(y_i)\operatorname{Jac}P(y_i)^\top=AG_i^{(1)}A^\top.$$ Therefore, $$\sigma_{\min}(\operatorname{Jac}P(y_i))^2=\lambda_{\min}(AG_i^{(1)}A^\top).$$ Hence, $$\begin{aligned}\delta_{\mathcal N}(A)&=\min_{1\leq i\leq|\mathcal N|}\sqrt{\|P(y_i)\|^2+\sigma_{\min}(\operatorname{Jac}P(y_i))^2}\\&=\min_{1\leq i\leq|\mathcal N|}\sqrt{\operatorname{tr}(AG_i^{(0)}A^\top)+\lambda_{\min}(AG_i^{(1)}A^\top)}.\end{aligned}$$ It remains to verify the rank property. By \Cref{prop:JetSec}, the matrix $\operatorname{Jac}P_{\mathcal B}(y_i)$ has rank $n$. Therefore, $$\operatorname{rank}(G_i^{(1)})=\operatorname{rank}\left(\operatorname{Jac}P_{\mathcal B}(y_i)\operatorname{Jac}P_{\mathcal B}(y_i)^\top\right)=n.$$ \end{proof}
We finally express the discretized norm in terms of the coefficient matrix, including the second-order contribution from the Hessian.

\begin{prop}
\label{prop:MatrixRepresentationNorm}
For every $y_i\in\mathcal N$, let $G_i^{(0)}$ and $G_i^{(1)}$ be the matrices introduced in \Cref{prop:MatrixRepresentationDiscriminant} and $S^{n-1}\subseteq \R^n$ the unit sphere. Then there exists a map $$G_i^{(2)}:S^{n-1}\times S^{n-1}\to\Sym(N,\mathbb R)$$ such that $$\n{A}_{\mathcal N}^{2}=\max_{1\leq i\leq|\mathcal N|}\left\{\operatorname{tr}(AG_i^{(0)}A^\top)+\lambda_{\max}(AG_i^{(1)}A^\top)+\max_{v_1,v_2\in S^{n-1}}\operatorname{tr}\left(AG_i^{(2)}(v_1,v_2)A^\top\right)\right\}. $$
\end{prop}

\begin{proof} Fix $y_i\in\mathcal N$, and let $G_i^{(0)}$ and $G_i^{(1)}$ be as in \Cref{prop:MatrixRepresentationDiscriminant}. For every $v_1,v_2\in S^{n-1}$, define $$ G_i^{(2)}(v_1,v_2) := D^2P_{\mathcal B}(y_i)[v_1,v_2] D^2P_{\mathcal B}(y_i)[v_1,v_2]^\top . $$ The matrices $G_i^{(2)}(v_1,v_2)$ are positive semidefinite. Moreover, for every $v_1,v_2\in S^{n-1}$, $$ \begin{aligned} \|D^2P(y_i)[v_1,v_2]\|^2 &= \|A D^2P_{\mathcal B}(y_i)[v_1,v_2]\|^2 = \operatorname{tr} \left( AG_i^{(2)}(v_1,v_2)A^\top \right). \end{aligned} $$ Therefore, $$ \|D^2P(y_i)\|_{\operatorname{op}}^2 = \max_{v_1,v_2\in S^{n-1}} \operatorname{tr} \left( AG_i^{(2)}(v_1,v_2)A^\top \right). $$ Using the definition of the discretized norm, $$ \n{A}_{\mathcal N}^{2} = \max_{1\leq i\leq|\mathcal N|} \left( \|P(y_i)\|^2 + \|\operatorname{Jac}P(y_i)\|_{\operatorname{op}}^2 + \|D^2P(y_i)\|_{\operatorname{op}}^2 \right), $$ we obtain $$ \begin{aligned} \n{A}_{\mathcal N}^{2} = \max_{1\leq i\leq|\mathcal N|} \Bigg\{ & \operatorname{tr}(AG_i^{(0)}A^\top) + \lambda_{\max}(AG_i^{(1)}A^\top)  + \max_{v_1,v_2\in S^{n-1}} \operatorname{tr} \left( AG_i^{(2)}(v_1,v_2)A^\top \right) \Bigg\}. \end{aligned} $$ \end{proof}

The matrices introduced above can be interpreted as Gram matrices associated with the zeroth- and first-order jets of the basis $\mathcal B$ at the points of $\mathcal N$, while the maps $G_i^{(2)}$ encode the corresponding second-order jet contributions.

\section{Localized Net of Polynomial Systems}
Using the matrix representation introduced in \Cref{sec:MatrixRep}, the constrained optimization problem takes the form
\begin{equation}\label{sec2Intro:ConstrOpt}
\inf_{U\in \mathcal{I}_{\mathcal N}(c,d,k)}\mathcal E(U)
\qquad\text{where}\qquad
\mathcal E(A)=\operatorname{tr}(AVA^\top).
\end{equation}
Since $V$ is symmetric and positive semidefinite, its low spectral subspaces provide a natural localization of the low-cost region in coefficient space.In \Cref{subsec:Unconstrained_Optimization_Problem}, we show that the sublevel sets $$\mathcal U(\varepsilon^*):=\left\{A\in\Mat(c\times N,\mathbb R)\ \middle|\ \mathcal E(A)\leq\varepsilon^*\right\}$$ are contained in tubular neighborhoods of the low spectral subspaces of $V$. We then estimate the covering complexity of these localized regions in \Cref{subsec:LocalizedNet} and incorporate the admissibility constraints in \Cref{subsec:Admissible_Localized_Net}. An explicit construction of the corresponding localized nets is given in \Cref{sec:Explicit_Net_Construction}.

\subsection{Unconstrained Optimization Problem}
\label{subsec:Unconstrained_Optimization_Problem}
Let $V\in\Sym(N,\mathbb R)$ be the positive semidefinite matrix associated with the observed sample by \Cref{prop:MatrixRepresentationCost}, so that $\mathcal E(A)=\operatorname{tr}(AVA^\top)$. Choose an ordered orthonormal eigenbasis $v_1,\dots,v_N$ of $\mathbb R^N$ such that  $$Vv_j=\varepsilon_jv_j, \qquad 0\leq\varepsilon_1\leq\varepsilon_2\leq\cdots\leq\varepsilon_N,$$ where $\varepsilon_1,\dots,\varepsilon_N$ are the eigenvalues of $V$, counted with multiplicity. For every $1\leq r<N$, let $$ E_r:=\operatorname{span}\{v_1,\dots,v_r\}\subseteq\mathbb R^N $$ be the low spectral subspace associated with the $r$ smallest eigenvalues, and let $\pi_r:\mathbb R^N\to E_r$ be the corresponding orthogonal projection. We equip $\Mat(c\times N,\mathbb R)$ with the Frobenius inner product $$\langle A,B\rangle_F:=\operatorname{tr}(AB^\top)$$ and the associated Frobenius norm $\|\cdot\|_F$. If $a_1,\dots,a_c\in\mathbb R^N$ are the rows of $A\in\Mat(c\times N,\mathbb R)$, we extend $\pi_r$ row-wise by defining $$ \Pi_r:\Mat(c\times N,\mathbb R)\longrightarrow\Mat(c\times N,\mathbb R),
\qquad
\Pi_r(A):=
\begin{pmatrix}
\pi_r(a_1)\\
\vdots\\
\pi_r(a_c)
\end{pmatrix}.
$$

The row-wise projection $\Pi_r$ naturally determines the matrix counterpart of the low spectral subspace $E_r$.

\begin{defn}
Let $V\in\Sym(N,\mathbb R)$ be positive semidefinite, and let $v_1,\dots,v_N$ be an ordered orthonormal eigenbasis of $V$. For $r<N$, let $E_r=\operatorname{span}\{v_1,\dots,v_r\}\subseteq\mathbb R^N$, and let $\Pi_r$ be the row-wise orthogonal projection defined above. We define the corresponding low spectral subspace $\mathcal{V}_r\subseteq \Mat(c\times N,\mathbb R)$ by $$\mathcal{V}_r=\big\{A\in\Mat(c\times N,\mathbb R) \ | \ \Pi_r(A)=A\big\}. $$
\end{defn}

Equivalently, $\mathcal V_r$ consists of the coefficient matrices whose rows have no components along the eigenvectors $v_{r+1},\dots,v_N$. Since $\Pi_r$ is an orthogonal projection, the distance from $A$ to $\mathcal V_r$ in the Frobenius norm is $$\operatorname{dist}_F(A,\mathcal V_r)=\|A-\Pi_r(A)\|_F.$$ For $\rho>0$,  define the tubular neighborhood of radius $\rho$ around $\mathcal V_r$ by $$\operatorname{Tube}_F(\mathcal V_r,\rho):=\left\{A\in\Mat(c\times N,\mathbb R)\ \middle|\ \operatorname{dist}_F(A,\mathcal V_r)\leq\rho\right\}.$$

The sublevel set $\mathcal U(\varepsilon^*)$ satisfies the following spectral localization estimate.

\begin{prop}
\label{prop:SpecLoc}
For every $1\leq r<N$ such that $\varepsilon_{r+1}>0$, $$\mathcal U(\varepsilon^*)\subseteq\operatorname{Tube}_F\left(\mathcal V_r,\sqrt{\frac{\varepsilon^*}{\varepsilon_{r+1}}}\right).$$
\end{prop}
\begin{proof}
By the definition of $\operatorname{Tube}_F(\mathcal V_r,\rho)$, it is enough to show that every $A\in\mathcal U(\varepsilon^*)$ satisfies $$\operatorname{dist}_F(A,\mathcal V_r)\leq\sqrt{\frac{\varepsilon^*}{\varepsilon_{r+1}}}.$$ Fix $A\in\mathcal U(\varepsilon^*)$. Let $v_1,\dots,v_N$ be an orthonormal basis that diagonalizes $V$. Each row $a_i$ of $A$ can be expanded as $$a_i=\sum_{j=1}^N \alpha_{ij}v_j.$$ Since $v_1,\dots,v_N$ diagonalize $V$, we have $$a_i V a_i^{\top}=\sum_{j=1}^N \varepsilon_j\alpha_{ij}^2.$$ Therefore, $$\mathcal E(A)=\operatorname{tr}(AVA^{\top})=\sum_{i=1}^c a_i V a_i^{\top}=\sum_{i=1}^c\sum_{j=1}^N \varepsilon_j\alpha_{ij}^2.$$ Since $\varepsilon_j\geq \varepsilon_{r+1}$ for every $j>r$, splitting the sum at $r$ gives
$$
\begin{aligned}
\mathcal E(A)=\sum_{i=1}^c\sum_{j=1}^r \varepsilon_j\alpha_{ij}^2+\sum_{i=1}^c\sum_{j=r+1}^N \varepsilon_j\alpha_{ij}^2& \geq \sum_{i=1}^c\sum_{j=r+1}^N \varepsilon_j\alpha_{ij}^2 \geq \varepsilon_{r+1}\sum_{i=1}^c\sum_{j=r+1}^N \alpha_{ij}^2.
\end{aligned}
$$ 
Using $\mathcal E(A)\leq\varepsilon^*$, we obtain
\begin{equation}
\label{prop:Spectral Localization_a)}
\sum_{i=1}^c\sum_{j=r+1}^N \alpha_{ij}^2\leq \frac{\varepsilon^*}{\varepsilon_{r+1}}.
\end{equation}
On the other hand, identifying $\Pi_r$ with right multiplication by the corresponding projection matrix on $\mathbb R^N$, we have
\begin{equation}
\label{prop:Spectral Localization_b)}
\operatorname{dist}_F(A,\mathcal{V}_r)^2=\|A(I-\Pi_r)\|_F^2=\sum_{i=1}^c\left\|\sum_{j=r+1}^N \alpha_{ij}v_j\right\|^2=\sum_{i=1}^c\sum_{j=r+1}^N \alpha_{ij}^2.
\end{equation}
Combining \eqref{prop:Spectral Localization_a)} with \eqref{prop:Spectral Localization_b)}, we conclude that $$\operatorname{dist}_F(A,\mathcal{V}_r)^2\leq \frac{\varepsilon^*}{\varepsilon_{r+1}}.$$ This proves the claim.
\end{proof}
Since all norms are equivalent on the finite-dimensional space
$\Mat(c\times N,\mathbb R)$, the spectral localization established in \Cref{prop:SpecLoc} can be formulated in the $\n{\cdot}\_{\mathcal N}$ norm. We recall this norm equivalence and introduce the corresponding constants, which will be used in the later discussion of arithmetic complexity.

\begin{prop}
\label{prop:NormsEquivalence}
There exist constants $m_{\mathcal N},M_{\mathcal N}>0$ such that, for every $A\in\Mat(c\times N,\mathbb R)$, $$m_{\mathcal N}\|A\|_F \le \|A\|_{\mathcal N} \le M_{\mathcal N}\|A\|_F. $$
\end{prop}

We denote by $\operatorname{dist}_{\mathcal N}$ and $\operatorname{Tube}_{\mathcal N}$ the distance and tubular neighborhoods associated with $\|\cdot\|_{\mathcal N}$. Then \Cref{prop:NormsEquivalence} yields the following discrete-norm version of the spectral localization estimate.

\begin{cor}
\label{cor:SpectralLocalizationDiscreteNorm}
For every $r<N$ such that $\varepsilon_{r+1}>0$, one has
\begin{equation}\label{eq:TubeInc}
\mathcal U(\varepsilon^*)\subseteq\operatorname{Tube}_{\mathcal N}\left(\mathcal V_r,M_{\mathcal N}\sqrt{\frac{\varepsilon^*}{\varepsilon_{r+1}}}\right).
\end{equation}
\end{cor}

\begin{proof}
Let $A\in\mathcal U(\varepsilon^*)$. Since $\Pi_r(A)\in\mathcal V_r$, the definition of the distance and \Cref{prop:NormsEquivalence} give
$$\operatorname{dist}_{\mathcal N}(A,\mathcal V_r)\leq\|A-\Pi_r(A)\|_{\mathcal N}\leq M_{\mathcal N}\|A-\Pi_r(A)\|_F=M_{\mathcal N}\operatorname{dist}_F(A,\mathcal V_r).$$
The conclusion follows from \Cref{prop:SpecLoc}.
\end{proof}

For every $1\leq r<N$ such that $\varepsilon_{r+1}>0$, set \begin{equation}\label{eq:LocalRadius} \rho_r(\varepsilon^*):=M_{\mathcal N}\sqrt{\frac{\varepsilon^*}{\varepsilon_{r+1}}}. \end{equation} The choice of $r$ determines a trade-off between the dimension of the low spectral subspace and the transverse localization radius. Indeed, $$\dim(\mathcal V_1)\leq\cdots\leq\dim(\mathcal V_r)\leq\cdots\leq\dim(\mathcal V_N),$$ while $$\rho_1(\varepsilon^*)\geq\cdots\geq\rho_r(\varepsilon^*)\geq\cdots\geq\rho_N(\varepsilon^*).$$ Thus, increasing $r$ enlarges the low spectral component to be covered while reducing the transverse radius.

For the total-degree Bernstein basis, \Cref{Bernstein_Polynomials_Basis} gives the dimension-independent bound $M_{\mathcal N}\leq C_{\mathrm{norm}}$. Hence $$\rho_r(\varepsilon^*)\leq C_{\mathrm{norm}}\sqrt{\frac{\varepsilon^*}{\varepsilon_{r+1}}},$$ and the transverse contribution to the covering estimate satisfies $$\rho_r(\varepsilon^*)^{c(N-r)} \leq \left(C_{\mathrm{norm}}^2\frac{\varepsilon^*}{\varepsilon_{r+1}}\right)^{\frac{c(N-r)}{2}}.$$

The spectral index $r$ is chosen to minimize the covering number of the corresponding tube. Since the radii $\rho_r(\varepsilon^*)$ depend on the eigenvalues of the empirical matrix $V$, this choice is data-dependent. The resulting reduction in covering complexity is therefore governed by the empirical spectral distribution and is stronger in the presence of a larger spectral gap above the low spectral subspace. Whenever $$\varepsilon_1,\ldots,\varepsilon_r\leq\varepsilon^*\qquad\text{and}\qquad C_{\mathrm{norm}}^2\varepsilon^*<\varepsilon_{r+1},$$ one has $\rho_r(\varepsilon^*)<1$. In this regime, the effective dimension reduction is determined by the number of eigenvalues below the target level, while the transverse contribution is controlled by $$\left(C_{\mathrm{norm}}^2\frac{\varepsilon^*}{\varepsilon_{r+1}}\right)^{\frac{c(N-r)}{2}}.$$ Thus, spectral localization is most effective when few eigenvalues lie below $\varepsilon^*$ and the spectral gap above them is large. The choice of $r$ balances the dimension $cr$ of the low spectral component against the transverse factor $\rho_r(\varepsilon^*)^{c(N-r)}$.

\subsection{Localized Net and Theoretical Covering Reduction}
\label{subsec:LocalizedNet}
In this subsection we derive an intrinsic estimate for the cardinality of an $h$--net of the tubes $\mathcal T_{\mathcal N}(r,\rho)$, for arbitrary index $r$ and radius $\rho>0$. This estimate does not rely on an explicit construction of the net, but quantifies the reduction in the covering complexity induced by the spectral localization. The explicit net construction is postponed to \Cref{sec:Explicit_Net_Construction}. For $1\leq r<N$ and $\rho>0$, define $$\mathcal T_{\mathcal N}(r,\rho):= S_{\mathcal N}^{cN-1}\cap \operatorname{Tube}_{\mathcal N}(\mathcal V_r,\rho). $$ We can prove the following.

\begin{prop} \label{prop:IntrinsicLocalizedTubeCovering} Let $0<h<1$ and $0<\rho<1$. Then there exists an $h$--net $\mathcal G_h(r,\rho)\subseteq S_{\mathcal N}^{cN-1}$ of $\mathcal T_{\mathcal N}(r,\rho)$ with respect to $\n{\cdot}_{\mathcal N}$ such that $$|\mathcal G_h(r,\rho)|\leq\left(1+\frac{C}{h}\right)^{cr}\left(1+\frac{C\rho}{h}\right)^{c(N-r)},$$ where $C>0$ is a universal constant. \end{prop}
\begin{proof}
i) The tubular neighborhood $\operatorname{Tube}_{\mathcal N}(\mathcal V_r,\rho)$ is determined by the distance from the linear subspace $\mathcal V_r$. Consider then the quotient space $$\mathcal V_r^{\mathrm{transv}}:=\Mat(c\times N,\mathbb R)/\mathcal V_r.$$ Since $\mathcal V_r$ is defined through the projection onto the low spectral subspace $E_r$, which is generated by the first $r$ eigenvectors of $V$, we have $$\dim\mathcal V_r=cr, \qquad \qquad \dim\bigl(\mathcal V_r^{\mathrm{transv}}\bigr)=c(N-r).$$ Consider then the quotient map $\hat{\Pi}_r:\Mat(c\times N,\R)\to\mathcal V_r^{\mathrm{transv}}$, where $\hat{\Pi}_r(A)=[A]$ denotes the equivalence class of $A$ modulo $\mathcal V_r$. We equip the quotient space with the quotient norm
\begin{equation}\label{eq:QuotNorm}
\n{[A]}_{\mathcal N/\mathcal V_r}:=\inf_{B\in\mathcal V_r}\n{A-B}_{\mathcal N}.
\end{equation}
Since $\Mat(c\times N,\mathbb R)$ is finite dimensional, \eqref{eq:QuotNorm} defines a well-defined norm on the quotient space. We define the low spectral ball and the transversal ball of radius $R>0$, respectively, by $$B_{\mathcal N}^{\mathcal V_r}(R):=\Big\{B\in\mathcal V_r\ \Big|\ \n{B}_{\mathcal N}\leq R\Big\}\qquad \text{and} \qquad B_{\mathcal N/\mathcal V_r}(R):=\Big\{U\in\mathcal V_r^{\mathrm{transv}}\ \Big|\ \n{U}_{\mathcal N/\mathcal V_r}\leq R\Big\}.$$ Consider now an $h/2$--net $\mathcal Q_h$ and an $h/4$--net $\mathcal Q_h^{\mathrm{transv}}$ of $$B_{\mathcal N}^{\mathcal V_r}(1+\rho+h/4)\qquad \text{and} \qquad B_{\mathcal N/\mathcal V_r}(\rho),$$ respectively, with respect to the corresponding norms. Remark that, since $\Mat(c\times N,\mathbb R)$ is finite dimensional, the infimum in the definition of the quotient norm is attained. Therefore, for every $U\in\mathcal Q_h^{\mathrm{transv}}$, we can choose a representative $Z_U\in\Mat(c\times N,\mathbb R)$ satisfying $$\hat{\Pi}_r(Z_U)=U, \qquad \n{Z_U}_{\mathcal N}=\n{U}_{\mathcal N/\mathcal V_r},$$ and therefore, in particular, $\n{Z_U}_{\mathcal N}\leq\rho$. Define the finite set $$\widetilde{\mathcal G}_h(r,\rho):=\Big\{B+Z_U\ \big|\ B\in\mathcal Q_h,\ U\in\mathcal Q_h^{\mathrm{transv}}\Big\}.$$
\
\\
ii) We now prove that $\widetilde{\mathcal G}_h(r,\rho)$ is an $h$--net of $\mathcal T_{\mathcal N}(r,\rho)$. Let $A\in\mathcal T_{\mathcal N}(r,\rho)$. Then by definition we have $$\n{A}_{\mathcal N}=1 \qquad \text{and} \qquad \operatorname{dist}_{\mathcal N}(A,\mathcal V_r)\leq\rho.$$ The second inequality, together with \eqref{eq:QuotNorm}, implies $$\n{\hat{\Pi}_r(A)}_{\mathcal N/\mathcal V_r}=\operatorname{dist}_{\mathcal N}(A,\mathcal V_r)\leq\rho.$$ Hence, $\hat{\Pi}_r(A)\in B_{\mathcal N/\mathcal V_r}(\rho)$ and therefore there exists $U\in\mathcal Q_h^{\mathrm{transv}}$ such that
\begin{equation}\label{eq:TransvBound}
\n{\hat{\Pi}_r(A)-U}_{\mathcal N/\mathcal V_r}\leq\frac{h}{4}.
\end{equation}
Let $Z_U\in\Mat(c\times N,\mathbb R)$ satisfy $\hat{\Pi}_r(Z_U)=U$. Then by \eqref{eq:TransvBound} we have $$\n{\hat{\Pi}_r(A-Z_U)}_{\mathcal N/\mathcal V_r}\leq\frac h4.$$ Therefore, by the definition of the quotient norm, there exists $B\in\mathcal V_r$ such that $$\n{A-Z_U-B}_{\mathcal N}\leq\frac h4.$$ Moreover, $$\begin{aligned}\n{B}_{\mathcal N}&\leq\n{A}_{\mathcal N}+\n{Z_U}_{\mathcal N}+\n{A-Z_U-B}_{\mathcal N}\leq1+\rho+\frac h4.\end{aligned}$$ Hence, $B\in B_{\mathcal N}^{\mathcal V_r}(1+\rho+h/4)$ and therefore there exists $C\in\mathcal Q_h$ such that $$\n{B-C}_{\mathcal N}\leq\frac h2.$$ Consequently, $$\begin{aligned}\n{A-(C+Z_U)}_{\mathcal N}&\leq\n{A-Z_U-B}_{\mathcal N}+\n{B-C}_{\mathcal N}\leq\frac{3h}{4}<h.\end{aligned}$$ Thus $\widetilde{\mathcal G}_h(r,\rho)$ is an $h$--net of $\mathcal T_{\mathcal N}(r,\rho)$.
\
\\
iii) The set $\widetilde{\mathcal G}_h(r,\rho)$ constructed above is not necessarily contained in the normalized sphere. We therefore normalize its nonzero elements. Consider the radial normalization map $$\Phi_{\mathcal N}:\Mat(c\times N,\mathbb R)\setminus\{0\}\longrightarrow S_{\mathcal N}^{cN-1},\qquad \Phi_{\mathcal N}(A):=\frac{A}{\n{A}_{\mathcal N}}. $$ For every nonzero $Y\in\widetilde{\mathcal G}_h(r,\rho)$, define $$\Phi_{\mathcal N}(Y):=\frac{Y}{\n{Y}_{\mathcal N}},$$ and set $$\mathcal G_h(r,\rho):=\left\{\Phi_{\mathcal N}(Y)\ \middle|\ Y\in\widetilde{\mathcal G}_h(r,\rho),\ Y\neq0\right\}.$$ By construction, $\mathcal G_h(r,\rho)\subseteq S_{\mathcal N}^{cN-1}$. We now observe that, for every $A\in S_{\mathcal N}^{cN-1}$ and every nonzero $B\in\Mat(c\times N,\mathbb R)$, $$\n{\Phi_{\mathcal N}(B)-A}_{\mathcal N}\leq 2\n{B-A}_{\mathcal N}.$$ Indeed, since $\n{A}_{\mathcal N}=1$, one has $$
\begin{aligned}
\n{\Phi_{\mathcal N}(B)-A}_{\mathcal N}
&\leq \n{\Phi_{\mathcal N}(B)-B}_{\mathcal N}+\n{B-A}_{\mathcal N}\\
&=\left|\frac{1}{\n{B}_{\mathcal N}}-1\right|\n{B}_{\mathcal N}+\n{B-A}_{\mathcal N}\\
&=\left|1-\n{B}_{\mathcal N}\right|+\n{B-A}_{\mathcal N}\\
&\leq 2\n{B-A}_{\mathcal N},
\end{aligned}
$$ where the last inequality follows from the reverse triangle inequality. Therefore, applying the construction in parts (i)--(ii) with resolution parameter $h/2$, for every $A\in\mathcal T_{\mathcal N}(r,\rho)$ there exists $Y\in\widetilde{\mathcal G}_{h/2}(r,\rho)$ such that $$\n{A-Y}_{\mathcal N}\leq\frac h2.$$ Since $h<1$ and $\n{A}_{\mathcal N}=1$, such a point $Y$ is necessarily nonzero. Hence, $$\n{A-\Phi_{\mathcal N}(Y)}_{\mathcal N}\leq2\n{A-Y}_{\mathcal N}\leq h.$$ Thus, after radial normalization, the construction at resolution $h/2$ yields an $h$--net $\mathcal G_h(r,\rho)\subseteq S_{\mathcal N}^{cN-1}$ of $\mathcal T_{\mathcal N}(r,\rho)$.
\
\\
iv) It remains to estimate the cardinality of this net. By the volumetric covering estimate in the $cr$--dimensional space $\mathcal V_r$, and since the quotient space has dimension $c(N-r)$, the nets used in the construction satisfy $$|\mathcal Q_h|\leq\left(1+\frac{C}{h}\right)^{cr},\qquad\text{and}\qquad|\mathcal Q_h^{\mathrm{transv}}|\leq\left(1+\frac{C\rho}{h}\right)^{c(N-r)},$$ for a universal constant $C>0$. Consequently, $$|\widetilde{\mathcal G}_h(r,\rho)|\leq|\mathcal Q_h|\,|\mathcal Q_h^{\mathrm{transv}}|\leq\left(1+\frac{C}{h}\right)^{cr}\left(1+\frac{C\rho}{h}\right)^{c(N-r)}.$$ Applying the same estimates to the construction at resolution $h/2$, and absorbing the resulting fixed numerical factors into the universal constant $C$, gives $$|\mathcal G_h(r,\rho)|\leq\left(1+\frac{C}{h}\right)^{cr}\left(1+\frac{C\rho}{h}\right)^{c(N-r)}.$$ This proves the claim.
\end{proof}

We compare the covering bound for $\mathcal T_{\mathcal N}(r,\rho)$ with the corresponding bound for the full sphere. Denote these bounds by $\mathcal C_h(\mathcal T_{\mathcal N}(r,\rho))$ and $\mathcal C_h(S_{\mathcal N}^{cN-1})$, respectively, using the same universal constant $C>0$, and define the associated metric entropy by $$H_h(\mathcal X):=\log \mathcal C_h(\mathcal X).$$  Since $$\mathcal C_h(\mathcal T_{\mathcal N}(r,\rho))=\left(1+\frac{C}{h}\right)^{cr}\left(1+\frac{C\rho}{h}\right)^{c(N-r)},\qquad \mathcal C_h(S_{\mathcal N}^{cN-1})=\left(1+\frac{C}{h}\right)^{cN},$$ the corresponding reduction in metric entropy is $$H_h(S_{\mathcal N}^{cN-1})-H_h(\mathcal T_{\mathcal N}(r,\rho))=c(N-r)\log\left(\frac{h+C}{h+C\rho}\right).$$ In particular, for $0<\rho<1$, $$\lim_{h\to0}\left[H_h(S_{\mathcal N}^{cN-1})-H_h(\mathcal T_{\mathcal N}(r,\rho))\right]=c(N-r)\log\frac1\rho.$$ Thus the asymptotic reduction in metric entropy increases with both $N-r$ and $1/\rho$.

\subsection{Constrained Optimization Problem and Admissible Localized Net}
\label{subsec:Admissible_Localized_Net}
We now incorporate the admissibility constraint into the localized covering. Define the set of discretely admissible systems by $$\mathcal K_{\mathcal N}(c,d,k):=\left\{P\in\mathbb R[x_1,\ldots,x_n]_{\le d}^c\ \middle|\ K_{\mathcal N}(P)\le k\right\},\qquad K_{\mathcal N}(P):=\frac{\n{P}_{\mathcal N}}{\delta_{\mathcal N}(P)},$$ and the feasible set can therefore be written as $\mathcal I_{\mathcal N}(c,d,k)=S_{\mathcal N}^{cN-1}\cap\mathcal K_{\mathcal N}(c,d,k).$  For a fixed spectral index $r$ and radius $\rho$, define $$\mathcal F_{\mathcal N}(r,\rho,k):=\mathcal T_{\mathcal N}(r,\rho)\cap\mathcal K_{\mathcal N}(c,d,k).$$ To construct a covering of $\mathcal F_{\mathcal N}(r,\rho,k)$ from an $h$--net of $\mathcal T_{\mathcal N}(r,\rho)$, we first establish the stability of the condition number under perturbations in the discretized norm.

\begin{prop}\label{prop:ConditionNumberStability} 
Fix $k\geq1$ and $h>0$. Let $A,B\in S_{\mathcal N}^{cN-1}$ satisfy $$K_{\mathcal N}(A)\leq k,\qquad \n{A-B}_{\mathcal N}\leq h.$$ Then $$K_{\mathcal N}(B)\leq\frac{k}{\sqrt{1-2k^2h}},\qquad \text{provided that}\qquad h<\frac{1}{2k^2}.$$ \end{prop} \begin{proof} Since $A\in S_{\mathcal N}^{cN-1}$ and $K_{\mathcal N}(A)\leq k$, we have $$\delta_{\mathcal N}(A)^2\geq\frac{1}{k^2}.$$ By \Cref{lem:LipschitzDiscriminant}, $$\delta_{\mathcal N}(B)^2\geq\delta_{\mathcal N}(A)^2-2\n{A-B}_{\mathcal N}\geq\delta_{\mathcal N}(A)^2-2h.$$ Using $\delta_{\mathcal N}(A)^2\geq1/k^2$, we obtain $$\delta_{\mathcal N}(B)^2\geq\delta_{\mathcal N}(A)^2\left(1-2k^2h\right).$$ Since $B\in S_{\mathcal N}^{cN-1}$, one has $\n{B}_{\mathcal N}=1$, and therefore $$K_{\mathcal N}(B)^2=\frac{1}{\delta_{\mathcal N}(B)^2}\leq\frac{k^2}{1-2k^2h}.$$ The assumption $h<1/(2k^2)$ guarantees that the denominator is positive. Taking the square root proves the claim. \end{proof} 

Define the relaxation factor $$\Lambda(h):=\frac{1}{\sqrt{1-2k^2h}}.$$ Then $\Lambda(h)\to1$ as $h\to0$, so the admissibility relaxation vanishes with the resolution scale. Since $\Lambda(h)\to1$ as $h\to0$, the loss in admissibility vanishes with the resolution scale. Moreover, $\Lambda(h)$ is increasing on the interval $0<h<1/(2k^2)$. Hence, if $$0<h\leq h_0\leq\frac{1}{4k^2},$$ then $$\Lambda(h)\leq\Lambda(h_0)\leq\Lambda\left(\frac{1}{4k^2}\right)=\sqrt{2}.$$ Therefore, defining $\Lambda_0:=\sqrt{2}$, we obtain the uniform bound $K_{\mathcal N}(B)\leq k\Lambda_0$ for every $$h\leq\frac{1}{4k^2}.$$ This stability estimate yields the following relaxed-admissible covering of the constrained tube.

 \begin{prop}\label{prop:RelaxedAdmissibleCovering} Let $\mathcal G_h(r,\rho)\subseteq S_{\mathcal N}^{cN-1}$ be an $h$--net of $\mathcal T_{\mathcal N}(r,\rho)$ with respect to $\n{\cdot}_{\mathcal N}$, and define the relaxed-admissible net $$\mathcal G_h^+(r,\rho,k):=\left\{G\in\mathcal G_h(r,\rho)\ \middle|\ K_{\mathcal N}(G)\leq k\Lambda(h)\right\}.$$ Then $\mathcal G_h^+(r,\rho,k)$ is an $h$--covering of $\mathcal F_{\mathcal N}(r,\rho,k)$. More precisely, $$\mathcal F_{\mathcal N}(r,\rho,k)\subseteq\bigcup_{G\in\mathcal G_h^+(r,\rho,k)}B_{\mathcal N}(G,h),\qquad \text{provided that}\qquad 0<h<\frac{1}{2k^2}.$$ \end{prop} 
\begin{proof} Let $A\in\mathcal F_{\mathcal N}(r,\rho,k)$. Since $\mathcal G_h(r,\rho)$ is an $h$--net of $\mathcal T_{\mathcal N}(r,\rho)$, there exists $G\in\mathcal G_h(r,\rho)$ such that $$\n{A-G}_{\mathcal N}\leq h.$$ Since $K_{\mathcal N}(A)\leq k$ and $A,G \in S_{\mathcal N}^{cN-1} $, \Cref{prop:ConditionNumberStability} gives $$K_{\mathcal N}(G)\leq k\Lambda(h).$$ Hence $G\in\mathcal G_h^+(r,\rho,k)$, and therefore $$\mathcal F_{\mathcal N}(r,\rho,k)\subseteq\bigcup_{G\in\mathcal G_h^+(r,\rho,k)}B_{\mathcal N}(G,h).$$ 

\end{proof}

\section{Algorithm Definition and Error Control}
\label{sec:Algorithm_Definition_Error_Control}
We now combine the spectral localization and relaxed-admissibility results to define the search algorithm and establish its convergence. The resulting optimization procedure is then used to prove \Cref{thm:RelaxedAdmissibleGlobalOptimizationPartI}, while the corresponding arithmetic complexity bounds are derived in \Cref{subsec:Auxiliary_Arithmetic_Complexity_Bounds}.

\subsection{Optimization Algorithm and Convergence}
\label{subsec:Optimization_Algorithm_Convergence}
The procedure consists of constructing a finite set of relaxed-admissible representatives of the localized feasible region and evaluating the empirical cost on the resulting grid. By \Cref{cor:SpectralLocalizationDiscreteNorm}, the restricted sublevel sets $\mathcal U(\varepsilon^*)\cap \mathcal I_{\mathcal N}(c,d,k)$ are contained in the tubular regions $\mathcal F_{\mathcal N}(r,\rho_r(\varepsilon^*),k)$ for every $1\leq r<N$ such that $\varepsilon_{r+1}>0$, where $$\rho_r(\varepsilon^*):=M_{\mathcal N}\sqrt{\frac{\varepsilon^*}{\varepsilon_{r+1}}}.$$

 Once the spectral index $r$ and the resolution parameter $h$ are fixed, \Cref{prop:RelaxedAdmissibleCovering} provides an $h$--covering of the corresponding localized admissible region by relaxed-admissible centers. Since the localization holds for every admissible spectral index $r$, we select $r$ by minimizing the covering number of the corresponding localized region. This motivates the following covering function.
\begin{defn}
\label{def:CoveringFunction}
For $h>0$, $\rho>0$, and $1\leq r\leq N$, we define the covering function by $$\varphi(r,h,\rho):=\begin{cases}\displaystyle \left(1+\frac{C}{h}\right)^{cr}\left(1+\frac{C\rho}{h}\right)^{c(N-r)}, & 0<\rho<1,\\[2mm]\displaystyle \left(1+\frac{C}{h}\right)^{cN}, & \rho\geq1,\end{cases}$$ where $C>0$ is the universal constant appearing in \Cref{prop:IntrinsicLocalizedTubeCovering}.
\end{defn}
We obtain the following relaxed-admissible global optimization procedure.
\begin{alg}[Global Optimization]
\label{alg:relaxed-global}
Fix $k\geq1$, a target cost level $\varepsilon^*>0$, and a resolution parameter $$0<h<\frac{1}{2k^2}.$$
\begin{enumerate}
\item Choose the spectral index $$r^*\in\operatorname*{argmin}_{\substack{1\leq r<N\\ \varepsilon_{r+1}>0}}\varphi(r,h,\rho_r(\varepsilon^*)),$$ where $\varphi$ is the covering function defined in \Cref{def:CoveringFunction}, and set $\rho^*:=\rho_{r^*}(\varepsilon^*)$.
\item Depending on the value of $\rho^*$:
\begin{enumerate}
\item If $\rho^*<1$, construct an $h$--net $\mathcal G_h(r^*,\rho^*)\subseteq S_{\mathcal N}^{cN-1}$ of $\mathcal T_{\mathcal N}(r^*,\rho^*)$.
\item If $\rho^*\geq1$, construct an $h$--net $\mathcal G_h(r^*,\rho^*)\subseteq S_{\mathcal N}^{cN-1}$ of $S_{\mathcal N}^{cN-1}$.
\end{enumerate}
\item Define the relaxed admissible grid $$\mathcal G_h^+(r^*,\rho^*,k):=\left\{A\in\mathcal G_h(r^*,\rho^*)\ \middle|\ K_{\mathcal N}(A)\leq k\Lambda(h)\right\}.$$
\item If $\mathcal G_h^+(r^*,\rho^*,k)=\emptyset$, return \textnormal{EMPTY} and terminate.
\item For every $A\in\mathcal G_h^+(r^*,\rho^*,k)$, evaluate the empirical cost $\mathcal E(A)$.
\item Choose $$G_h^*\in\operatorname*{argmin}_{A\in\mathcal G_h^+(r^*,\rho^*,k)}\mathcal E(A),$$ and return $\widehat U_h:=G_h^*$.
\end{enumerate}
\end{alg}

The following theorem gives the approximation guarantee for \Cref{alg:relaxed-global}, together with the admissibility relaxation and the corresponding bound on the number of evaluated candidates.

\begin{thm}\label{thm:RelaxedAdmissibleGlobalOptimization} Fix $k\geq1$, a target cost level $\varepsilon^*>0$, and a resolution parameter $$0<h<\frac{1}{2k^2}.$$ Assume that $\mathcal I_{\mathcal N}(c,d,k)\cap\mathcal U(\varepsilon^*)\neq\emptyset$. Let $\widehat U_h$ be the output of \Cref{alg:relaxed-global}. Then $\widehat U_h$ satisfies $$\mathcal E(\widehat U_h)\leq\inf_{U\in\mathcal I_{\mathcal N}(c,d,k)}\mathcal E(U)+L_{\mathcal E}h, \qquad \text{and} \qquad K_{\mathcal N}(\widehat U_h)\leq k\Lambda(h).$$ In particular, the approximation error vanishes and the admissibility relaxation tends to $1$ as $h\to0$. Moreover, the number of evaluated candidates is bounded by $\varphi(r^*,h,\rho^*)$, where $$r^*\in\operatorname*{argmin}_{\substack{1\leq r<N\\ \varepsilon_{r+1}>0}}\varphi(r,h,\rho_r(\varepsilon^*)),$$ and $\rho^*:=\rho_{r^*}(\varepsilon^*)$. \end{thm}

 \begin{proof} Consider the normalized feasible set $\mathcal I_{\mathcal N}(c,d,k)$ and a global minimizer $$U^*\in\operatorname*{argmin}_{U\in\mathcal I_{\mathcal N}(c,d,k)}\mathcal E(U)$$ of the constrained problem. Since $\mathcal I_{\mathcal N}(c,d,k)\cap\mathcal U(\varepsilon^*)\neq\emptyset$, we have $\mathcal E(U^*)\leq\varepsilon^*$. Consider the spectral index selected by the algorithm $$r^*\in\operatorname*{argmin}_{\substack{1\leq r<N\\ \varepsilon_{r+1}>0}}\varphi(r,h,\rho_r(\varepsilon^*)),$$ and set $\rho^*:=\rho_{r^*}(\varepsilon^*)$. If $\rho^*<1$, the spectral localization estimate in \Cref{cor:SpectralLocalizationDiscreteNorm} gives $$U^*\in\mathcal F_{\mathcal N}(r^*,\rho^*,k).$$ Hence, since $\mathcal G_h(r^*,\rho^*)$ is an $h$--net of $\mathcal T_{\mathcal N}(r^*,\rho^*)$, there exists $G^*\in\mathcal G_h(r^*,\rho^*)$ such that $$\n{G^*-U^*}_{\mathcal N}\leq h.$$ If $\rho^*\geq1$, the algorithm instead constructs an $h$--net of $S_{\mathcal N}^{cN-1}$. Since $U^*\in S_{\mathcal N}^{cN-1}$, there again exists $G^*\in\mathcal G_h(r^*,\rho^*)$ such that $$\n{G^*-U^*}_{\mathcal N}\leq h.$$ In either case, since $K_{\mathcal N}(U^*)\leq k$, \Cref{prop:ConditionNumberStability} gives $$K_{\mathcal N}(G^*)\leq k\Lambda(h).$$ Hence $G^*\in\mathcal G_h^+(r^*,\rho^*,k)$. By the definition of the algorithm, $\mathcal E(\widehat U_h)\leq\mathcal E(G^*)$. Moreover, by \Cref{lem:LipschitzCost}, $$|\mathcal E(G^*)-\mathcal E(U^*)|\leq L_{\mathcal E}h,$$ and therefore $$\mathcal E(\widehat U_h)\leq\mathcal E(U^*)+L_{\mathcal E}h=\inf_{U\in\mathcal I_{\mathcal N}(c,d,k)}\mathcal E(U)+L_{\mathcal E}h.$$ Finally, since $\widehat U_h\in\mathcal G_h^+(r^*,\rho^*,k)$, $$K_{\mathcal N}(\widehat U_h)\leq k\Lambda(h).$$ Finally, by the construction of $\mathcal G_h(r^*,\rho^*)$ and the definition of the covering function $\varphi$, $$|\mathcal G_h^+(r^*,\rho^*,k)|\leq|\mathcal G_h(r^*,\rho^*)|\leq\varphi(r^*,h,\rho^*).$$ This proves the claim. \end{proof}

The convergence result in \Cref{thm:RelaxedAdmissibleGlobalOptimization} assumes the non-emptiness condition
\begin{equation}
\label{eq:EmptCond}
\mathcal U(\varepsilon^*)\cap\mathcal I_{\mathcal N}(c,d,k)\neq\emptyset.
\end{equation}
This condition can be detected on a finite relaxed-admissible grid at the enlarged cost level $\varepsilon^*+L_{\mathcal E}h$, as shown in the following lemma.
\begin{lem}\label{lem:BufferedAdmissibleDetection} Fix $k\geq1$ and $\bar\varepsilon>0$. Let $r<N$ satisfy $\varepsilon_{r+1}>0$, and let $$0<h<\frac{1}{2k^2}.$$ Let $\mathcal G_h(r,\rho_r(\bar\varepsilon+L_{\mathcal E}h))\subseteq S_{\mathcal N}^{cN-1}$ be an $h$--net of $\mathcal T_{\mathcal N}(r,\rho_r(\bar\varepsilon+L_{\mathcal E}h))$, and define the relaxed-admissible grid $$\mathcal G_h^+(r,\rho_r(\bar\varepsilon+L_{\mathcal E}h),k):=\left\{G\in\mathcal G_h(r,\rho_r(\bar\varepsilon+L_{\mathcal E}h))\ \middle|\ K_{\mathcal N}(G)\leq k\Lambda(h)\right\}.$$ Then, for any $0<\varepsilon^*\leq\bar\varepsilon$: \begin{enumerate} \item[(i)] If $\mathcal I_{\mathcal N}(c,d,k)\cap\mathcal U(\varepsilon^*)\neq\emptyset$, then there exists $G\in\mathcal G_h^+(r,\rho_r(\bar\varepsilon+L_{\mathcal E}h),k)$ such that $$\mathcal E(G)\leq\varepsilon^*+L_{\mathcal E}h.$$ \item[(ii)] If no point of $\mathcal G_h^+(r,\rho_r(\bar\varepsilon+L_{\mathcal E}h),k)$ satisfies $$\mathcal E(G)\leq\varepsilon^*+L_{\mathcal E}h,$$ then $\mathcal I_{\mathcal N}(c,d,k)\cap\mathcal U(\varepsilon^*)=\emptyset$. \end{enumerate} \end{lem} \begin{proof} Assume that $\mathcal I_{\mathcal N}(c,d,k)\cap\mathcal U(\varepsilon^*)\neq\emptyset$, and let $U\in\mathcal I_{\mathcal N}(c,d,k)\cap\mathcal U(\varepsilon^*)$. Then $$\n{U}_{\mathcal N}=1,\qquad K_{\mathcal N}(U)\leq k,\qquad \mathcal E(U)\leq\varepsilon^*.$$ By spectral localization and the assumption $\varepsilon^*\leq\bar\varepsilon$, $$U\in\mathcal T_{\mathcal N}(r,\rho_r(\varepsilon^*))\subseteq\mathcal T_{\mathcal N}(r,\rho_r(\bar\varepsilon+L_{\mathcal E}h)).$$ Since $\mathcal G_h(r,\rho_r(\bar\varepsilon+L_{\mathcal E}h))$ is an $h$--net of $\mathcal T_{\mathcal N}(r,\rho_r(\bar\varepsilon+L_{\mathcal E}h))$, there exists $G\in\mathcal G_h(r,\rho_r(\bar\varepsilon+L_{\mathcal E}h))$ such that $$\n{G-U}_{\mathcal N}\leq h.$$ By \Cref{prop:ConditionNumberStability}, $$K_{\mathcal N}(G)\leq k\Lambda(h),$$ and therefore $$G\in\mathcal G_h^+(r,\rho_r(\bar\varepsilon+L_{\mathcal E}h),k).$$ Moreover, by \Cref{lem:LipschitzCost}, $$\mathcal E(G)\leq\mathcal E(U)+L_{\mathcal E}h\leq\varepsilon^*+L_{\mathcal E}h.$$ This proves part (i). Part (ii) follows immediately by contraposition. \end{proof}

We now turn to the proof of \Cref{thm:RelaxedAdmissibleGlobalOptimizationFinal} . The final step combines the optimization procedure with \Cref{lem:BufferedAdmissibleDetection} and the corresponding arithmetic complexity estimate. The latter is expressed in terms of the cardinality of the underlying localized grid, allowing the bound to be combined with different constructions of localized coverings. The auxiliary estimates required for evaluating the discretized norm and discriminant are collected in \Cref{sec:Explicit_Net_Construction}.

In addition to the cardinality of the localized grid, the arithmetic complexity depends on the cost of evaluating the discretized norm and discriminant at each candidate point. The second-order contribution to $\n{A}_{\mathcal N}$ requires computing the spectral norm of the tensor $D^2P(y_i)\in\mathbb R^c\otimes\operatorname{Sym}^2(\mathbb R^n)$. We denote the arithmetic complexity of this computation by $T_{\mathrm{tensor}}(n,c)$ and keep it abstract in the statement below. The corresponding norm and discriminant evaluation costs, together with possible choices for the tensor-norm computation, are analyzed in \Cref{subsec:Auxiliary_Arithmetic_Complexity_Bounds}.

\begin{thm}\label{thm:RelaxedAdmissibleGlobalOptimizationFinal}
Fix a degree $d$, a condition-number bound $k\geq1$, and $0<\varepsilon\leq\bar\varepsilon$. Then there exists a constant $\Lambda_0>1$ such that the algorithm of \Cref{alg:relaxed-global}, with the localized grid constructed at level $\bar\varepsilon+L_{\mathcal E}h$, either returns \textnormal{EMPTY}, in which case $$\min_{A\in\mathcal I_{\mathcal N}(c,d,k)}\mathcal E(A)>\bar\varepsilon,$$ or outputs a matrix $\hat A\in\mathcal I_{\mathcal N}(c,d,\Lambda_0k)$ satisfying one of the following alternatives:
\begin{enumerate}[label=(\roman*)]
\item The output satisfies $$\mathcal E(\hat A)\leq\varepsilon.$$
\item The output satisfies $\mathcal E(\hat A)>\varepsilon$, and $$\min_{A\in\mathcal I_{\mathcal N}(c,d,k)}\mathcal E(A)\geq\min\left\{\frac{\mathcal E(\hat A)}{2},\bar\varepsilon\right\}.$$
\end{enumerate}
Moreover, for a sample size chosen according to the sufficient bound of \Cref{cor:TestingSampleComplexity} and fixed confidence level, and for fixed $l$, $d\geq1$, $k$, $\varepsilon$, and $\bar\varepsilon$, and for sufficiently large $n$, if $\rho^*<1$, the arithmetic complexity of the algorithm satisfies $$\begin{aligned} T_{\mathrm{alg}}(n,k,l,d,\varepsilon)\leq\mathcal O\bigg(\left(1+\frac{CM_{\mathcal N}k^2}{m_{\mathcal N}\varepsilon}\right)^{cr^*}&\left(1+\frac{CM_{\mathcal N}^2k^2}{m_{\mathcal N}\varepsilon}\sqrt{\frac{\bar\varepsilon}{\varepsilon_{r^*+1}}}\right)^{c(N-r^*)}\cdot \\ &\cdot(Cd^4k)^n\left(n^{d+3}+n^{2d+1}+T_{\mathrm{tensor}}(n,c)\right)\bigg), \end{aligned}$$ where $r^*$ is the spectral index selected by \Cref{alg:relaxed-global}. If $\rho^*\geq1$, the corresponding full-sphere covering factor is bounded by $$\left(1+\frac{CM_{\mathcal N}k^2}{m_{\mathcal N}\varepsilon}\right)^{cN}.$$
\end{thm}

\begin{proof} I°) Fix $0<\varepsilon\leq\bar\varepsilon$ and set $$\varepsilon^*:=\frac{\varepsilon}{2}.$$ Fix also $\eta\leq1/2$ and, by refining the discretization net if necessary, assume that its radius is at most $1/2$. Hence, by \Cref{lem:LipschitzCost}, $$L_{\mathcal E}=2\left(1+\frac{r}{1-\eta}\right)^2\leq8.$$ We fix the resolution parameter as $$h:=\frac{\min\{\varepsilon^*,1\}}{8k^2}.$$ Since $k\geq1$, one has $$h\leq\frac{1}{8k^2}<\frac{1}{4k^2},\qquad L_{\mathcal E}h\leq\frac{\min\{\varepsilon^*,1\}}{k^2}\leq\varepsilon^*,$$ so that the error and stability estimates used below apply. We then apply \Cref{alg:relaxed-global} with the localized grid constructed at the buffered level $\bar\varepsilon+L_{\mathcal E}h$, and let $r^*$ and $\rho^*=\rho_{r^*}(\bar\varepsilon+L_{\mathcal E}h)$ denote the corresponding spectral index and localization radius selected by the algorithm. We write $\mathcal G_h^+(r^*,\rho^*,k)$ for the resulting relaxed-admissible grid. If $\mathcal G_h^+(r^*,\rho^*,k)=\emptyset$, \Cref{lem:BufferedAdmissibleDetection}, applied with $\varepsilon^*=\bar\varepsilon$, implies that $\mathcal I_{\mathcal N}(c,d,k)\cap\mathcal U(\bar\varepsilon)=\emptyset$. The feasible set $\mathcal I_{\mathcal N}(c,d,k)$ is nonempty, since it contains every constant polynomial system of Euclidean norm one, and compact by \Cref{prop:FeasPolyComp}. Therefore, $$\min_{A\in\mathcal I_{\mathcal N}(c,d,k)}\mathcal E(A)>\bar\varepsilon.$$ In this case, the algorithm returns \textnormal{EMPTY} and terminates. Otherwise, the grid is nonempty, and we proceed as follows. Let $G_h^*$ be a point of minimal empirical cost in $\mathcal G_h^+(r^*,\rho^*,k)$ and set $\hat A:=G_h^*$. Since $\hat A\in S_{\mathcal N}^{cN-1}$ and $$K_{\mathcal N}(\hat A)\leq k\Lambda(h),$$ while $h\leq1/(4k^2)$, the uniform stability estimate gives $$\Lambda(h)\leq\Lambda_0:=\sqrt{2}.$$ Therefore $\hat A\in\mathcal I_{\mathcal N}(c,d,\Lambda_0k)$. We now distinguish two cases. \ \ (a) Suppose that $\mathcal E(\hat A)\leq\varepsilon$. Then alternative (i) holds. \ \ (b) Suppose instead that $\mathcal E(\hat A)\geq\varepsilon$. We claim that $$\min_{A\in\mathcal I_{\mathcal N}(c,d,k)}\mathcal E(A)\geq\min\left\{\frac{\mathcal E(\hat A)}{2},\bar\varepsilon\right\}.$$ Let $$0<t<\min\left\{\frac{\mathcal E(\hat A)}{2},\bar\varepsilon\right\}.$$ Since $\mathcal E(\hat A)\geq\varepsilon=2\varepsilon^*$ and $L_{\mathcal E}h\leq\varepsilon^*$, we have $$L_{\mathcal E}h\leq\frac{\mathcal E(\hat A)}{2},$$ and therefore $$t+L_{\mathcal E}h<\frac{\mathcal E(\hat A)}{2}+\frac{\mathcal E(\hat A)}{2}=\mathcal E(\hat A).$$ Since $\hat A=G_h^*$ has minimal empirical cost on $\mathcal G_h^+(r^*,\rho^*,k)$, no point of this grid satisfies $$\mathcal E(G)\leq t+L_{\mathcal E}h.$$ Since $t<\bar\varepsilon$, \Cref{lem:BufferedAdmissibleDetection} therefore implies $\mathcal I_{\mathcal N}(c,d,k)\cap\mathcal U(t)=\emptyset$. Since this holds for every $$0<t<\min\left\{\frac{\mathcal E(\hat A)}{2},\bar\varepsilon\right\},$$ we obtain $$\min_{A\in\mathcal I_{\mathcal N}(c,d,k)}\mathcal E(A)\geq\min\left\{\frac{\mathcal E(\hat A)}{2},\bar\varepsilon\right\}.$$ Hence alternative (ii) holds. 
\
\\
 II°) The explicit arithmetic complexity estimate follows by combining the constructive covering bound of \Cref{sec:Explicit_Net_Construction} with the admissibility cost derived in \Cref{subsec:Auxiliary_Arithmetic_Complexity_Bounds}. We first keep the different contributions explicit and combine them only at the end. \begin{enumerate} \item We begin with the cardinality of the discretization net $\mathcal N$ used to approximate the norm and the discriminant functional. By \Cref{thm:disc_C2_ball_one_net_delta}, $\mathcal N$ can be chosen so that \begin{equation}\label{eq:NetSizeBound}|\mathcal N|\leq\left(1+\frac{Cd^4kR}{\eta(1-\eta)}\right)^n.\end{equation} 

\item For each center of the localized grid, evaluating the condition number requires the matrices associated with the discretized norm and discriminant functional. These matrices are computed once in a preprocessing step. By \Cref{prop:ArithmeticComplexityPreprocessingStep}, the corresponding arithmetic complexity satisfies $$T_{\mathrm{prepr}}(n,d,k,l)=\mathcal O\left(|\mathcal N|\binom{n+d}{d}^2\right).$$

 \item By the explicit construction of \Cref{sec:Explicit_Net_Construction}, the localized grid is obtained from Frobenius nets on the longitudinal and transverse factors associated with the decomposition $\Mat(c\times N,\mathbb R)=\mathcal V_r\oplus\mathcal V_r^{\perp_F}$, followed by radial normalization. We define the corresponding constructive covering function by $$\hat{\varphi}(r,h,\rho):=\begin{cases}\displaystyle \left(1+\frac{CM_{\mathcal N}}{m_{\mathcal N}h}\right)^{cr}\left(1+\frac{CM_{\mathcal N}\rho}{m_{\mathcal N}h}\right)^{c(N-r)}, & 0<\rho<1,\\[2mm]\displaystyle \left(1+\frac{CM_{\mathcal N}}{m_{\mathcal N}h}\right)^{cN}, & \rho\geq1.\end{cases}$$ Thus, for the spectral index $r^*$ selected by the algorithm and the corresponding radius $\rho^*$, $$|\mathcal G_h(r^*,\rho^*)|\leq\hat{\varphi}(r^*,h,\rho^*).$$ Since the auxiliary Frobenius nets are constructed in dimensions $cr^*$ and $c(N-r^*)$, generating and storing each resulting center requires $\mathcal O(cN)$ arithmetic operations, and hence $$T_{\mathrm{grid}}(r^*,h,\rho^*)=\mathcal O\left(cN\,\hat{\varphi}(r^*,h,\rho^*)\right).$$ \item Once the preprocessing step has been performed, the relaxed-admissibility test is applied to each candidate of the localized grid. By \Cref{lem:ArithmeticCostAdmissibility}, the arithmetic complexity of this test for a single candidate satisfies $$T_{\mathrm{adm}}(n,c,d,k,|\mathcal N|)=\mathcal O\left(|\mathcal N|\left(cNn^2+c^2n+c^3+T_{\mathrm{tensor}}(n,c)\right)\right).$$ \item After the relaxed-admissible grid has been selected, the empirical cost is evaluated on its elements. Denoting by $T_{\mathcal E}(n,c,d)$ the arithmetic cost of a single evaluation, we have $$T_{\mathcal E}(n,c,d)=\mathcal O\left(c\binom{n+d}{d}^2+c^2\binom{n+d}{d}\right).$$ \end{enumerate} Combining the preceding estimates, and recalling the choices $$h=\frac{\min\{\varepsilon^*,1\}}{8k^2},\qquad \varepsilon^*:=\frac{\varepsilon}{2},$$ the arithmetic complexity satisfies $$
\begin{aligned}
T_{\mathrm{alg}}(n,k,l,d,\varepsilon)\le T_{\mathrm{prepr}}(n,d,k,l)&+T_{\mathrm{grid}}(r^*,h,\rho^*)+\\
&+\hat{\varphi}(r^*,h,\rho^*)\Big(T_{\mathrm{adm}}(n,c,d,k,|\mathcal N|)+T_{\mathcal E}(n,c,d)\Big).
\end{aligned}$$ Since the single-candidate admissibility cost dominates the $\mathcal O(cN)$ cost required to construct a grid point, $T_{\mathrm{grid}}$ is absorbed into the last term. Hence $$T_{\mathrm{alg}}(n,k,l,d,\varepsilon)\leq T_{\mathrm{prepr}}(n,d,k,l)+\hat{\varphi}(r^*,h,\rho^*)\Big(T_{\mathrm{adm}}(n,c,d,k,|\mathcal N|)+T_{\mathcal E}(n,c,d)\Big).$$

III°) We now derive the high-dimensional form of the preceding complexity bound. We consider the asymptotic regime in which $k$, $\varepsilon$, $\bar\varepsilon$, $d$, and $l$ are fixed, while $R$ and $\eta$ are absorbed into the implicit constants, and $n\to\infty$. We assume that the sample size is chosen according to the sufficient bound in \Cref{cor:TestingSampleComplexity}. Thus, for fixed $d$, $\varepsilon$, $\alpha$, and geometric parameters, one has $m=\mathcal O(n^{d+1})$. The costs of evaluating the basis at the sample points, constructing the empirical matrix $V$, and computing its spectral decomposition are therefore polynomial in $n$ and are absorbed into the final exponential upper bound. Since $c=n-l$, one has $c\sim n$, while, for fixed $d$, $$N=\binom{n+d}{d}\sim\frac{n^d}{d!}.$$ Moreover, $$|\mathcal N|\leq\left(1+\frac{Cd^4kR}{\eta(1-\eta)}\right)^n=\mathcal O\left((Cd^4k)^n\right).$$ The preprocessing complexity and the single-candidate admissibility cost satisfy $$T_{\mathrm{prepr}}(n,d,k,l)=\mathcal O\left(|\mathcal N|n^{2d}\right),\qquad T_{\mathrm{adm}}(n,c,d,k,|\mathcal N|)=\mathcal O\left(|\mathcal N|\left(n^{d+3}+T_{\mathrm{tensor}}(n,c)\right)\right),$$ while the empirical cost evaluation satisfies $$T_{\mathcal E}(n,c,d)=\mathcal O\left(n^{2d+1}+n^{d+2}\right).$$ Combining these estimates yields $$T_{\mathrm{alg}}(n,k,l,d,\varepsilon)\leq\mathcal O\left(\hat{\varphi}(r^*,h,\rho^*)(Cd^4k)^n\left(n^{d+3}+n^{2d+1}+T_{\mathrm{tensor}}(n,c)\right)\right).$$ To make the geometric contribution explicit, suppose first that $\rho^*<1$. By the constructive covering estimate, $$\hat{\varphi}(r^*,h,\rho^*)=\left(1+\frac{CM_{\mathcal N}}{m_{\mathcal N}h}\right)^{cr^*}\left(1+\frac{CM_{\mathcal N}\rho^*}{m_{\mathcal N}h}\right)^{c(N-r^*)},$$ where, by the definition of the buffered localization radius, $$\rho^*=M_{\mathcal N}\sqrt{\frac{\bar\varepsilon+L_{\mathcal E}h}{\varepsilon_{r^*+1}}}.$$ Since $L_{\mathcal E}\leq8$, $h=\min\{\varepsilon^*,1\}/(8k^2)$, $\varepsilon^*=\varepsilon/2$, and $\varepsilon\leq\bar\varepsilon$, we have $$L_{\mathcal E}h\leq\varepsilon^*\leq\frac{\bar\varepsilon}{2},$$ and therefore $$\bar\varepsilon+L_{\mathcal E}h\leq\frac{3}{2}\bar\varepsilon.$$ Hence $$\rho^*\leq CM_{\mathcal N}\sqrt{\frac{\bar\varepsilon}{\varepsilon_{r^*+1}}}.$$ Moreover, $$\frac{1}{h}=\frac{16k^2}{\min\{\varepsilon,2\}}.$$ Since $0<\varepsilon\leq\bar\varepsilon$ and $\bar\varepsilon$ is fixed in the asymptotic regime, $$\frac{1}{\min\{\varepsilon,2\}}\leq\frac{\max\{1,\bar\varepsilon/2\}}{\varepsilon},$$ and the factor $\max\{1,\bar\varepsilon/2\}$ can be absorbed into the implicit constant. Therefore, absorbing fixed numerical factors into $C$, we obtain $$\hat{\varphi}(r^*,h,\rho^*)\leq\left(1+\frac{CM_{\mathcal N}k^2}{m_{\mathcal N}\varepsilon}\right)^{cr^*}\left(1+\frac{CM_{\mathcal N}^2k^2}{m_{\mathcal N}\varepsilon}\sqrt{\frac{\bar\varepsilon}{\varepsilon_{r^*+1}}}\right)^{c(N-r^*)}.$$ Consequently, for fixed $k$, $\varepsilon$, $\bar\varepsilon$, $d\geq1$, and $l$, and as $n\to\infty$, we obtain $$\begin{aligned} T_{\mathrm{alg}}(n,k,l,d,\varepsilon)\leq\mathcal O\bigg(\left(1+\frac{CM_{\mathcal N}k^2}{m_{\mathcal N}\varepsilon}\right)^{cr^*}&\left(1+\frac{CM_{\mathcal N}^2k^2}{m_{\mathcal N}\varepsilon}\sqrt{\frac{\bar\varepsilon}{\varepsilon_{r^*+1}}}\right)^{c(N-r^*)}\cdot \\ &\cdot(Cd^4k)^n\left(n^{d+3}+n^{2d+1}+T_{\mathrm{tensor}}(n,c)\right)\bigg). \end{aligned}$$ If $\rho^*\geq1$, the localized tube coincides with the full normalized sphere, and the constructive covering bound is $$\hat{\varphi}(r^*,h,\rho^*)=\left(1+\frac{CM_{\mathcal N}}{m_{\mathcal N}h}\right)^{cN}.$$ Using $$\frac{1}{h}=\frac{16k^2}{\min\{\varepsilon,2\}}\leq\frac{16\max\{1,\bar\varepsilon/2\}k^2}{\varepsilon},$$ and absorbing the fixed factor depending on $\bar\varepsilon$ into the implicit constant, this is of order $$\left(1+\frac{CM_{\mathcal N}k^2}{m_{\mathcal N}\varepsilon}\right)^{cN}.$$ \end{proof}

\section{Arithmetic Complexity Bounds and Net Construction}
\label{sec:Explicit_Net_Construction}

\subsection{Auxiliary Arithmetic Complexity Bounds}
\label{subsec:Auxiliary_Arithmetic_Complexity_Bounds}
Evaluating the admissibility condition at each grid point requires computing $K_{\mathcal N}(A)$ over the discretization set $\mathcal N$. As a preliminary step, we construct the matrices associated with the discretized quadratic forms entering $\n{A}_{\mathcal N}$ and $\delta_{\mathcal N}(A)$. We denote the arithmetic complexity of this preprocessing step by $T_{\mathrm{prepr}}(n,d,k,l)$.

\begin{prop}\label{prop:ArithmeticComplexityPreprocessingStep} 
The time complexity of the preprocessing step is bounded by $$T_{\mathrm{prepr}}(n,d,k,l)\leq \mathcal O\left(|\mathcal N|\binom{n+d}{d}^{2}\right).$$ \end{prop}
 \begin{proof} The preprocessing step consists in computing, for every point of the discretization set $\mathcal N$, the matrices associated with the quadratic forms defining the discretized norm and discriminant. Each of these matrices is of size $N\times N$ and can be constructed by computing the corresponding pairwise products of the basis evaluations (and their derivatives up to order two). Therefore, the arithmetic cost associated with one discretization point is bounded by $ \mathcal O(N^2).$ Since the construction is performed independently for every point of $\mathcal N$, the total preprocessing cost satisfies $$ T_{\mathrm{prepr}}(n,d,k,l) \leq \mathcal O(|\mathcal N|N^2). $$ Substituting the value of $N$ gives the claimed bound. 
\end{proof}
Once the preprocessing step is completed, we first estimate the arithmetic complexity $T_{\mathrm{discr}}(n,l,k,d)$ of evaluating the discretized discriminant functional.

\begin{prop}\label{prop:ArithmeticCostDiscriminant} For any $A\in\Mat(c\times N,\mathbb R)$, the arithmetic complexity $T_{\mathrm{discr}}(n,l,k,d)$ of evaluating $\delta_{\mathcal N}(A)^2$ satisfies $$T_{\mathrm{discr}}(n,l,k,d)\leq\mathcal O\Big(|\mathcal N|\big(cNn+c^2n+c^3\big)\Big).$$ \end{prop}
\begin{proof} By \Cref{prop:MatrixRepresentationDiscriminant}, the evaluation of the squared discriminant is given by $$\delta_{\mathcal N}(A)^2=\min_{1\leq i\leq|\mathcal N|}\left\{\operatorname{tr}(AG_i^{(0)}A^\top)+\lambda_{\min}(AG_i^{(1)}A^\top)\right\}.$$ Fix $i\in\{1,\ldots,|\mathcal N|\}$. Since $$G_i^{(0)}=P_{\mathcal B}(y_i)P_{\mathcal B}(y_i)^\top,$$ we have $$\operatorname{tr}(AG_i^{(0)}A^\top)=\n{AP_{\mathcal B}(y_i)}^2,$$ which can be evaluated in $\mathcal O(cN)$ arithmetic operations. Moreover, $$G_i^{(1)}=J_{\mathcal B}(y_i)J_{\mathcal B}(y_i)^\top.$$ Setting $$B_i:=J_{\mathcal B}(y_i)^\top\in\Mat(n\times N,\mathbb R),$$ we obtain $$AG_i^{(1)}A^\top=(AB_i^\top)(AB_i^\top)^\top.$$ The computation of $AB_i^\top\in\Mat(c\times n,\mathbb R)$ requires $\mathcal O(cNn)$ arithmetic operations, while the construction of the symmetric matrix $$(AB_i^\top)(AB_i^\top)^\top\in\Sym(c,\mathbb R)$$ requires $\mathcal O(c^2n)$ operations. Finally, computing its smallest eigenvalue requires $\mathcal O(c^3)$ arithmetic operations. Thus, the contribution of a single discretization point has cost $$\mathcal O\left(cN+cNn+c^2n+c^3\right)=\mathcal O\left(cNn+c^2n+c^3\right).$$ Repeating the computation for all points of $\mathcal N$ and taking the minimum requires only an additional $\mathcal O(|\mathcal N|)$ comparisons, which is absorbed in the preceding bound. This proves the claim. \end{proof}

We next estimate the arithmetic complexity of evaluating the discretized norm. In addition to the quadratic terms already appearing in the discriminant, $\n{A}_{\mathcal N}$ contains a maximization over $S^{n-1}\times S^{n-1}$ arising from the Hessian contribution. This maximization is the spectral norm of the associated partially symmetric tensor; see \cite{KofidisRegalia2002,KoldaMayo2011}. We denote the arithmetic complexity of the chosen tensor-norm procedure by $T_{\mathrm{tensor}}(n,c)$.

\begin{prop}\label{prop:ArithmeticCostNorm} For any $A\in\Mat(c\times N,\mathbb R)$, the arithmetic complexity $T_{\mathrm{norm}}(n,l,k,d)$ of evaluating $\n{A}_{\mathcal N}$ satisfies $$T_{\mathrm{norm}}(n,l,k,d)\leq\mathcal O\Big(|\mathcal N|\big(cNn^2+c^2n+c^3+T_{\mathrm{tensor}}(n,c)\big)\Big),$$ where $T_{\mathrm{tensor}}(n,c)$ denotes the arithmetic complexity of evaluating the spectral norm of the tensor $D^2P(y_i)\in\mathbb R^c\otimes\operatorname{Sym}^2(\mathbb R^n)$ once this tensor has been constructed. \end{prop}
\begin{proof} By \Cref{prop:MatrixRepresentationNorm}, $$\n{A}_{\mathcal N}^2=\max_{1\leq i\leq|\mathcal N|}\left\{\operatorname{tr}(AG_i^{(0)}A^\top)+\lambda_{\max}(AG_i^{(1)}A^\top)+\max_{v_1,v_2\in S^{n-1}}\operatorname{tr}\left(AG_i^{(2)}(v_1,v_2)A^\top\right)\right\}.$$ Fix $i\in\{1,\ldots,|\mathcal N|\}$. Since $$G_i^{(0)}=P_{\mathcal B}(y_i)P_{\mathcal B}(y_i)^\top,$$ we have $$\operatorname{tr}(AG_i^{(0)}A^\top)=\n{AP_{\mathcal B}(y_i)}^2,$$ which can be evaluated in $\mathcal O(cN)$ arithmetic operations. Moreover, $$G_i^{(1)}=J_{\mathcal B}(y_i)J_{\mathcal B}(y_i)^\top.$$ Setting $$B_i:=J_{\mathcal B}(y_i)^\top\in\Mat(n\times N,\mathbb R),$$ we obtain $$AG_i^{(1)}A^\top=(AB_i^\top)(AB_i^\top)^\top.$$ Computing $AB_i^\top$ requires $\mathcal O(cNn)$ arithmetic operations, forming the resulting symmetric $c\times c$ matrix requires $\mathcal O(c^2n)$ operations, and computing its largest eigenvalue requires $\mathcal O(c^3)$ operations. Finally, the Hessian contribution satisfies $$\max_{v_1,v_2\in S^{n-1}}\operatorname{tr}\left(AG_i^{(2)}(v_1,v_2)A^\top\right)=\n{D^2P(y_i)}_{\operatorname{op}}^2.$$ Since the tensors $D^2P_{\mathcal B}(y_i)$ are available from the preprocessing step, constructing $$D^2P(y_i)=A\,D^2P_{\mathcal B}(y_i)\in\mathbb R^c\otimes\operatorname{Sym}^2(\mathbb R^n)$$ requires $\mathcal O(cNn^2)$ arithmetic operations. Once this tensor has been constructed, its spectral norm is evaluated with cost $T_{\mathrm{tensor}}(n,c)$. Hence the total cost associated with a single discretization point is bounded by $$\mathcal O\left(cN+cNn+c^2n+c^3+cNn^2+T_{\mathrm{tensor}}(n,c)\right)=\mathcal O\left(cNn^2+c^2n+c^3+T_{\mathrm{tensor}}(n,c)\right).$$ Repeating the computation for all points of $\mathcal N$ and taking the maximum requires only an additional $\mathcal O(|\mathcal N|)$ comparisons, which is absorbed in the preceding bound. This proves the claim. \end{proof}

Let $T_{\mathrm{adm}}(n,c,d,k,|\mathcal N|)$ denote the arithmetic complexity of the complete relaxed-admissibility test. Combining \Cref{prop:ArithmeticCostNorm} and \Cref{prop:ArithmeticCostDiscriminant} gives the following bound.

\begin{lem}
\label{lem:ArithmeticCostAdmissibility}
Let $A\in\Mat(c\times N,\mathbb R)$. After the preprocessing step, the evaluation of the relaxed-admissibility condition requires at most $$T_{\mathrm{adm}}(n,c,d,k,|\mathcal N|)=\mathcal O\left(|\mathcal N|\left(cNn^2+c^2n+c^3+T_{\mathrm{tensor}}(n,c)\right)\right).$$
\end{lem}
\begin{proof} The relaxed-admissibility test requires the evaluation of the condition number $$K_{\mathcal N}(A)=\frac{\n{A}_{\mathcal N}}{\delta_{\mathcal N}(A)}$$ followed by comparison with the prescribed threshold. By \Cref{prop:ArithmeticCostNorm} and \Cref{prop:ArithmeticCostDiscriminant}, the corresponding arithmetic costs satisfy $$T_{\mathrm{adm}}=T_{\mathrm{norm}}+T_{\mathrm{discr}}+\mathcal O(1).$$ Combining the previous estimates, and observing that the term $cNn^2$ dominates $cNn$, gives $$T_{\mathrm{adm}}(n,c,d,k,|\mathcal N|)=\mathcal O\left(|\mathcal N|\left(cNn^2+c^2n+c^3+T_{\mathrm{tensor}}(n,c)\right)\right).$$
\end{proof}
We emphasize that $T_{\mathrm{tensor}}(n,c)$ is kept abstract and denotes the arithmetic complexity of the procedure used to evaluate the tensor spectral norm. In principle, a certified approximation can be obtained by discretizing $S^{n-1}\times S^{n-1}$ with a sufficiently fine net and maximizing over the resulting finite set. In practical implementations, one may instead use iterative tensor power-type methods based on repeated contractions and normalization steps. Such methods are typically much more efficient in practice and have polynomial arithmetic cost per iteration, although they do not in general provide a certificate of global optimality. The corresponding computational cost is therefore left in the generic form $T_{\mathrm{tensor}}(n,c)$ throughout the complexity estimates. If $N_{\mathrm{it}}$ iterations are performed, each iteration requires $\mathcal O(cn^2)$ arithmetic operations for the Hessian contractions, yielding $$T_{\mathrm{tensor}}(n,c)\leq\mathcal O(N_{\mathrm{it}}cn^2).$$ The iteration count $N_{\mathrm{it}}$ depends on the desired accuracy and the initialization strategy, and this dependence will be kept explicit in the complexity estimates below.

\subsection{Explicit Net Construction}
\label{subsec:Explicit_Net_Construction}
The covering argument of \Cref{subsec:LocalizedNet} yields the correct intrinsic cardinality bound but does not provide an explicit net. We now give a constructive version of that argument.

We use the orthogonal decomposition induced by the low spectral subspace $\mathcal V_r$ and its Frobenius orthogonal complement to construct separate nets for the longitudinal and transverse components of the localized tube. The centers of the resulting Euclidean net are then mapped onto the unit sphere of $\n{\cdot}_{\mathcal N}$ by radial normalization. The norm-equivalence constants $m_{\mathcal N}$ and $M_{\mathcal N}$ control the distortion introduced by this procedure and determine the required resolution of the underlying Frobenius nets.

The resulting explicit covering has the same dimensional exponents $cr$ and $c(N-r)$ as the intrinsic bound of \Cref{prop:IntrinsicLocalizedTubeCovering}, with an additional dependence on the norm-equivalence ratio $M_{\mathcal N}/m_{\mathcal N}$. This ratio is estimated explicitly for the total-degree Bernstein basis in \Cref{Bernstein_Polynomials_Basis}.

\begin{thm}\label{thm:ExplicitLocalizedTubeCovering} Let $0<h<1$ and $0<\rho<1$. Then one can construct an $h$--net $\mathcal G_h(r,\rho)\subseteq S_{\mathcal N}^{cN-1}$ of $\mathcal T_{\mathcal N}(r,\rho)$ with respect to $\n{\cdot}_{\mathcal N}$ such that $$\left|\mathcal G_h(r,\rho)\right|\leq\left(1+\frac{CM_{\mathcal N}}{m_{\mathcal N}h}\right)^{cr}\left(1+\frac{CM_{\mathcal N}\rho}{m_{\mathcal N}h}\right)^{c(N-r)},$$ where $C>0$ is a universal constant. \end{thm} \begin{proof} Consider the Frobenius-orthogonal decomposition $$\Mat(c\times N,\mathbb R)=\mathcal V_r\oplus\mathcal V_r^{\perp_F},$$ where $$\dim\mathcal V_r=cr,\qquad \dim\mathcal V_r^{\perp_F}=c(N-r).$$ Set $$h_0:=\frac{h}{2\sqrt{2}M_{\mathcal N}}.$$ Choose a proper $h_0$--net $\mathcal U_{h_0}$ of $$B_F^{\mathcal V_r}\left(\frac1{m_{\mathcal N}}\right):=\left\{U\in\mathcal V_r\ \middle|\ \n{U}_F\leq\frac1{m_{\mathcal N}}\right\}$$ and a proper $h_0$--net $\mathcal Z_{h_0}$ of $$B_F^{\mathcal V_r^{\perp_F}}\left(\frac{\rho}{m_{\mathcal N}}\right):=\left\{Z\in\mathcal V_r^{\perp_F}\ \middle|\ \n{Z}_F\leq\frac{\rho}{m_{\mathcal N}}\right\}.$$ After fixing Frobenius-orthonormal bases of $\mathcal V_r$ and $\mathcal V_r^{\perp_F}$, these nets can be constructed using standard Euclidean lattice nets. Define $$\mathcal G_h(r,\rho):=\left\{\Phi_{\mathcal N}(U+Z)\ \middle|\ U\in\mathcal U_{h_0},\ Z\in\mathcal Z_{h_0},\ U+Z\neq0\right\}.$$ By construction, $\mathcal G_h(r,\rho)\subseteq S_{\mathcal N}^{cN-1}$. Let $A\in\mathcal T_{\mathcal N}(r,\rho)$ and write $$A=A_\parallel+A_\perp,\qquad A_\parallel\in\mathcal V_r,\qquad A_\perp\in\mathcal V_r^{\perp_F}.$$ Since $\n{A}_{\mathcal N}=1$, the norm equivalence gives $$\n{A_\parallel}_F\leq\n{A}_F\leq\frac1{m_{\mathcal N}}.$$ Moreover, $$m_{\mathcal N}\operatorname{dist}_F(A,\mathcal V_r)\leq\operatorname{dist}_{\mathcal N}(A,\mathcal V_r)\leq\rho.$$ Since $A_\perp$ is the Frobenius-orthogonal component of $A$, $$\n{A_\perp}_F=\operatorname{dist}_F(A,\mathcal V_r)\leq\frac{\rho}{m_{\mathcal N}}.$$ Hence there exist $U\in\mathcal U_{h_0}$ and $Z\in\mathcal Z_{h_0}$ such that $$\n{A_\parallel-U}_F\leq h_0,\qquad \n{A_\perp-Z}_F\leq h_0.$$ Since the two differences belong to Frobenius-orthogonal subspaces, $$\n{A-(U+Z)}_F^2=\n{A_\parallel-U}_F^2+\n{A_\perp-Z}_F^2\leq2h_0^2.$$ Therefore, $$\n{A-(U+Z)}_{\mathcal N}\leq M_{\mathcal N}\n{A-(U+Z)}_F\leq\sqrt2M_{\mathcal N}h_0=\frac h2.$$ Since $h<1$ and $\n{A}_{\mathcal N}=1$, this also implies $U+Z\neq0$. Consider the radial normalization map $$\Phi_{\mathcal N}:\Mat(c\times N,\mathbb R)\setminus\{0\}\longrightarrow S_{\mathcal N}^{cN-1},\qquad \Phi_{\mathcal N}(B):=\frac{B}{\n{B}_{\mathcal N}}.$$ As shown in \Cref{subsec:LocalizedNet}, for every $A\in S_{\mathcal N}^{cN-1}$ and every nonzero $B\in\Mat(c\times N,\mathbb R)$, $$\n{\Phi_{\mathcal N}(B)-A}_{\mathcal N}\leq2\n{B-A}_{\mathcal N}.$$ Applying this estimate with $B=U+Z$, we obtain $$\n{A-\Phi_{\mathcal N}(U+Z)}_{\mathcal N}\leq2\n{A-(U+Z)}_{\mathcal N}\leq h.$$ Thus $\mathcal G_h(r,\rho)$ is an $h$--net of $\mathcal T_{\mathcal N}(r,\rho)$. Finally, radial normalization cannot increase cardinality, and hence $$\left|\mathcal G_h(r,\rho)\right|\leq|\mathcal U_{h_0}|\,|\mathcal Z_{h_0}|.$$ The standard Euclidean covering estimate gives $$|\mathcal U_{h_0}|\leq\left(1+\frac{2}{m_{\mathcal N}h_0}\right)^{cr}\leq\left(1+\frac{CM_{\mathcal N}}{m_{\mathcal N}h}\right)^{cr},$$ while $$|\mathcal Z_{h_0}|\leq\left(1+\frac{2\rho}{m_{\mathcal N}h_0}\right)^{c(N-r)}\leq\left(1+\frac{CM_{\mathcal N}\rho}{m_{\mathcal N}h}\right)^{c(N-r)}.$$ Combining these estimates proves the claim. \end{proof}

\appendix

\clearpage

\bigskip
\begin{center}
{\Large\bf Appendix}\\[0.3cm]
\end{center}
\addcontentsline{toc}{part}{Appendix}
\bigskip

\section{Discretization of Polynomial Norms}\label{sec:Discretization_Polynomial_Norms}

To ensure that the discretized norm provides a faithful approximation of the continuous norm, a quantitative bound on the mesh size of the net must be established. For general smooth functions, such a bound would typically depend on their $C^3$--data. For the class of polynomial functions, this dependence can be avoided by leveraging Markov-type inequalities, which allow for the control of higher-order derivatives using lower-order ones, typically up to constants depending on the degree $d$ and the ambient dimension $n$. 

In contrast, for the specific choice of operator-based norms considered here, these bounds can be significantly refined. Specifically, the density of the discretization required to guarantee uniform control over the derivatives depends solely on the degree $d$ and remains independent of the ambient dimension, resulting in a net whose size grows approximately as $O(d^n)$. This is achieved by applying a classical one-dimensional Markov inequality--reported below in a version adapted to the present scope--to the multilinear maps defined by the first and second derivatives.

\begin{thm}[A.\,A. Markov]\label{thm:Markov}
Let $q \in \mathbb{R}[t]_{\le d}$ be a univariate polynomial of degree at most $d$. If $|q(t)| \le M$ for every $t \in [-\tau, \tau]$, then  $$|q'(t)| \le \frac{d^2}{\tau}M \quad \text{for every } t \in [-\tau, \tau].$$
\end{thm}

\begin{proof}
Define the rescaled polynomial $\tilde{q}(s) := q(\tau s)$ for $s \in [-1, 1]$. Then $\deg(\tilde{q}) \le d$ and $|\tilde{q}(s)| \le M$ for every $s \in [-1, 1]$. Applying the classical Markov inequality on the interval $[-1, 1]$ (see \cite{LorentzApprox}) yields 
$$|\tilde{q}'(s)| \le M d^2 \quad \text{for every } s \in [-1, 1].$$ 
Since $\tilde{q}'(s) = \tau q'(\tau s)$, substituting $t = \tau s$ results in 
$$|q'(t)| \le \frac{d^2}{\tau}M \quad \text{for every } t \in [-\tau, \tau],$$ 
which completes the proof.
\end{proof}

This univariate estimate can be extended to the directional derivatives of a multivariate polynomial map. By considering the polynomial along line segments, the following lemma establishes local control over higher-order derivatives in terms of the supremum of lower-order derivatives over a local ball. Crucially, the resulting bound is expressed in the operator norm and depends only on the ratio $d^2/\tau$, being independent of the ambient dimensions $n$ and $c$.

\begin{lem}\label{lem:MarkovBallOperator_local}
Let $\tau>0$ and let $P\in\R[x_1,\dots,x_n]_{\le d}^c$ be a polynomial map of degree at most $d$. Then for every integer $m\in\{0,\dots,d-1\}$ and every $x\in\R^n$ one has \begin{equation}\label{DerBoundsMarkov}\|D^{m+1}P(x)\|_{\mathrm{op}}\le \frac{(d-m)^2}{\tau}\sup_{y\in B_\tau(x)}\|D^mP(y)\|_{\mathrm{op}}.\end{equation} In particular, the constant depends only on the degree $d$ and the radius $\tau$.
\end{lem}

\begin{proof}
Fix $m\in\{0,\dots,d-1\}$ and $x\in\R^n$. Let $u\in S^{c-1}\subseteq \R^c$ and $v_0,\dots,v_m\in S^{n-1}\subseteq  \R^n$ be unit vectors. Define $$q(t):=\langle u,D^mP(x+t v_0)[v_1,\dots,v_m]\rangle,\qquad t\in[-\tau,\tau].$$ Since $P$ has degree at most $d$, the entries of the map $D^mP$ are polynomials of degree at most $d-m$, and therefore $q$ is a univariate polynomial of degree at most $d-m$. Moreover, for every $t\in[-\tau,\tau]$ one has $x+t v_0\in B_\tau(x)$, hence $$|q(t)|\le \sup_{y\in B_\tau(x)}\|D^mP(y)\|_m.$$ Applying Markov's inequality on $[-\tau,\tau]$ yields $$|q'(0)|\le \frac{(d-m)^2}{\tau}\sup_{y\in B_\tau(x)}\|D^mP(y)\|_m.$$ On the other hand, by definition of $q$, $$q'(0)=\langle u,D^{m+1}P(x)[v_0,\dots,v_m]\rangle.$$ Taking the supremum over all unit vectors $u\in S^{c-1}$ and $v_0,\dots,v_m\in S^{n-1}$, and using the dual characterization of the operator norm, we obtain $$\|D^{m+1}P(x)\|_{m+1}\le \frac{(d-m)^2}{\tau}\sup_{y\in B_\tau(x)}\|D^mP(y)\|_m,$$ which concludes the proof.
\end{proof}

While a direct application of Markov's inequality controls the higher-order derivatives of a polynomial pointwise, the estimate \eqref{DerBoundsMarkov} in \Cref{lem:MarkovBallOperator_local} provides pointwise control of higher-order derivatives in terms of the values of lower-order derivatives over a surrounding ball. This weaker control is not sufficient to establish a global approximation between the discretized $C^2$--norm and the continuous norm, and therefore requires an additional step. Specifically, a method to estimate the growth of a polynomial over larger concentric balls is needed. This can be achieved by leveraging the extremal properties of Chebyshev polynomials.

Chebyshev polynomials of the first kind $\{T_d\}_{d \ge 0}$ (see \cite{Rivlin1974} for a general treatment) are defined on the interval $[-1, 1]$ by $$T_d(x) := \cos(d \arccos x).$$ These satisfy the recurrence relation $T_0(x) = 1$, $T_1(x) = x$, and $T_{d+1}(x) = 2xT_d(x) - T_{d-1}(x)$, which allows them to be uniquely extended to polynomials of degree $d$ over the entire real line $\mathbb{R}$. In particular, they possess the extremal property 
$$|T_d(x)| \le 1 \quad \text{for every } x \in [-1, 1],$$
while exhibiting the fastest possible growth among all polynomials of degree $d$ outside this interval. The following lemma formalizes this growth inequality.

\begin{lem}\label{lem:Chebyshev_growth}
Let $p\in\mathbb{R}[t]_{\le d}$ be a polynomial of degree at most $d$. Then for every $t\in\mathbb{R}$ with $|t|\ge1$ one has $$|p(t)|\le T_d(|t|)\sup_{s\in[-1,1]}|p(s)|,$$ where $T_d$ denotes the Chebyshev polynomial of the first kind.
\end{lem}

\begin{proof}
The case $d=0$ is immediate, so assume $d\geq1$. Set
$$M:=\sup_{s\in[-1,1]}|p(s)|.$$
If $M=0$, then $p\equiv0$ and the claim is trivial. By replacing $p(s)$ with $p(-s)$ if necessary, it is enough to consider $t\geq1$. Suppose, by contradiction, that
$$|p(t)|>T_d(t)M.$$
Set
$$a:=\frac{T_d(t)M}{p(t)},$$
so that $|a|<1$, and define
$$r(s):=T_d(s)-\frac{a}{M}p(s).$$
For the Chebyshev points $s_k=\cos(k\pi/d)$, $k=0,\ldots,d$, we have
$$(-1)^k r(s_k)=1-\frac{a}{M}(-1)^kp(s_k)>0.$$
Hence $r$ has at least one zero in each of the $d$ disjoint intervals $(s_{k+1},s_k)$. Moreover, by the definition of $a$, we have $r(t)=0$. Since $t>1$, this gives at least $d+1$ distinct zeros of a polynomial of degree at most $d$, a contradiction. Therefore
$$|p(t)|\leq T_d(t)M.$$
Replacing $t$ by $|t|$ gives the stated inequality.
\end{proof}

By applying the univariate Chebyshev growth estimate along the line segments connecting the center of a ball to its boundary, we can compare the supremum of a polynomial derivative over two concentric balls of different radii. We record the resulting bound in the following lemma.

\begin{lem}\label{lem:ratio_inner_outer_derivatives}
Let $\rho, \tau > 0$, and let $P \in \mathbb{R}[x_1, \dots, x_n]_{\le d}^c$. For every $y \in \mathbb{R}^n$ and every integer $m \in \{0, \dots, d\}$, the following inequality holds:
$$\sup_{x \in B_{\rho+\tau}(y)} \|D^m P(x)\|_{\mathrm{op}} \le T_{d-m} \left( 1 + \frac{\tau}{\rho} \right) \sup_{x \in B_\rho(y)} \|D^m P(x)\|_{\mathrm{op}}.$$
\end{lem}

\begin{proof}
Fix $x\in B_{\rho+\tau}(y)$. If $x=y$, the conclusion is immediate. Otherwise set $$v:=\frac{x-y}{\|x-y\|}\in S^{n-1}\qquad\text{and}\qquad r:=\|x-y\|\le \rho+\tau.$$ Let $u\in\mathbb R^c$ and $w_1,\dots,w_m\in\mathbb R^n$ be unit vectors such that $$|\langle u,D^mP(x)[w_1,\dots,w_m]\rangle|=\|D^mP(x)\|_{\op}.$$ Consider the univariate polynomial $$q(s):=\langle u,D^mP(y+\rho s\,v)[w_1,\dots,w_m]\rangle,\qquad s\in\mathbb R.$$ Since $P$ has degree at most $d$, the map $D^mP$ is a polynomial of degree at most $d-m$, hence $q$ is a univariate polynomial of degree at most $d-m$. Moreover, if $|s|\le1$, then $y+\rho s\,v\in B_\rho(y)$, so $$|q(s)|\le \sup_{z\in B_\rho(y)}\|D^mP(z)\|_{\op}.$$ Now set $$t:=\frac{\|x-y\|}{\rho}.$$ Then $$q(t)=\langle u,D^mP(y+\rho t\,v)[w_1,\dots,w_m]\rangle=\langle u,D^mP(x)[w_1,\dots,w_m]\rangle,$$ hence $$|q(t)|=\|D^mP(x)\|_{\op}.$$ If $t\le1$, then immediately $$\|D^mP(x)\|_{\op}\le \sup_{z\in B_\rho(y)}\|D^mP(z)\|_{\op}.$$ If instead $1\le t\le 1+\tau/\rho$, applying \Cref{lem:Chebyshev_growth} to $q$ yields $$\|D^mP(x)\|_{\op}=|q(t)|\le T_{d-m}(t)\sup_{|s|\le1}|q(s)|\le T_{d-m}\!\Big(1+\frac{\tau}{\rho}\Big)\sup_{z\in B_\rho(y)}\|D^mP(z)\|_{\op},$$ where we used that $T_{d-m}$ is increasing on $[1,\infty)$. Since $x\in B_{\rho+\tau}(y)$ was arbitrary, taking the supremum over $x$ concludes the proof.
\end{proof}

By combining the local derivative control from \Cref{lem:MarkovBallOperator_local} with the global growth estimates provided by \Cref{lem:ratio_inner_outer_derivatives}, we can now establish a uniform bound for the continuous $C^l$--norm in terms of its discretized version. This result ensures that, for a sufficiently dense net $\mathcal N$, the values on the net provide a reliable proxy for the polynomial's behavior over the entire domain. For the purpose of the next result only, we use the notation $$\|P\|_{C^l,\mathcal N}:=\max_{y\in\mathcal N} \sqrt{\sum_{m=0}^l\|D^mP(y)\|_{\mathrm{op}}^2},$$ for the discretized $C^l$--quantity associated with an $r$--net $\mathcal N$. This notation will be used only in the statement and proof below; throughout the rest of the paper, the symbol $\|\cdot\|_{\mathcal N}$ is reserved for the discretized $C^2$--norm.

\begin{thm}\label{thm:disc_C2_ball_one_net_delta} Let $R > 0$ and $r > 0$, and let $\mathcal{N} \subseteq D_R^n$ be an $r$-net of $D_R^n$. Let $P \in \mathbb{R}[x_1, \dots, x_n]_{\le d}^c$ be a polynomial map of degree at most $d$. Then for every $0 < \eta < 1$, there exists an admissible radius $r(\eta, d, R) > 0$ such that, whenever $r \le r(\eta, d, R)$, one has $$\|P\|_{C^l(D_R^n)} \le \frac{1}{1-\eta} \|P\|_{C^l,\mathcal N}.$$ In particular, the admissible radius $r(\eta, d, R)$ is independent of the ambient dimensions $n$ and $c$. \end{thm}

\begin{proof} Fix $x_0\in D_R^n$, and let $y_0\in\mathcal N$ be such that $\|x_0-y_0\|\le r$. \medskip \noindent\textbf{I)} We denote by $\|\cdot\|_m$ the Euclidean norm when $m=0$, and the operator norm on $m$--linear maps when $m\ge1$. Applying the fundamental theorem of calculus to $D^mP$ along the segment $[y_0,x_0]$, we obtain for every unit vectors $v_1,\dots,v_m\in\mathbb R^n$ $$\begin{aligned} D^mP(x_0)[v_1,\dots,v_m]-D^mP(y_0)&[v_1,\dots,v_m]=\\ &=\int_0^1 D^{m+1}P(y_0+t(x_0-y_0))[x_0-y_0,v_1,\dots,v_m],dt. \end{aligned}$$ Taking the supremum over all unit vectors $v_1,\dots,v_m$ yields $$\|D^mP(x_0)-D^mP(y_0)\|_{m}\le \|x_0-y_0\|\sup_{x\in[x_0,y_0]}\|D^{m+1}P(x)\|_{m+1}.$$ Since $\|x_0-y_0\|\le r$, it follows that $$\|D^mP(x_0)-D^mP(y_0)\|_{m}\le r\sup_{x\in[x_0,y_0]}\|D^{m+1}P(x)\|_{m+1}.$$ Hence $$\|D^mP(x_0)\|_{m}\le \|D^mP(y_0)\|_{m}+r\sup_{x\in[x_0,y_0]}\|D^{m+1}P(x)\|_{m+1}.$$ Assume $r\le\tau/2$ and let $m\in\{0,\dots,d-1\}$. Then for every $x\in[x_0,y_0]$ one has $$\|x-x_0\|\le r\le\frac{\tau}{2},$$ hence $B_{\tau/2}(x)\subset B_\tau(x_0)$. By \Cref{lem:MarkovBallOperator_local} applied with parameter $\tau/2$, we obtain $$\|D^{m+1}P(x)\|_{m+1}\le \frac{2(d-m)^2}{\tau}\sup_{z\in B_\tau(x_0)}\|D^mP(z)\|_{m}.$$ \medskip \noindent\textbf{II)} Combining the previous estimates yields, for every $m\in\{0,1,\dots,l\}$, $$\|D^mP(x_0)\|_{m}\le \|D^mP(y_0)\|_{m}+r\frac{2(d-m)^2}{\tau}\sup_{z\in B_\tau(x_0)}\|D^mP(z)\|_{m}.$$ Taking the Euclidean norm over $m=0,\dots,l$, using Minkowski's inequality and $(d-m)^2\le d^2$, we obtain $$\left(\sum_{m=0}^l\|D^mP(x_0)\|_{m}^2\right)^{1/2}\le \left(\sum_{m=0}^l\|D^mP(y_0)\|_{m}^2\right)^{1/2}+r\frac{2d^2}{\tau}\left(\sum_{m=0}^l\left(\sup_{z\in B_\tau(x_0)}\|D^mP(z)\|_{m}\right)^2\right)^{1/2}.$$ Since $$\left(\sum_{m=0}^l\left(\sup_{z\in B_\tau(x_0)}\|D^mP(z)\|_{m}\right)^2\right)^{1/2}\le\sqrt{l+1}\sup_{z\in B_\tau(x_0)}\left(\sum_{m=0}^l\|D^mP(z)\|_{m}^2\right)^{1/2},$$ taking the supremum over $x_0\in D_R^n\cap B_r(y_0)$ and observing that $B_\tau(x_0)\subseteq D_{R+\tau}^n$, we obtain \begin{equation}\begin{aligned} \sup_{x_0\in D_R^n\cap B_r(y_0)}\left(\sum_{m=0}^l\|D^mP(x_0)\|_{m}^2\right)^{1/2}\le &\left(\sum_{m=0}^l\|D^mP(y_0)\|_{m}^2\right)^{1/2}+\\ &+r\frac{2\sqrt{l+1}\,d^2}{\tau}\sup_{z\in D_{R+\tau}^n}\left(\sum_{m=0}^l\|D^mP(z)\|_{m}^2\right)^{1/2}. \end{aligned}\end{equation} Taking the maximum over $y_0\in\mathcal N$ and using the fact that $\mathcal N$ is an $r$--net of $D_R^n$ yields $$\|P\|_{C^l(D_R^n)}\le \|P\|_{C^l,\mathcal N}+r\frac{2\sqrt{l+1}\,d^2}{\tau}\|P\|_{C^l(D_{R+\tau}^n)}.$$ By \Cref{lem:ratio_inner_outer_derivatives}, for every $m=0,\dots,l$, $$\sup_{x\in D_{R+\tau}^n}\|D^mP(x)\|_m\le T_{d-m}\!\left(1+\frac{\tau}{R}\right)\sup_{x\in D_R^n}\|D^mP(x)\|_m\le T_d\!\left(1+\frac{\tau}{R}\right)\sup_{x\in D_R^n}\|D^mP(x)\|_m.$$ Therefore $$\|P\|_{C^l(D_{R+\tau}^n)}\le\sqrt{l+1}\,T_d\!\left(1+\frac{\tau}{R}\right)\|P\|_{C^l(D_R^n)}.$$ It follows that $$\|P\|_{C^l(D_R^n)}\le \|P\|_{C^l,\mathcal N}+r\frac{2(l+1)d^2}{\tau}T_d\!\left(1+\frac{\tau}{R}\right)\|P\|_{C^l(D_R^n)}.$$ Assume now that $$r\frac{2(l+1)d^2}{\tau}T_d\!\left(1+\frac{\tau}{R}\right)\le\eta.$$ Then $$\|P\|_{C^l(D_R^n)}\le\frac{1}{1-\eta}\|P\|_{C^l,\mathcal N}.$$ \medskip \noindent\textbf{III)} Finally, it suffices to choose $$r\le\frac{\eta\tau}{2(l+1)d^2T_d\!\left(1+\frac{\tau}{R}\right)}.$$ In particular, taking $0<\tau\le R$ and using the identity $$T_d(2)=\frac{(2+\sqrt3)^d+(2-\sqrt3)^d}{2}$$ together with the monotonicity of $T_d$ on $[1,\infty)$, we obtain $$T_d\!\left(1+\frac{\tau}{R}\right)\le T_d(2)\le (2+\sqrt3)^d.$$ Therefore it is sufficient to require $$r\le\frac{\eta\tau}{2(l+1)d^2(2+\sqrt3)^d}.$$ This bound also implies $r\le\tau/2$, and depends only on $\eta$, $d$, $l$, and $\tau$, independently of the ambient dimension $n$ and the codimension $c$. \end{proof}
The bound for the discretization radius $r$ derived in the previous proof depends on the free parameter $\tau$, which represents the radius of the local balls used for derivative control. While the general estimate in \Cref{thm:disc_C2_ball_one_net_delta} suggests an exponential dependence on the degree $d$, the following corollary shows that by optimizing the choice of $\tau$ relative to the domain radius $R$, we can achieve a discretization scale that is purely polynomial in $d$. \begin{cor} Under the assumptions of \Cref{thm:disc_C2_ball_one_net_delta}, by choosing the local radius $\tau=R/d^2$, the conclusion of the theorem holds for a discretization scale satisfying $$r\le C_l\frac{\eta R}{d^4},$$ where $C_l>0$ depends only on $l$. Consequently, the admissible radius of the net depends only polynomially on the degree $d$ and is independent of the ambient dimensions. \end{cor} \begin{proof} From the proof of \Cref{thm:disc_C2_ball_one_net_delta}, the admissible mesh size must satisfy $$r\le\frac{\eta\tau}{2(l+1)d^2T_d(1+\tau/R)}.$$ Setting $\tau=R/d^2$, we obtain $$r\le\frac{\eta R}{2(l+1)d^4T_d(1+1/d^2)}.$$ To establish a polynomial dependence on $d$, it suffices to show that $T_d(1+1/d^2)$ remains bounded for all $d\ge1$. Using the representation $T_d(x)=\cosh(d\operatorname{arccosh}x)$ for $x\ge1$ and the inequality $$\operatorname{arccosh}(1+\delta)\le\sqrt{2\delta+\delta^2},$$ we have $$d\operatorname{arccosh}\left(1+\frac{1}{d^2}\right)\le d\sqrt{\frac{2}{d^2}+\frac{1}{d^4}}=\sqrt{2+\frac{1}{d^2}}\le\sqrt3.$$ Since $\cosh$ is monotonic, $T_d(1+1/d^2)\le\cosh(\sqrt3)$ for all $d\ge1$. Substituting this into the bound for $r$ yields $$r\le\frac{\eta R}{2(l+1)\cosh(\sqrt3)}\frac{1}{d^4},$$ which confirms the polynomial dependence and concludes the proof. \end{proof}

\section{Bernstein Polynomials and Constant Bounds}
\label{Bernstein_Polynomials_Basis}

In this section, we derive explicit norm-equivalence bounds for the total-degree Bernstein basis associated with a regular simplex. These estimates provide quantitative control of the constants entering the discretized norm construction of \Cref{subsec:Explicit_Net_Construction}.

\subsection{Total-degree Bernstein polynomial basis}
\label{Total-degree_Bernstein_polynomials_basis}
The constants $m_{\mathcal N}$ and $M_{\mathcal N}$ enter the localization and covering arguments in two distinct ways. The constant $M_{\mathcal N}$ controls the radius of the localized tubular neighborhood obtained from the spectral estimate, while the ratio $M_{\mathcal N}/m_{\mathcal N}$ determines the loss in the explicit construction of the corresponding net. Since both constants depend on the polynomial basis used to identify $\mathbb R[x_1,\ldots,x_n]_{\leq d}^c$ with $\Mat(c\times N,\mathbb R)$, the choice of basis directly affects these quantitative bounds. The purpose of this section is to estimate them explicitly for the total-degree Bernstein basis and, in particular, to establish a uniform upper bound for $M_{\mathcal N}$.

We recall the notion of the Bernstein basis on a simplex and its basic properties; see, for instance, \cite[Chap.~10]{PrautzschBoehmPaluszny2002}. A \textit{regular simplex} centered at the origin is the convex hull $\Delta:=\operatorname{conv}\{q_0,\dots,q_n\}\subseteq\R^n$ of points $q_0,\dots,q_n\in\R^n$, called its vertices, satisfying $$\sum_{i=0}^n q_i=0 \qquad\text{and}\qquad \langle q_i,q_j\rangle=\begin{cases} L^2, & i=j,\\ -\dfrac{L^2}{n}, & i\neq j, \end{cases}$$ for some $L>0$. Equivalently, the vertices $q_0,\dots,q_n$ have the same distance from the origin and the same pairwise distance from each other. Every point $x\in\Delta$ admits a unique representation $$x=\sum_{i=0}^n\lambda_i(x)q_i,\qquad \text{with}\qquad \lambda_i(x)\geq0,\qquad \sum_{i=0}^n\lambda_i(x)=1.$$ The numbers $\lambda_0(x),\ldots,\lambda_n(x)$ define the barycentric coordinates of $x$ with respect to the vertices of $\Delta$. More generally, the barycentric coordinates extend to affine functions $\lambda_0,\dots,\lambda_n:\mathbb R^n\to\mathbb R$ characterized by $$x=\sum_{i=0}^n\lambda_i(x)q_i,\qquad \sum_{i=0}^n\lambda_i(x)=1.$$ For a regular simplex centered at the origin, these coordinates are explicitly given by $$\lambda_i(x)=\frac{1}{n+1}+\frac{n}{(n+1)L^2}\langle x,q_i\rangle,\qquad i=0,\dots,n.$$ We now use these barycentric coordinates to define the Bernstein polynomials of total degree $d$.

\begin{defn}
Let $q_0,\dots,q_n\in\mathbb R^n$ be the vertices of a regular simplex $\Delta$ centered at the origin, and let $\lambda_0,\dots,\lambda_n$ be the associated barycentric coordinates. For $\alpha=(\alpha_0,\dots,\alpha_n)\in\mathbb N^{n+1}$ with $|\alpha|=d$, define $$B_\alpha^d(x)=\binom{d}{\alpha}\lambda_0(x)^{\alpha_0}\cdots\lambda_n(x)^{\alpha_n},\qquad\text{where}\qquad \binom{d}{\alpha}=\frac{d!}{\alpha_0!\cdots\alpha_n!}.$$
\end{defn}
The polynomials $B_\alpha^d$ introduced above are called the total-degree Bernstein polynomials of degree $d$ associated with the regular simplex $\Delta$. The following proposition collects the basic properties of the corresponding Bernstein basis.
\begin{prop}
The family of Bernstein polynomials $$\mathcal B_d=\{B_\alpha^d:|\alpha|=d\}$$ is a basis of $\mathbb R[x_1,\dots,x_n]_{\leq d}$. Moreover, on $\Delta$, the Bernstein basis satisfies $$B_\alpha^d(x)\geq0,\qquad \sum_{|\alpha|=d}B_\alpha^d(x)=1.$$
\end{prop}

We now turn to the explicit estimation of the norm-equivalence constants $m_{\mathcal N}$ and $M_{\mathcal N}$ for the total-degree Bernstein basis introduced above. To this end, we compare the Euclidean norm of the Bernstein coefficient matrix with the $C^2$--norm of the associated polynomial system. We choose the regular simplex $\Delta=\operatorname{conv}\{q_0,\dots,q_n\}$ large enough so that $D_R\subseteq\Delta$. Since the inradius of a regular simplex centered at the origin with vertex norm $L$ is $L/n$, it is sufficient to require $L\geq nR$.

\begin{prop} Let $\Delta=\operatorname{conv}\{q_0,\dots,q_n\}$ be a regular simplex such that $D_R\subseteq\Delta$, and let $P_A(x)=AP_{\mathcal B}(x)$, where $\mathcal B$ is the total-degree Bernstein basis associated with $\Delta$. Then $$\n{P_A}_{C^2(D_R)}\leq M_{\mathrm{Bern}}(n,d,L)\n{A}_F,$$ where $$M_{\mathrm{Bern}}(n,d,L):=1+d\sqrt{\frac{n}{n+1}}\frac{1}{L}+\sqrt{2}d(d-1)\frac{n}{(n+1)L^2}.$$ \end{prop} \begin{proof} Recall that $$\n{P_A}_{C^2}:=\sup_{x\in D_R}\sqrt{\n{P_A(x)}_{\R^c}^2+\n{DP_A(x)}_{\mathrm{op}}^2+\n{D^2P_A(x)}_{\mathrm{op}}^2}.$$ Let $$\mathcal I_d:=\{\alpha\in\mathbb N^{n+1}:|\alpha|=d\},\qquad |\mathcal I_d|=N,$$ and fix an ordering $$\mathcal I_d=\{\alpha^{(1)},\dots,\alpha^{(N)}\}.$$ Accordingly, we write the Bernstein basis as $$\mathcal B=\{B_{\alpha^{(1)}}^d,\dots,B_{\alpha^{(N)}}^d\}$$ and, for $A=(A_{ij})\in\Mat(c\times N,\mathbb R)$, we denote by $$a_{\alpha^{(j)}}:=A_{\cdot j}\in\mathbb R^c$$ the $j$--th column of $A$. With this convention, $$P_A(x)=AP_{\mathcal B}(x)=\sum_{j=1}^N a_{\alpha^{(j)}}B_{\alpha^{(j)}}^d(x)=\sum_{\alpha\in\mathcal I_d}a_\alpha B_\alpha^d(x).$$ Write $$P_A(x)=\sum_{|\alpha|=d}a_\alpha B_\alpha^d(x),$$ where $a_\alpha\in\mathbb R^c$ are the columns of $A$. Since the Bernstein basis is nonnegative on $\Delta$ and forms a partition of unity, for every $x\in D_R\subseteq\Delta$ we have $$\n{P_A(x)}_{\R^c}\leq\sum_{|\alpha|=d}B_\alpha^d(x)\n{a_\alpha}_{\R^c}\leq\n{A}_F.$$ The barycentric coordinates satisfy $$\nabla\lambda_i=\frac{n}{(n+1)L^2}q_i.$$ Moreover, the regular simplex relations imply the tight-frame identity $$\sum_{i=0}^n q_iq_i^\top=\frac{n+1}{n}L^2I_n.$$ Consequently, for every $v\in\mathbb R^n$, \begin{equation}\label{eq:SumGradients}\sum_{i=0}^n\braket{\nabla\lambda_i,v}^2=\frac{n}{(n+1)L^2}\n{v}_{\R^n}^2.\end{equation} For $v\in\mathbb R^n$, differentiation of the Bernstein expansion gives $$D_vP_A(x)=d\sum_{|\beta|=d-1}B_\beta^{d-1}(x)\sum_{i=0}^n a_{\beta+e_i}\braket{\nabla\lambda_i,v}.$$ Hence, by positivity, partition of unity, and Cauchy--Schwarz, $$\n{D_vP_A(x)}_{\R^c}\leq d\sum_{|\beta|=d-1}B_\beta^{d-1}(x)\left(\sum_{i=0}^n\n{a_{\beta+e_i}}_{\R^c}^2\right)^{1/2}\left(\sum_{i=0}^n\braket{\nabla\lambda_i,v}^2\right)^{1/2}.$$ Since $$\left(\sum_{i=0}^n\n{a_{\beta+e_i}}_{\R^c}^2\right)^{1/2}\leq\n{A}_F,$$ we obtain from \eqref{eq:SumGradients} $$\n{DP_A(x)}_{\mathrm{op}}\leq d\sqrt{\frac{n}{n+1}}\frac{1}{L}\n{A}_F.$$ If $d\leq1$, then $D^2P_A\equiv0$. Assume therefore that $d\geq2$. For $v,w\in\mathbb R^n$, differentiating once more gives $$D^2_{v,w}P_A(x)=d(d-1)\sum_{|\gamma|=d-2}B_\gamma^{d-2}(x)\sum_{i,j=0}^n a_{\gamma+e_i+e_j}\braket{\nabla\lambda_i,v}\braket{\nabla\lambda_j,w}.$$ For each $\gamma$, Cauchy--Schwarz yields $$\n{\sum_{i,j=0}^n a_{\gamma+e_i+e_j}\braket{\nabla\lambda_i,v}\braket{\nabla\lambda_j,w}}_{\R^c}\leq\sqrt{\left(\sum_{i,j=0}^n\n{a_{\gamma+e_i+e_j}}_{\R^c}^2\right)\left(\sum_{i,j=0}^n\braket{\nabla\lambda_i,v}^2\braket{\nabla\lambda_j,w}^2\right)}.$$ For fixed $\gamma$, each coefficient $a_\alpha$ can occur at most twice in the sum over ordered pairs $(i,j)$, and hence $$\sum_{i,j=0}^n\n{a_{\gamma+e_i+e_j}}_{\R^c}^2\leq2\n{A}_F^2.$$ Moreover, $$\begin{aligned}\sum_{i,j=0}^n\braket{\nabla\lambda_i,v}^2\braket{\nabla\lambda_j,w}^2&=\left(\sum_{i=0}^n\braket{\nabla\lambda_i,v}^2\right)\left(\sum_{j=0}^n\braket{\nabla\lambda_j,w}^2\right)\\&=\left(\frac{n}{(n+1)L^2}\right)^2\n{v}_{\R^n}^2\n{w}_{\R^n}^2.\end{aligned}$$ Using again positivity and partition of unity, together with \eqref{eq:SumGradients}, we obtain $$\n{D^2P_A(x)}_{\mathrm{op}}\leq\sqrt{2}d(d-1)\frac{n}{(n+1)L^2}\n{A}_F.$$ Therefore, for every $x\in D_R$, $$\begin{aligned}\Big(\n{P_A(x)}_{\R^c}^2+\n{DP_A(x)}_{\mathrm{op}}^2+&\n{D^2P_A(x)}_{\mathrm{op}}^2\Big)^{1/2}\leq\\&\leq\left(1+d\sqrt{\frac{n}{n+1}}\frac{1}{L}+\sqrt{2}d(d-1)\frac{n}{(n+1)L^2}\right)\n{A}_F.\end{aligned}$$ Taking the supremum over $x\in D_R$ proves the claim. \end{proof}

We observe that the condition $D_R\subseteq\Delta$ imposes a lower bound on the size of the simplex. Since the vertices satisfy $\n{q_i}=L$ and the inradius of the regular simplex is $L/n$, the inclusion $D_R\subseteq\Delta$ is guaranteed by $L\geq nR$. With the extremal choice $L=nR$, the Bernstein constant becomes $$M_{\mathrm{Bern}}(n,d,R)=1+\frac{d}{R\sqrt{n(n+1)}}+\frac{\sqrt{2}d(d-1)}{n(n+1)R^2}.$$ In particular, for fixed $d$ and $R$, one has $$M_{\mathrm{Bern}}(n,d,R)\to1\qquad\text{as}\qquad n\to\infty.$$ Thus, asymptotically in the ambient dimension, this coefficient-to-$C^2$ comparison becomes arbitrarily close to the constant one.

\subsection{Inverse Bernstein Norm Estimate} \label{subsec:InverseBernsteinNorm} 
We now prove the converse norm estimate, namely a bound of the Bernstein coefficient matrix in terms of the supremum norm of the polynomial. The argument is based on the following simple observation. Although the ambient simplex has dimension $n$, each degree-$d$ Bernstein coefficient involves at most $d$ barycentric coordinates, and is therefore already determined by the restriction of the polynomial to a face of dimension at most $d-1$. On such a face, the inverse Bernstein norm comparison is a finite-dimensional statement whose constant depends only on $d$. We first make the restriction-to-faces argument precise, then use it to obtain a dimension-independent bound on each individual Bernstein coefficient, and finally pass from this coefficientwise estimate to the Frobenius norm of the full coefficient matrix.

Recall that every face of the simplex $\Delta=\operatorname{conv}\{q_0,\ldots,q_n\}$ is itself a simplex whose vertices form a subset of $\{q_0,\ldots,q_n\}$. More precisely, for $0\leq m\leq n$ and $I=\{i_0,\ldots,i_m\}\subseteq\{0,\ldots,n\}$, we denote by $$\Delta_I:=\operatorname{conv}\{q_{i_0},\ldots,q_{i_m}\}$$ the corresponding $m$--dimensional face. The barycentric coordinates of $\Delta$ restrict naturally to $\Delta_I$: the coordinates associated with vertices outside $I$ vanish, while the remaining ones coincide with the barycentric coordinates of $\Delta_I$. Consequently, the degree-$d$ Bernstein basis on $\Delta_I$ is obtained by restricting to $\Delta_I$ the Bernstein basis functions on $\Delta$ whose multi-indices are supported on $I$. The following proposition records this compatibility and, in particular, shows that each Bernstein coefficient of $P$ can be recovered from the restriction of $P$ to a face of dimension at most $d-1$.

\begin{prop} \label{prop:BernsteinRestrictionToFaces} Let $P\in\R[x_1,\dots,x_n]_{\leq d}^c$ be a polynomial on the simplex $\Delta=\operatorname{conv}\{q_0,\ldots,q_n\}$ written in the total-degree Bernstein basis as $$P=\sum_{|\alpha|=d}a_\alpha B_\alpha^d.$$ Let $I=\{i_0,\ldots,i_m\}\subseteq\{0,\ldots,n\}$ and let $\Delta_I:=\operatorname{conv}\{q_{i_0},\ldots,q_{i_m}\}.$ Then: \begin{enumerate} \item $$P|_{\Delta_I}=\sum_{\substack{|\alpha|=d\\ \operatorname{supp}(\alpha)\subseteq I}}a_\alpha B_\alpha^d|_{\Delta_I}.$$ \item For every multi-index $\alpha$ with $\operatorname{supp}(\alpha)\subseteq I$, let $$\beta:=(\alpha_{i_0},\ldots,\alpha_{i_m})\in\mathbb N^{m+1}.$$ Then the restricted functions $B_\alpha^d|_{\Delta_I}$ are precisely the degree-$d$ Bernstein basis functions on $\Delta_I$ associated with the corresponding multi-indices $\beta$. \item In particular, for every $|\alpha|=d$, the coefficient $a_\alpha$ is a Bernstein coefficient of the restriction of $P$ to the face $$\Delta_\alpha:=\operatorname{conv}\{q_i:i\in\operatorname{supp}(\alpha)\},$$ whose dimension is at most $d-1$. \end{enumerate} \end{prop} \begin{proof} i) Let $\lambda_0,\ldots,\lambda_n$ be the barycentric coordinates of $\Delta$, and let $\lambda_{i_0}^I,\ldots,\lambda_{i_m}^I$ be the barycentric coordinates of the face $\Delta_I$. For every $x\in\Delta_I$, one has $$\lambda_j(x)=0,\qquad j\notin I,$$ while $$\lambda_{i_j}(x)=\lambda_{i_j}^I(x),\qquad j=0,\ldots,m.$$ Hence, if $\operatorname{supp}(\alpha)\nsubseteq I$, then $$B_\alpha^d|_{\Delta_I}\equiv0.$$ Restricting the Bernstein expansion of $P$ to $\Delta_I$ therefore gives $$P|_{\Delta_I}=\sum_{\substack{|\alpha|=d\\ \operatorname{supp}(\alpha)\subseteq I}}a_\alpha B_\alpha^d|_{\Delta_I}.$$ ii) Suppose that $\operatorname{supp}(\alpha)\subseteq I$, and let $$\beta:=(\alpha_{i_0},\ldots,\alpha_{i_m})\in\mathbb N^{m+1}.$$ Then $|\beta|=d$, and using the restriction of the barycentric coordinates to $\Delta_I$, we obtain $$B_\alpha^d|_{\Delta_I}=\binom{d}{\alpha}\prod_{j=0}^m(\lambda_{i_j}^I)^{\alpha_{i_j}}=\binom{d}{\beta}\prod_{j=0}^m(\lambda_{i_j}^I)^{\beta_j}=B_\beta^{d,I}.$$ Thus the nonzero restrictions of the Bernstein basis functions of $\Delta$ are precisely the degree-$d$ Bernstein basis functions on $\Delta_I$. iii) Finally, let $|\alpha|=d$. Since each nonzero component of $\alpha$ is at least one, $|\operatorname{supp}(\alpha)|\leq d.$ Hence the face $$\Delta_\alpha:=\operatorname{conv}\{q_i:i\in\operatorname{supp}(\alpha)\}$$ has dimension $$\dim\Delta_\alpha=|\operatorname{supp}(\alpha)|-1\leq d-1.$$ Taking $I=\operatorname{supp}(\alpha)$ in the previous two steps shows that $a_\alpha$ is one of the Bernstein coefficients of the restriction $P|_{\Delta_\alpha}$. \end{proof}

For a fixed face dimension $m\in\{0,\dots,d-1\}$, fix an $m$--dimensional face $\Delta_I$ of $\Delta$. The degree-$d$ Bernstein representation on $\Delta_I$ gives a linear isomorphism between the finite-dimensional space of scalar polynomials of degree at most $d$ on $\Delta_I$ and the corresponding coefficient space. We define the scalar inverse Bernstein constant by $$\kappa_{d,m}:=\sup\left\{\max_{|\beta|=d}|a_\beta|\ \middle|\ p=\sum_{|\beta|=d}a_\beta B_\beta^{d,I},\ \n{p}_{C^0(\Delta_I)}\leq1\right\}.$$ Thus, $\kappa_{d,m}$ is the smallest constant such that every scalar polynomial $$p=\sum_{|\beta|=d}a_\beta B_\beta^{d,I}$$ satisfies $$\max_{|\beta|=d}|a_\beta|\leq\kappa_{d,m}\n{p}_{C^0(\Delta_I)}.$$ By finite dimensionality, $\kappa_{d,m}<\infty$. Moreover, its value depends only on $d$ and $m$, and not on the particular face $\Delta_I$: indeed, any two $m$--dimensional simplices are affinely equivalent, and under the corresponding affine change of variables both the Bernstein representation and the supremum norm are preserved. Hence the same constant $\kappa_{d,m}$ applies to every $m$--dimensional face of $\Delta$. The same constant also controls $\mathbb R^c$--valued polynomials independently of $c$. Indeed, let $$P=\sum_{|\beta|=d}a_\beta B_\beta^{d,I},\qquad a_\beta\in\mathbb R^c,$$ and fix $\beta$. If $a_\beta\neq0$, set $u:=a_\beta/\n{a_\beta}_{\mathbb R^c}$. Then the scalar polynomial $p_u:=\braket{u,P}$ has Bernstein coefficient $\braket{u,a_\beta}=\n{a_\beta}_{\mathbb R^c}$ and satisfies $$\n{p_u}_{C^0(\Delta_I)}\leq\n{P}_{C^0(\Delta_I)}.$$ Therefore $$\n{a_\beta}_{\mathbb R^c}\leq\kappa_{d,m}\n{P}_{C^0(\Delta_I)}.$$ Finally, since by \Cref{prop:BernsteinRestrictionToFaces} only faces of dimension at most $d-1$ are needed, we set $$\kappa_d:=\max_{0\leq m\leq d-1}\kappa_{d,m}.$$ Then $\kappa_d<\infty$ and depends only on $d$, independently of the ambient dimension $n$ and the target dimension $c$.

 \begin{lem} \label{lem:BernsteinCoefficientControl} Let $P\in\mathbb R[x_1,\dots,x_n]_{\leq d}^c$ have Bernstein representation $$P=\sum_{|\alpha|=d}a_\alpha B_\alpha^d,$$ and let $A$ be the corresponding coefficient matrix. Then $$\max_{|\alpha|=d}\n{a_\alpha}_{\mathbb R^c}\leq\kappa_d\n{P}_{C^0(\Delta)}.$$ Consequently, $$\n{A}_F\leq\kappa_d\sqrt{N}\,\n{P}_{C^0(\Delta)}.$$ \end{lem} \begin{proof} Fix a multi-index $\alpha$ with $|\alpha|=d$ and let $S:=\operatorname{supp}(\alpha).$ By \Cref{prop:BernsteinRestrictionToFaces}, the face $$\Delta_\alpha:=\operatorname{conv}\{q_i:i\in S\}$$ has dimension at most $d-1$, and $a_\alpha$ is one of the Bernstein coefficients of the restriction $P|_{\Delta_\alpha}$. Hence, by the definition of $\kappa_d$, $$\n{a_\alpha}_{\mathbb R^c}\leq\kappa_d\n{P}_{C^0(\Delta_\alpha)}\leq\kappa_d\n{P}_{C^0(\Delta)}.$$ Taking the maximum over all $|\alpha|=d$ gives $$\max_{|\alpha|=d}\n{a_\alpha}_{\mathbb R^c}\leq\kappa_d\n{P}_{C^0(\Delta)}.$$ Finally, since the vectors $a_\alpha$ are the columns of $A$, $$\n{A}_F^2=\sum_{|\alpha|=d}\n{a_\alpha}_{\mathbb R^c}^2\leq N\max_{|\alpha|=d}\n{a_\alpha}_{\mathbb R^c}^2.$$ Taking square roots proves the second estimate. \end{proof} Let $T_d$ denote the Chebyshev polynomial of the first kind. To transfer the coefficient estimate of \Cref{lem:BernsteinCoefficientControl} from the simplex $\Delta$ to the target domain $D_R^n$, we use a standard growth bound for polynomials between concentric Euclidean balls. Since the vertices of $\Delta$ have norm $L$, one has $\Delta\subseteq D_L^n$, so it is enough to compare the $C^0$--norms on $D_L^n$ and $D_R^n$. \begin{lem}\label{lem:ChebyshevGrowthBalls} Let $0<R\leq L$, and let $P\in\mathbb R[x_1,\dots,x_n]_{\leq d}^c$. Then $$\n{P}_{C^0(D_L^n)}\leq T_d\left(\frac{L}{R}\right)\n{P}_{C^0(D_R^n)}.$$ \end{lem} \begin{proof} Fix $x\in D_L^n$. If $\n{x}_{\mathbb R^n}\leq R$, the estimate is immediate. Assume therefore that $\n{x}_{\mathbb R^n}>R$, and set $$v:=\frac{x}{\n{x}_{\mathbb R^n}}.$$ If $P(x)=0$, there is nothing to prove. Otherwise, choose $u\in\mathbb R^c$ with $\n{u}_{\mathbb R^c}=1$ such that $$\braket{u,P(x)}=\n{P(x)}_{\mathbb R^c}.$$ The function $$p(t):=\braket{u,P(tv)}$$ is a univariate polynomial of degree at most $d$. By the classical Chebyshev growth inequality, $$|p(\n{x}_{\mathbb R^n})|\leq T_d\left(\frac{\n{x}_{\mathbb R^n}}{R}\right)\sup_{|t|\leq R}|p(t)|.$$ Moreover, $$\sup_{|t|\leq R}|p(t)|\leq\n{P}_{C^0(D_R^n)}.$$ Since $T_d$ is increasing on $[1,\infty)$ and $\n{x}_{\mathbb R^n}\leq L$, we obtain $$\n{P(x)}_{\mathbb R^c}\leq T_d\left(\frac{L}{R}\right)\n{P}_{C^0(D_R^n)}.$$ Taking the supremum over $x\in D_L^n$ proves the claim. \end{proof} Combining \Cref{lem:BernsteinCoefficientControl} with the preceding growth estimate yields a bound for the Frobenius norm of the Bernstein coefficient matrix in terms of the $C^0$ norm on $D_R^n$. Since the latter is bounded by the $C^2$ norm, this gives the desired lower norm comparison. \begin{prop}\label{prop:LowerBernsteinNormComparison} Let $\Delta=\operatorname{conv}\{q_0,\ldots,q_n\}$ be a regular simplex centered at the origin, with $\|q_i\|=L$, and assume that $D_R^n\subseteq\Delta$. Then, for every Bernstein coefficient matrix $A$ associated with $\Delta$, $$m_{\mathrm{Bern}}(n,d,L,R)\|A\|_F\leq\|P_A\|_{C^2(D_R^n)},$$ where $$m_{\mathrm{Bern}}(n,d,L,R):=\frac{1}{\kappa_d\sqrt{\binom{n+d}{d}}\,T_d\left(\frac LR\right)}.$$ \end{prop} \begin{proof} Since all vertices of $\Delta$ have norm $L$, and the Euclidean ball $D_L^n$ is convex, the simplex is contained in the ball: $\Delta\subseteq D_L^n$. Therefore, by \Cref{lem:BernsteinCoefficientControl} and \Cref{lem:ChebyshevGrowthBalls}, we obtain $$\|A\|_F\leq\kappa_d\sqrt{\binom{n+d}{d}}\|P_A\|_{C^0(\Delta)}\leq\kappa_d\sqrt{\binom{n+d}{d}}\,T_d\left(\frac LR\right)\|P_A\|_{C^0(D_R^n)}.$$ Finally, $$\|P_A\|_{C^0(D_R^n)}\leq\|P_A\|_{C^2(D_R^n)},$$ and hence $$\|A\|_F\leq\kappa_d\sqrt{\binom{n+d}{d}}\,T_d\left(\frac LR\right)\|P_A\|_{C^2(D_R^n)}.$$ Rearranging gives the claimed estimate. \end{proof} The previous estimates provide a lower bound for the continuous $C^2$--to--coefficient norm comparison. We now transfer this estimate to the discretized norm. By \Cref{thm:disc_C2_ball_one_net_delta}, choosing the mesh size of $\mathcal N$ sufficiently small with $\eta=\varepsilon_0$ and $l=2$, we have $$\|P_A\|_{\mathcal N}\geq(1-\varepsilon_0)\|P_A\|_{C^2(D_R^n)}.$$ Combining this estimate with \Cref{prop:LowerBernsteinNormComparison}, we obtain $$\|P_A\|_{\mathcal N}\geq(1-\varepsilon_0)m_{\mathrm{Bern}}(n,d,L,R)\|A\|_F.$$ Hence $$m_{\mathcal N}:=\min_{\|A\|_F=1}\|A\|_{\mathcal N}\geq\frac{1-\varepsilon_0}{\kappa_d\sqrt{\binom{n+d}{d}}\,T_d(L/R)}.$$

 Choosing the minimal simplex size $L=nR$ gives $$m_{\mathcal N}\geq\frac{1-\varepsilon_0}{\kappa_d\sqrt{\binom{n+d}{d}}\,T_d(n)}.$$ For fixed degree $d$, we have $$\sqrt{\binom{n+d}{d}}=O_d(n^{d/2}),\qquad T_d(n)=O_d(n^d).$$ Therefore, there exists a constant $c_{d,\varepsilon_0}>0$ such that $$m_{\mathcal N}\geq c_{d,\varepsilon_0}n^{-3d/2}.$$ Equivalently, the inverse norm-comparison constant satisfies $$\frac1{m_{\mathcal N}}\leq C_{d,\varepsilon_0}n^{3d/2}$$ for some constant $C_{d,\varepsilon_0}>0$. This estimate shows that, for fixed degree, the deterioration of the norm-equivalence constant is at most polynomial in the ambient dimension. It is natural to ask whether this dimension dependence can be removed entirely. The following example shows that a dimension-independent positive lower bound is impossible in general, even in the linear case. \begin{prop} \label{prop:DimensionDependenceLowerBernstein} A dimension-independent positive lower bound for $m_{\mathcal N}$ cannot hold in general, even for degree $d=1$. More precisely, there exists a sequence of linear polynomial systems with coefficient matrices satisfying $\|A\|_F=1$ such that, for the minimal choice $L=nR$, $$\|P_A\|_{C^2(D_R^n)}=O_R(n^{-1}).$$ Consequently, the lower norm-comparison constant necessarily satisfies $$m_{\mathcal N}\leq C_Rn^{-1}$$ for some constant $C_R>0$. \end{prop} \begin{proof} First consider a scalar degree-one Bernstein polynomial $$p_A(x)=\sum_{i=0}^na_i\lambda_i(x),$$ where the coefficients satisfy $$\sum_{i=0}^na_i=0,\qquad\sum_{i=0}^na_i^2=1.$$ Then the corresponding scalar coefficient matrix has Frobenius norm one. For an arbitrary target dimension $c\geq1$, we may regard $p_A$ as the first component of an $\mathbb R^c$--valued polynomial system, with all remaining components equal to zero; the Frobenius norm and the $C^2$ norm are unchanged. Using the explicit expression of the barycentric coordinates, we obtain $$p_A(x)=\frac{n}{(n+1)L^2}\left\langle x,\sum_{i=0}^na_iq_i\right\rangle.$$ By the Gram matrix of the regular simplex and the relation $\sum_{i=0}^na_i=0$, $$\left\|\sum_{i=0}^na_iq_i\right\|^2=L^2\frac{n+1}{n}.$$ Therefore, $$\|\nabla p_A\|=\frac1L\sqrt{\frac n{n+1}}.$$ Since $p_A$ is linear, its second derivatives vanish and $$\|p_A\|_{C^2(D_R^n)}=\sqrt{\|p_A\|_{C^0(D_R^n)}^2+\|\nabla p_A\|^2}.$$ Moreover, $$\|p_A\|_{C^0(D_R^n)}=R\|\nabla p_A\|,$$ and hence, for $L=nR$, $$\|p_A\|_{C^2(D_R^n)}=\frac{\sqrt{R^2+1}}{nR}\sqrt{\frac n{n+1}}=O_R(n^{-1}).$$ Since the discretized norm is bounded above by the corresponding continuous $C^2$ norm, this family of coefficient matrices with $\|A\|_F=1$ implies $$m_{\mathcal N}\leq C_Rn^{-1},$$ which proves the claim. \end{proof}


\section{Lipschitz Stability Estimates}
\label{sec:Lipschitz_Stabilities}
The proof of \Cref{thm:RelaxedAdmissibleGlobalOptimization} makes use of the following lemma, which provides the Lipschitz continuity of the empirical cost with respect to the metric induced by $\n{\cdot}_{\mathcal N}$. In the following, we assume that $\mathcal N$ is an $r$--net of $D_R^n$ such that
\begin{equation}\label{Appx:eqNetControl}
\n{P}_{C^2(D_R^n)}\leq\frac{1}{1-\eta}\n{P}_{\mathcal N}
\end{equation}
for every polynomial system $P$.
\begin{lem}\label{lem:LipschitzCost} Let $\mathcal N$ be an $r$--net of $D_R^n$ such that \eqref{Appx:eqNetControl} holds. Then, for every $A,B\in S_{\mathcal N}^{cN-1}$, $$\left|\mathcal E(A)-\mathcal E(B)\right|\leq L_{\mathcal E}\n{A-B}_{\mathcal N},\qquad L_{\mathcal E}:=2\left(1+\frac{r}{1-\eta}\right)^2.$$ \end{lem}
\begin{proof} Let $x_1,\dots,x_m$ be the observed data points. Since $\mathcal N$ is an $r$--net of $D_R^n$, for every $x_i$ there exists $y_i\in\mathcal N$ such that $\n{x_i-y_i}\leq r$. For every coefficient matrix $A\in\Mat(c\times N,\mathbb R)$, the mean value estimate gives $$\n{P_A(x_i)}_{\mathbb R^c}\leq\n{P_A(y_i)}_{\mathbb R^c}+r\,\n{P_A}_{C^1(D_R^n)}.$$ By the definition of the discretized norm and \eqref{Appx:eqNetControl}, $$\n{P_A(y_i)}_{\mathbb R^c}\leq\n{A}_{\mathcal N},\qquad \n{P_A}_{C^1(D_R^n)}\leq\n{P_A}_{C^2(D_R^n)}\leq\frac{1}{1-\eta}\n{A}_{\mathcal N}.$$ Hence $$\n{P_A(x_i)}_{\mathbb R^c}\leq\left(1+\frac{r}{1-\eta}\right)\n{A}_{\mathcal N},$$ and therefore $$\sqrt{\mathcal E(A)}=\left(\frac{1}{m}\sum_{i=1}^m\n{P_A(x_i)}_{\mathbb R^c}^2\right)^{1/2}\leq\left(1+\frac{r}{1-\eta}\right)\n{A}_{\mathcal N}.$$ The map $A\mapsto\sqrt{\mathcal E(A)}$ is the seminorm induced by the positive semidefinite quadratic form $A\mapsto\operatorname{tr}(AVA^\top)$. Hence, $$\left|\sqrt{\mathcal E(A)}-\sqrt{\mathcal E(B)}\right|\leq\sqrt{\mathcal E(A-B)}\leq\left(1+\frac{r}{1-\eta}\right)\n{A-B}_{\mathcal N}.$$ Moreover, since $A,B\in S_{\mathcal N}^{cN-1}$, $$\sqrt{\mathcal E(A)}+\sqrt{\mathcal E(B)}\leq2\left(1+\frac{r}{1-\eta}\right).$$ Consequently, $$\left|\mathcal E(A)-\mathcal E(B)\right|=\left|\sqrt{\mathcal E(A)}-\sqrt{\mathcal E(B)}\right|\left(\sqrt{\mathcal E(A)}+\sqrt{\mathcal E(B)}\right)\leq2\left(1+\frac{r}{1-\eta}\right)^2\n{A-B}_{\mathcal N}.$$ \end{proof}

We record the Lipschitz stability of the discretized discriminant on the normalized sphere, providing a direct control of its variation under perturbations in the discretized norm.

\begin{lem}
\label{lem:LipschitzDiscriminant}
The functional $\delta_{\mathcal N}(\cdot)^2$ is Lipschitz continuous on $S_{\mathcal N}^{cN-1}$ with respect to the metric induced by $\n{\cdot}_{\mathcal N}$. More precisely, for every $A,B\in S_{\mathcal N}^{cN-1}$, $$\left|\delta_{\mathcal N}(A)^2-\delta_{\mathcal N}(B)^2\right|\leq2\n{A-B}_{\mathcal N}.$$
\end{lem}
\begin{proof} For every $y_i\in\mathcal N$, define $$d_i(A):=\sqrt{\n{P_A(y_i)}_{\mathbb R^c}^2+\sigma_{\min}(\Jac P_A(y_i))^2},$$ so that $$\delta_{\mathcal N}(A)=\min_i d_i(A).$$ By the reverse triangle inequality, $$\left|\n{P_A(y_i)}_{\mathbb R^c}-\n{P_B(y_i)}_{\mathbb R^c}\right|\leq\n{P_A(y_i)-P_B(y_i)}_{\mathbb R^c},$$ while the perturbation estimate for singular values gives $$\left|\sigma_{\min}(\Jac P_A(y_i))-\sigma_{\min}(\Jac P_B(y_i))\right|\leq\n{\Jac P_A(y_i)-\Jac P_B(y_i)}_{\mathrm{op}}.$$ Hence, using the reverse triangle inequality in $\mathbb R^2$, $$\begin{aligned}|d_i(A)-d_i(B)|&\leq\sqrt{\n{P_A(y_i)-P_B(y_i)}_{\mathbb R^c}^2+\n{\Jac P_A(y_i)-\Jac P_B(y_i)}_{\mathrm{op}}^2}\\&\leq\n{A-B}_{\mathcal N}.\end{aligned}$$ Since the minimum of finitely many functions with the same Lipschitz constant is Lipschitz with the same constant, it follows that $$|\delta_{\mathcal N}(A)-\delta_{\mathcal N}(B)|\leq\n{A-B}_{\mathcal N}.$$ Moreover, by the definition of the discretized norm and since $A,B\in S_{\mathcal N}^{cN-1}$, $$\delta_{\mathcal N}(A)\leq\n{A}_{\mathcal N}=1,\qquad \delta_{\mathcal N}(B)\leq\n{B}_{\mathcal N}=1.$$ Therefore, $$\begin{aligned}\left|\delta_{\mathcal N}(A)^2-\delta_{\mathcal N}(B)^2\right|&=|\delta_{\mathcal N}(A)-\delta_{\mathcal N}(B)|\bigl(\delta_{\mathcal N}(A)+\delta_{\mathcal N}(B)\bigr)\\&\leq2\n{A-B}_{\mathcal N}.\end{aligned}$$ \end{proof}


\bibliography{Bibliography/references} 
\bibliographystyle{alpha}

\end{document}